%% file: collas2026_conference.tex
\documentclass{article} 
\usepackage{collas2026_conference,times}
\usepackage{easyReview}

\input{math_commands.tex}

\usepackage{hyperref}
\usepackage{amsthm}
\usepackage{algorithm, algpseudocode}
\usepackage{graphicx}
\usepackage{enumitem}
\usepackage{subcaption}         
\usepackage[space]{grffile}     
\usepackage{url}                
\usepackage{lipsum}             
\usepackage{booktabs}
\usepackage{bm}
\usepackage{multirow}
\usepackage{wrapfig}
\usepackage{hyperref}
\usepackage{url}
\usepackage{cleveref}
\newcommand{\tighttable}{\setlength{\tabcolsep}{3.2pt}\renewcommand{\arraystretch}{1.05}}

\newtheorem{proposition}{Proposition}
\hypersetup{
    colorlinks=true,
    linkcolor=red,
    filecolor=magenta,
    urlcolor=blue,
    citecolor=purple,
    pdftitle={EBTL},
    pdfpagemode=FullScreen,
    }

\title{Energy-Based Transfer for Reinforcement Learning}

\author{Zeyun Deng$^1$, Jasorsi Ghosh$^1$, Fiona Xie$^1$, Kai Cheng$^1$, Yuzhe Lu$^2$, Katia Sycara$^2$, Joseph Campbell$^1$ \\
$^1$ Purdue University \quad $^2$ Carnegie Mellon University
}

\collasfinalcopy 

\begin{document}

\maketitle

\begin{abstract}
Reinforcement learning remains sample inefficient, particularly when agents must adapt over time to related but distinct tasks, as in continual learning settings. Efficiency can be improved by transferring knowledge from a pretrained teacher policy to guide exploration in new but related tasks. However, if the new task sufficiently differs from the teacher’s training task, the transferred guidance may be sub-optimal and bias exploration toward low-reward behaviors.
We propose an energy-based transfer learning method that uses out-of-distribution detection to selectively issue guidance, enabling the teacher to intervene only in states within its training distribution.
We theoretically link energy scores to state-visitation patterns under on-policy reinforcement learning.
Empirically, our method improves sample efficiency and performance across both single-task and multi-task settings.
\end{abstract}

\section{Introduction}
Reinforcement learning (RL) excels at sequential decision-making  \citep{sutton1998reinforcement}, but credit assignment, sparse rewards, and modeling errors make it notoriously sample inefficient.
This is limiting in multi-task or continual learning settings, where agents must repeatedly learn to solve new tasks, particularly when those tasks are related to ones they have seen before.
A natural question arises: \textit{can we transfer knowledge from previously solved tasks to accelerate learning in new ones}?

One common  approach reuses a pretrained teacher to guide a student, either directly 
by suggesting actions \citep{uchendu2023jump} or indirectly by shaping 
rewards \citep{brys2015policy}. Early guidance can steer the student 
toward high-reward behaviors and reduce the need for random exploration, 
making transfer learning an appealing strategy.
However, when tasks differ sufficiently, teacher guidance can impair 
rather than accelerate learning. The teacher may issue suboptimal advice 
that biases exploration toward low-reward regions of the state-action 
space \citep{taylor2009transfer}, a phenomenon known as negative transfer.

\textbf{In this paper, we introduce an introspective transfer learning method that selectively guides exploration only when the teacher’s knowledge is likely to be beneficial.}
Our approach, \textit{energy-based transfer learning} (EBTL), is based on the insight that guidance should only be issued when the student visits states that lie within the teacher’s training distribution.
Leveraging concepts from energy-based learning \citep{lecun2006tutorial} and out-of-distribution detection  \citep{liu2020energy}, the teacher computes energy scores over states visited by the student during training, treating high-energy states as in-distribution and therefore eligible for guidance.
This mechanism enables the teacher to act only when it is sufficiently ``familiar'' with the current context, making training more efficient not by issuing \textit{more} guidance but by issuing \textit{correct} guidance. Our contributions are:

\begin{itemize}
\item We introduce an energy-based transfer learning method that selectively guides exploration only when the student’s state lies within the teacher’s training distribution.
\item We theoretically show that energy scores tend to increase on visited states as training progresses toward convergence, supporting the use of energy as a familiarity signal.
\item We empirically demonstrate that our method is more sample-efficient and has higher returns than other transfer baselines, across both single-task and multi-task settings.
\end{itemize}

\begin{figure}[htbp]
    \centering
    \includegraphics[width=0.75\linewidth]{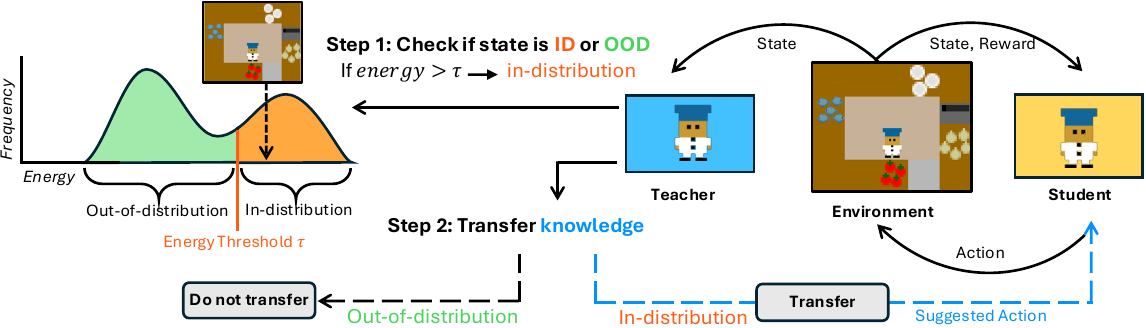}
    \caption{Overview of \textbf{energy-based transfer learning}. 
    As the student interacts with the environment, the teacher: 
    1) checks if each state is in-distribution or out-of-distribution by comparing the state's energy score to a pre-defined \textit{energy threshold}; 
    2) If the state exceeds the \textit{energy threshold}, then it is considered in-distribution for the teacher and an expert action is suggested to the student.
    }
    \label{fig:overview}
\end{figure}

\section{Related Work}
Reinforcement learning provides a framework for agents to learn 
policies that maximize cumulative reward \citep{sutton1998reinforcement}, 
but remains notoriously sample-inefficient, particularly in sparse-reward 
environments \citep{andrychowicz2017hindsight}. Transfer learning 
mitigates this by reusing knowledge from previously solved tasks to 
accelerate learning on new ones \citep{weiss2016survey}. A key 
distinction is whether the teacher interacts with the target task 
during transfer, yielding \emph{offline} vs.\ \emph{online} RL 
transfer.

Offline methods train teacher policies without target-task interaction, 
using only source data to learn representations \citep{bose2024offline} 
or task structure \citep{6400810}. Without target feedback, these 
methods must rely on broad generalization guarantees, which often yield 
conservative or mismatched advice, or prior assumptions about the 
target MDP, which require domain knowledge, making offline transfer 
brittle under covariate shift.
Beyond fully offline methods, related strategies reuse source knowledge 
during target training. Pretraining initializes policies, value 
functions, or representations before online fine-tuning 
\citep{abel2018policy}, but can misguide exploration when source and 
target distributions misalign. Hierarchical transfer uses high-level 
controllers with pre-learned options \citep{barreto2019option}, but 
degrades when the option library inadequately covers the target task.

By contrast, online transfer learning adapts during the student's training on the target task: a teacher pretrained on a source task monitors rollouts and provides guidance as the student learns. Guidance is delivered interactively through action suggestions \citep{torrey2013teaching} or reward shaping \citep{ng1999policy}. 
Behavior-based approaches encourage the student to align with the 
teacher: policy distillation introduces an auxiliary divergence loss 
to promote imitation during training 
\citep{rusu2015policy, schmitt2018kickstarting}, while action advising 
lets the teacher intervene with actions during exploration 
\citep{torrey2013teaching}. A specific instantiation of action 
advising is probabilistic policy reuse \citep{fernandez2006probabilistic}, 
where a library of past policies biases exploration via a decaying 
follow probability. The usefulness of each past policy is tracked via 
a scalar reuse gain averaged over entire episodes, which conflates 
performance across states and tasks, making it unreliable in 
multi-task settings where a policy may be beneficial in some states 
but harmful in others.
A persistent challenge is deciding when to advise, as poorly timed interventions can hinder learning \citep{torrey2013teaching}. 
JumpStart RL restricts advice to a fixed episode prefix \citep{uchendu2023jump}. 
Critic-based methods \citep{campbell2023introspective}
fine-tune the teacher’s critic on student exploration data to estimate transferability, but differences in predicted returns may come from inaccurate value estimates rather than true task mismatch, causing interventions to depend on value convergence. 

Beyond transfer RL, imitation learning has also studied when expert guidance should be requested. Methods such as DAgger~\citep{ross2011reduction}, AggreVaTe~\citep{ross2014reinforcement}, and active variants such as ASkDAgger~\citep{luijkx2025askdagger} differ in how the learner determines when to query an expert. EBTL instead makes this decision from the teacher side: the teacher estimates its own familiarity with the current state and advises only when that state is sufficiently in-distribution.

We adopt an online transfer setting that avoids requiring hand-crafted target-task knowledge and allows the teacher to decide, state by state, what knowledge is beneficial to transfer. Our method builds on this principle by applying theoretically grounded out-of-distribution detection to estimate teacher familiarity and selectively issue guidance. Importantly, we focus on an actor-based formulation of online transfer, where the decision of when to advise is derived from the teacher’s policy rather than its value estimates, leading to a more stable and direct signal for intervention.

\section{Background}
\label{sec:background}
\textbf{Reinforcement Learning.}
We study a Markov decision process (MDP), defined by the tuple $(\mathcal{S},\mathcal{A},P,R,\gamma)$: $\mathcal{S}$ is the state space; $\mathcal{A}$ the action space; $P(\cdot\mid s,a)$ the transition kernel on $\mathcal{S}$; $R:\mathcal{S}\times\mathcal{A}\to\mathbb{R}$ the reward; and $\gamma\in[0,1)$ the discount factor. At time $t$, the agent in $s_t\in\mathcal{S}$ chooses $a_t\in\mathcal{A}$, transitions to $s_{t+1}\sim P(\cdot\mid s_t,a_t)$, and receives $r_t=R(s_t,a_t)$. The goal is to learn a policy $\pi(a\mid s)$  maximizing the discounted return $\mathbb{E}_{\pi}\!\left[\sum_{t=0}^{\infty}\gamma^{t} r_t\right]$.

\textbf{Energy-Based Out-of-Distribution Detection.}
In this work, we determine whether a state is in-distribution (ID) 
or out-of-distribution (OOD) for a given policy. Common baselines 
include maximum softmax probability (MaxP)~\citep{hendrycks2016baseline}, 
entropy ~\citep{sedlmeier2020policy}, and temperature-scaled variants such as 
ODIN~\citep{liang2018odin}, all of which operate on the softmax output 
distribution. The root issue is that softmax normalization discards 
logit magnitude: for an OOD input with large diffuse logits, softmax 
can still produce a peaked distribution, yielding high MaxP, low 
entropy, and high ODIN confidence that incorrectly signals 
in-distribution \citep{nguyen2015deep, hein2019relu}.
The \textit{energy} score avoids this by operating directly on the 
raw logits without normalization. Given $\mathbf{x} \in \mathbb{R}^D$ 
and network logits $f(\mathbf{x}) \in \mathbb{R}^K$, the free energy 
is defined as
\begin{equation}
    E(\mathbf{x}; f) \;=\; -T \log \sum_{i=1}^{K} e^{f_i(\mathbf{x})/T},
    \label{energy}
\end{equation}
where $T > 0$ controls logit sharpness. By preserving logit magnitude, 
energy correctly reflects the diffuseness of OOD inputs and has been 
shown to more reliably separate ID from OOD examples than 
softmax-based criteria \citep{liu2020energy}. An input is classified 
as OOD if $E(\mathbf{x}; f) > \tau$ for a pre-computed threshold 
$\tau$, and ID otherwise. Since policy networks are logit-based 
decision models subject to the same overconfidence failure mode, 
energy provides a natural uncertainty signal for RL state 
distributions, a connection we formalize in 
the following section.

\section{Energy-Based Transfer Learning}
Our goal is to improve sample efficiency in reinforcement learning, particularly when agents must solve related tasks.
A common approach is to leverage a teacher policy trained on a source task to guide a student on a new target task. 
However, naïve teacher guidance can degrade performance when the student visits states outside the teacher's training distribution, driving uninformative exploration.
We therefore propose a transfer framework where the teacher provides guidance only in states close to its training distribution, framed as \textit{OOD detection for reinforcement learning}. 


\textbf{Problem Formulation.} \label{sec:problem_formulation}
We consider transfer between a source task $\mathcal{M}_{\text{src}} = \langle \mathcal{S}, \mathcal{A}, \mathcal{T}, \mathcal{R}, \rho_{0,\text{src}}, \gamma \rangle$ and a target task $\mathcal{M}_{\text{tgt}} = \langle \mathcal{S}, \mathcal{A}, \mathcal{T}, \mathcal{R}, \rho_{0,\text{tgt}}, \gamma \rangle$, which share a state space, action space, state transition function, reward function, and discount factor, but which may differ in their initial state distributions.
In practice, the set of reachable states between source and target tasks may differ via the change in $\rho_0$, while any state reachable in both tasks shares the same transitions and rewards.
Consequently, we frame this as a form of \textit{covariate shift}: the distribution of encountered states changes across tasks, while the state-conditional dynamics and rewards remain unchanged.


Let \( \pi_{\text{teacher}} \) and \( \pi_{\text{student}} \) denote the teacher and student policies, respectively. We denote a trajectory as \( X = \{x_t\}_{t=1}^n \), where each transition \( x_t = (s_t, a_t, s_{t+1}, r_t) \) consists of the state \( s_t \), action \( a_t \), next state \( s_{t+1} \), and reward \( r_t \). We define a score function \( \phi(s; \pi) \), where a state $s$ is considered ID with respect to a policy $\pi$ if $\phi(s) \geq \tau$, for some threshold $\tau \in \mathbb{R}$, and OOD otherwise. Throughout this paper, ID is taken with respect to $\pi_{\text{teacher}}$; formally, ID states are those with substantial density under $d^{\pi_{\text{teacher}}}_{\text{src}}$, the discounted state-visitation distribution of $\pi_{\text{teacher}}$ in $\mathcal{M}_{\text{src}}$. The action selection rule is then defined as:

\begin{equation}
\label{eq:action_rule}
\begin{aligned}
a = 
\begin{cases}
a_{\text{teacher}} \sim \pi_{\text{teacher}}(\cdot \mid s), & \text{if } \phi(s; \pi_{\text{teacher}}) \geq \tau, \\
a_{\text{student}} \sim \pi_{\text{student}}(\cdot \mid s), & \text{if } \phi(s; \pi_{\text{teacher}}) < \tau.
\end{cases}
\end{aligned}
\end{equation}

\textbf{Assumptions.} EBTL relies on two assumptions that together 
ensure the teacher's energy score correctly identifies states where 
its guidance is beneficial to the student:

\textit{(1) Teacher Reliability:} 
 The teacher policy $\pi_\text{teacher}$ is trained on-policy to near-optimality on the source task. Since on-policy 
training generates data by sampling from the teacher's own policy, the 
resulting state-visitation distribution concentrates on regions where 
the teacher has learned to act reliably. Consequently, guidance issued 
in these states is generally trustworthy. When the teacher is 
suboptimal, its advice may be unreliable even within its training 
distribution, placing an inherent ceiling on the achievable transfer 
gain.

\textit{(2) Observable Task Shift:} Differences between source and target tasks are reflected in the observable state $s$.
In particular, when the appropriate behavior depends on task context, that context must be represented in $s$; for example, a task-conditioned policy may receive the current goal, recipe, or instruction as part of its input.
This allows the teacher to distinguish states requiring different behaviors across tasks.
Under these assumptions, energy-based filtering enables reliable knowledge transfer (Equation~\ref{eq:action_rule}): familiar states tend to receive higher energy scores  and trigger teacher guidance while unfamiliar states tend to receive lower scores and defer to the student policy.





\subsection{Energy Scores and State Visitation}
\label{sec:energy-visitation}
We seek a score function that reflects the teacher's familiarity with each state without explicitly tracking visitation counts. Drawing on energy-based out-of-distribution detection \citep{liu2020energy}, we define the \textit{energy score} $\phi(s;\pi_{\text{teacher}}) = -E(s;\pi_{\text{teacher}})$, where \( E(s;\pi_{\text{teacher}}) \) is the free energy computed from the teacher's network logits (Equation~\ref{energy}). We set $\phi$ to the \textit{negative} free energy so that familiar (ID) states receive higher scores than unfamiliar (OOD) states. 

Note that Equation~\ref{energy} is defined over a discrete set of 
logits, and is therefore most natural when the policy outputs a 
categorical distribution over a finite action space.
For continuous action spaces, one may either discretize the action space or derive an analogous energy score from the policy's output distribution.
Refer to Appendices~\ref{app:cont-energy} and \ref{app:CEBTL} for the latter approach where we derive an energy score for a continuous action policy parameterized by a Gaussian distribution.

Since the teacher policy is assumed to converge to a near-optimal policy prior to transfer, it is expected to gradually place higher probability mass on near-optimal actions during training. We formalize this key property below, which enables EBTL.

\begin{proposition}[Asymptotic on-policy monotonicity of the energy score]
\label{prop:onpolicy-id}
Under on-policy training towards convergence, policy gradient updates 
tend to increase $\phi_\theta(s)$ for visited states.
\end{proposition}
\begin{proof}
Assume actor-critic architecture for the on-policy update. 
Let $\pi_\theta(a \mid s) = \mathrm{softmax}(f_\theta(s)/T)_a$ with 
logits $f_\theta(s) \in \mathbb{R}^K$ for $K$ actions, temperature 
$T > 0$, and write $\pi_j = \pi_\theta(j \mid s)$ for brevity.
Let $A(s,a) = Q(s,a) - V(s)$ denote the advantage of action $a$ 
in state $s$, where $Q^{\pi}(s, a) = \mathbb{E}_{\pi}[\sum_{t=0}^{\infty} \gamma^t r_t \mid s_0 = s, a_0 = a]$ and $V^{\pi}(s) = \mathbb{E}_{\pi}[\sum_{t=0}^{\infty} \gamma^t r_t \mid s_0 = s]$ are the action-value and 
state-value functions, respectively.

\paragraph{Energy change.}
Let $E_\theta(s) = -T\log\sum_{j=1}^K e^{f_j(s)/T}$ denote the 
free energy and $\phi_\theta(s) = -E_\theta(s)$ the energy score.
A policy gradient step on $A(s,a)\log \pi_\theta(a \mid s)$ for a 
sampled action $a$ updates each logit by
$\Delta f_j = \frac{\eta\, A(s,a)}{T}\bigl[\mathbf{1}_{j=a} - 
\pi_j\bigr]$, where $\eta > 0$ is the step size and 
$\mathbf{1}_{j=a}$ is the indicator function. Using 
$\frac{\partial E_\theta(s)}{\partial f_j} = -\pi_j$, the 
first-order change in the score is 
(see Appendix~\ref{app:proof-onpolicy})
\[
    \Delta\phi_\theta(s) 
    \;=\; 
    \frac{\eta\, A(s,a)}{T}
    \!\left[\pi_a - \sum_{j=1}^{K} \pi_j^2\right].
\]
The sign of $\Delta\phi_\theta(s)$ depends on two factors, both of 
which are non-negative towards convergence.

\paragraph{Non-negativity towards convergence.}
For the bracket term, conditioning on the high-probability event 
$\pi_a = \max_j \pi_j$ gives $\pi_a \ge \pi_j$ for all $j$, so
$\sum_j \pi_j^2 \;\le\; \pi_a \sum_j \pi_j \;=\; \pi_a$,
since $\sum_j \pi_j = 1$, and hence $\pi_a - \sum_j \pi_j^2 \ge 0$.

For the advantage term, conditioning on $a = a^*$,
\[
A(s, a^*) = Q(s,a^*) - V(s) 
= \sum_{a \neq a^*}\pi(a\mid s)\bigl[Q(s,a^*) - Q(s,a)\bigr] \ge 0,
\]
since $a^* = \arg\max_a Q(s,a)$ implies $Q(s,a^*) \ge Q(s,a)$ for all $a$.

Together, near convergence the policy concentrates on $a^*$, so $\pi(a^*\mid s)\to 1$. On the event $a=a^*$, both factors are non-negative, implying $\Delta\phi_\theta(s,a^*)\ge 0$. For $a\neq a^*$, the contribution is weighted by $\pi(a\mid s)$, whose total mass vanishes as $\pi(a^*\mid s)\to 1$. Since $\Delta\phi_\theta(s,a)$ is bounded for fixed $s$, it follows that 
$\mathbb{E}_{a\sim\pi}[\Delta\phi_\theta(s)] 
= \pi(a^*\mid s)\Delta\phi_\theta(s,a^*) 
+ \sum_{a\neq a^*}\pi(a\mid s)\Delta\phi_\theta(s,a) 
\to \Delta\phi_\theta(s,a^*) \ge 0$.
Hence, policy gradient updates tend to increase the energy score on visited states in expectation near convergence.
\end{proof}

\begin{algorithm}[H]
\caption{Energy-Based Transfer for Reinforcement Learning}
\label{alg:ebtl}
\begin{algorithmic}
\State \textbf{Input:} Teacher policy $\pi_{\text{teacher}}$, student policy $\pi_{\text{student}}$, energy threshold $\tau$, decay function $\delta$
\While{not done}
    \State Initialize empty batch $B \leftarrow \emptyset$
    \For{$t = 1 \rightarrow H$}
        \State Sample \( p \sim \mathcal{U}(0, 1) \) \Comment{Sample probability of issuing guidance}
\If{$-E(s_t; \pi_{\text{teacher}}) \geq \tau$ and $p < \delta(t)$}
    \State $a_t \gets \pi_{\text{teacher}}(a \mid s_t)$ \Comment{If $s_t$ is ID}
\ElsIf{$-E(s_t; \pi_{\text{teacher}})< \tau $}
    \State $a_t \gets \pi_{\text{student}}(a \mid s_t)$ \Comment{If $s_t$ is OOD}
\EndIf
        \State Take action $a_t$, observe $r_t$, $s_{t+1}$
        \State Compute importance ratio 
        $\rho_t \leftarrow \frac{\pi_{\text{student}}(a_t \mid s_t)}
        {\pi_{\text{teacher}}(a_t \mid s_t)}$ if $s_t$ is ID, 
        else $\rho_t \leftarrow 1$
        \Comment{Correct for off-policy teacher actions}
        \State $B \gets B \cup (s_t, a_t, s_{t+1}, r_{t}, \rho_t)$
    \EndFor
    \State Update \( \pi_{\text{student}} \) with batch $B$ \Comment{Any on-policy update}
\EndWhile
\end{algorithmic}
\end{algorithm}

\subsection{Algorithm}
Algorithm~\ref{alg:ebtl} summarizes EBTL. At each timestep, the 
teacher evaluates the current state via its energy-based OOD score 
to determine familiarity. If the state is classified as 
ID and a decaying schedule permits intervention, the 
teacher's action is adopted; otherwise, the student acts 
independently.
The intervention threshold is set as $\tau = \text{Quantile}_q
(\{\phi(s) \mid s \in \mathcal{S}_\text{teacher}\})$, where 
$\mathcal{S}_\text{teacher}$ denotes the states visited under the 
teacher policy during training and $q \in [0,1)$ controls 
strictness. Smaller values of $q$ admit a broader set of states as 
ID, whereas larger values impose stricter similarity to 
the teacher's training distribution.
To control the rate of teacher intervention, we apply a 
decay schedule $\delta(t)$, following prior 
work~\citep{schmitt2018kickstarting, uchendu2023jump, 
campbell2023introspective}. This enables early reliance on the 
teacher while gradually transferring decision-making to the student. 
Appendix~\ref{app:decay} establishes that some form of decay is 
necessary for the student to achieve autonomy, though the specific 
schedule is not critical provided it eventually reduces teacher 
influence to zero. Finally, since teacher actions are off-policy 
from the student's perspective, we apply importance-ratio correction 
to mitigate the resulting distributional 
mismatch~\citep{campbell2023introspective}.

\begin{figure}[t]
    \centering
    \includegraphics[width=(0.8\linewidth)]{Figures/Env_Overview-2.pdf}
    \caption{
    Environments used for empirical experiments. See Section~\ref{sec:experiments} for detailed descriptions.
    }
    \label{fig:env_overview}
\end{figure}

\paragraph{Energy Regularization.}
\label{loss}
As discussed in Section~\ref{sec:energy-visitation}, the score function \( \phi(s) \) is related to the teacher’s state-visitation frequency: frequently visited states tend to receive higher scores. 
This relation, however, provides no guarantee for states that lie outside the teacher’s experience. 
OOD states may produce heterogeneous scores that overlap with ID values, weakening separability.
Collapsing all OOD states to a uniformly low score would also remove useful structure by treating partially compatible states the same as completely unrelated ones.
We therefore add an energy regularizer that enforces a margin between ID and truly OOD states, improving separability while preserving the relative ordering among familiar states.
To improve separability, we augment teacher training with the energy-based loss 
of \citep{liu2020energy}. Let $\mathcal{D}_{\text{in}}$ denote ID states collected 
during teacher training and $\mathcal{D}_{\text{out}}$ denote a set of OOD states. 
For samples $s_{\text{in}} \sim \mathcal{D}_{\text{in}}$ and $s_{\text{out}} \sim \mathcal{D}_{\text{out}}$, we optimize the following loss: 
\begin{align*}
\mathcal{L}_{\text{energy}} = &\; \mathbb{E}_{\mathbf{s}_{\text{in}}} \left[ \left( \max(0, m_{\text{in}} - \phi(\mathbf{s}_{\text{in}})) \right)^2 \right]
+ \mathbb{E}_{\mathbf{s}_{\text{out}}} \left[ \left( \max(0, \phi(\mathbf{s}_{\text{out}}) - m_{\text{out}}) \right)^2 \right],
\end{align*}
$m_{\text{in}}, m_{\text{out}} \in \mathbb{R}$ are margins for ID and OOD energies, penalizing ID energies \(< m_{\text{in}}\) and OOD energies \(> m_{\text{out}}\). 
The teacher loss is $\mathcal{L}_{\text{total}} = \mathcal{L}_{\text{RL}} + \lambda \mathcal{L}_{\text{energy}}$, where \( \lambda \in \mathbb{R}^+ \) controls the regularization strength.  
In the main paper, ID samples are drawn from the teacher’s most recent update batches to approximate the empirical state distribution induced by the teacher policy, and OOD samples are drawn from random rollouts in the student environment to approximate states outside the teacher’s training distribution. 
The energy loss is robust to these choices (Appendix~\ref{app:energy-loss}): 
 uniform random sampling of OOD states from the full observation space, without assuming target-task knowledge, yields comparable transfer performance; 
separability is largely insensitive to the margin values $m_{\text{in}}$ and $m_{\text{out}}$; ID samples can be obtained by rolling out the teacher when training data are unavailable; and energy regularization accelerates convergence without degrading final teacher performance.

\section{Experiments}
\label{sec:experiments}
We evaluate our method in two settings: \textbf{single-task} and \textbf{multi-task}. 
The single-task setting uses Minigrid \citep{MinigridMiniworld23} navigation tasks where the agent reaches a goal location.
The multi-task setting uses Overcooked \citep{carroll2019utility}, where the agent learns to cook different recipes.
For each setting, we design scenarios with increasing covariate shift between teacher and student distributions to test robustness under progressively harder transfer.
We examine \textbf{whether our method improves sample efficiency} and 
\textbf{when the teacher provides guidance during learning}.
We compare against baselines that do not require training or tuning additional networks beyond the teacher policy prior to transfer.
This isolates the effect of the transfer mechanism itself, rather than gains from additional model capacity.
Introspective Action Advising \citep{campbell2023introspective} instead fine-tunes the teacher's critic on student data, so its transfer signal depends on both task compatibility and accuracy of the adapted value function.
Because this introduces an additional online value-estimation component, we present a dedicated comparison with IAA in Appendix~\ref{sec:iaa}.


We consider the following baselines:
\textbf{(1) No Transfer.} The student is trained from scratch on the target task, providing the standard reference point for all transfer methods.
\textbf{(2) Standard Action Advising (AA)}~\citep{torrey2013teaching}. The teacher provides actions throughout training with a decaying rate. This is equivalent to our method with $q = -1$, i.e., no OOD filtering. It serves as a key ablation, showing that indiscriminate guidance can lead to negative transfer when the teacher is applied in unfamiliar states.
\textbf{(3) JumpStart RL (JSRL)}~\citep{uchendu2023jump}:
A curriculum-based method that restricts guidance to early timesteps in each episode. It does not consider state-level distributional mismatch and may issue advice in irrelevant states due to its purely time-based schedule.
\textbf{(4) Kickstarting RL (KSRL)}~\citep{schmitt2018kickstarting}:
A distillation-based method that transfers knowledge from a teacher policy via an auxiliary imitation loss whose weight is gradually annealed over training.  While it encourages policy alignment, it does not account for distributional mismatch and applies supervision uniformly across states.
\textbf{(5) Fine-Tuning (FT)}: The student is initialized from a pretrained teacher policy. Convolutional
layers are frozen, and only the remaining parameters are updated during training.
When comparing against baselines, all shared hyperparameters are kept the same across methods to ensure a fair comparison. For Minigrid, each baseline is tuned independently, with our method's configuration reported in this section and baseline details provided in Appendix~\ref{app:baselines}. Since our ablation studies on Minigrid already show that transfer performance is directly linked to the quality of filtered advice, we evaluate on Overcooked without environment-specific tuning to test generalization to a more complex setting. Due to long training times, hyperparameters for all methods are set to reasonable values based on their original implementations, with shared parameters such as decay rate kept consistent across methods.

\begin{figure}[t]
    \centering
    \includegraphics[width=0.9\textwidth]{Figures/Grid_results_600dpi.pdf}
   \caption{\textbf{Minigrid} (10 seeds). 
    \textbf{(a)}  Empirical energy score distributions with respect to the teacher
policy. The source task (blue) shows the teacher’s training distribution. The target task (orange +
purple), measured during transfer, is bimodal: one mode overlaps with the source (shared sub-task,
in-distribution), while the other does not (non-shared sub-task, out-of-distribution). 
\textbf{(b)} 
Guidance quality: fraction of transferable target states correctly identified as ID and advised (orange), and fraction of non-transferable target states incorrectly identified as ID and advised (purple), with the decay schedule disabled.
An ideal gating signal yields an orange bar substantially higher than the purple bar.
    \textbf{(c)} Relative transfer (\% vs.\ scratch) across $q$ and decay schedules (50\%: guidance reaches 0 at mid-training). $q=-1$ denotes always advising (AA baseline). 
    \textbf{(d)} Learning curves: EBTL vs. baselines. Guidance ends at mid-training with $q=0.5$.
}
    \label{fig:grid_results}
\end{figure}

\begin{figure}[t]
    \centering
    \includegraphics[width=0.9\textwidth]{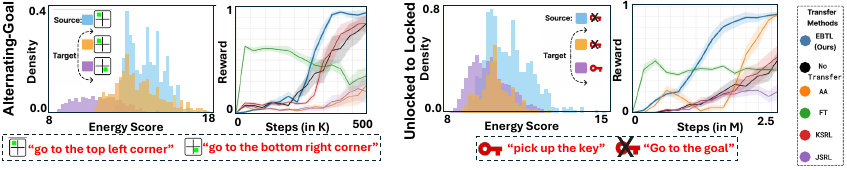}
   \caption{\textbf{Lang-Minigrid}. (10 seeds). 
Guidance ends at mid-training with $q=0.5$.
Task variation is specified through language instructions: in Alternating-Goal, the instruction indicates the target corner (e.g., \textit{go to the top left corner}), while in Unlocked-to-Locked, it specifies sequential subgoals (e.g., \textit{pick up the key, then go to the goal}).
}
    \label{fig:language}
\end{figure}

\subsection{Single-Task Setting: Minigrid}
Our Minigrid environment consists of four interconnected rooms designed for single-task transfer, illustrated in Figure~\ref{fig:env_overview}a. We construct two setups.
\textbf{Alternating Goal Room.} The source task always places the goal in a random location in Room 1 (upper-left), while the target task randomly places it in either Room 1 (upper-left) or Room 3 (lower-right).
The teacher should intervene only when the goal is in Room 1, where its prior experience applies; when the goal is in Room 3, the student must act independently.
\textbf{Locked Room.} The source task contains no doors, while the target task features a locked door partitioning the lower area that requires collecting a key in the upper rooms to open. This induces two distinct regimes: \textit{pre-key states}, where the agent must solve a key-retrieval problem entirely absent from the source task and teacher guidance is unreliable, and \textit{post-key states}, where the door is open and room connectivity is restored. Even in post-key states, the door remains visually present in the observation, introducing an observation shift the teacher has never encountered. Nevertheless, a reliable teacher should assign post-key states higher familiarity than pre-key states, as both involve the door's presence but post-key states are otherwise closer to the teacher's training distribution. Together, the novel pre-key phase, the persistent observation shift, and the altered state reachability make this a harder and more structured transfer problem than the Alternating Goal setting.
We evaluate both setups under two state encodings. 
In the discrete setting, the observation is a grid matching the environment layout, where each cell records the object type, color, and state (e.g., whether a door is open or locked), where task variation simply appears through object locations such as the goal or key. Results are shown in Figure~\ref{fig:grid_results}. 
In the language-conditioned setting, tasks are also specified via textual instructions, with results in Figure~\ref{fig:language}. 

\textbf{EBTL consistently outperforms all baselines and is robust to 
hyperparameters.} Across both Minigrid transfer setups, EBTL achieves 
the highest sample efficiency under discrete observations 
(Fig.~\ref{fig:grid_results}d) and language-conditioned observations 
(Fig.~\ref{fig:language}). Notably, most 
baselines fail to surpass training from scratch, as indiscriminate 
guidance introduces incorrect advice that biases the student toward 
suboptimal behaviors and ultimately hurts performance. This failure is 
not due to poor hyperparameter choices but rather a fundamental issue: 
these methods issue guidance regardless of whether the teacher is 
familiar with the current state, making them susceptible to harmful 
advice in out-of-distribution states. Across our hyperparameter 
sweep (Appendix~\ref{app:baselines}), JSRL exhibits large variance 
across different settings and never surpasses EBTL, rarely exceeding 
training from scratch. KSRL shows slightly lower variance than JSRL 
but consistently underperforms, never reaching EBTL's performance 
level. 
Fine-tuning illustrates a different failure mode: its reward curve
starts high because the pretrained policy already handles the shared
sub-task, but then plateaus. This suggests that on the unseen sub-task, the policy reuses
source-task behaviors that no longer apply, and unlearning them
inadvertently degrades the shared sub-task as well, forcing a slow
balancing act between the two.
In contrast, EBTL has two key hyperparameters: the energy 
quantile $q$ and the decay schedule. Across varying decay schedules 
(Fig.~\ref{fig:grid_results}c), performance remains stable across a 
wide range: $q \in [0.2, 0.7]$ achieves nearly identical results, confirming that EBTL's advantage 
stems from its selective guidance mechanism rather than hyperparameter 
tuning.

\begin{wrapfigure}{l}{0.38\textwidth}
        \centering
    \includegraphics[width=\linewidth]{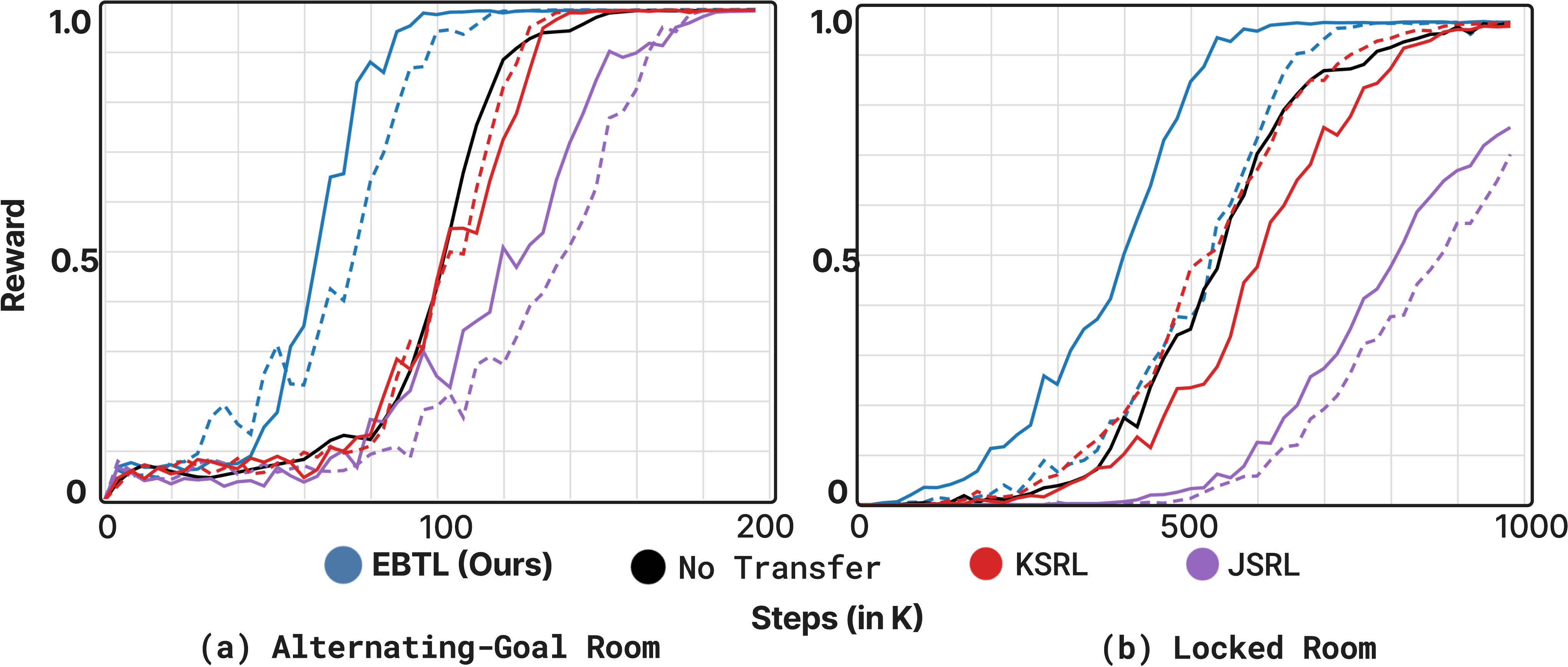}
        \caption{(10 seeds). Learning curves \textbf{with (solid)} and \textbf{without (dashed)} energy regularization in Minigrid environments.}
        \label{fig:Energy}
\end{wrapfigure}
\textbf{Energy-based scores provide superior ID/OOD separability, 
driving EBTL's improvement over baselines.}
As the energy score is theoretically connected to state visitation under on-policy training, it serves as a natural proxy for teacher familiarity. While our primary contribution is not a new OOD detection method, the quality of ID/OOD separation directly impacts transfer performance.
We therefore verify that the energy score provides more reliable separation than established baselines.
As shown in Figure~\ref{fig:sep}, MaxP~\citep{hendrycks2016baseline} and entropy~\citep{sedlmeier2020policy} fail to separate ID and OOD distributions meaningfully. 
The unregularized energy score, however, achieves separation slightly better than ODIN~\citep{liang2018odin} without requiring any OOD samples for hyperparameter calibration.
This is consistent with Proposition~\ref{prop:onpolicy-id}: on-policy training naturally induces a useful energy structure over visited states. Nevertheless, it provides no guarantee on entirely unseen states, whose scores may overlap with ID values, as discussed in Section~\ref{loss}.
Energy regularization addresses this gap, sharpening the ID/OOD boundary beyond what either ODIN or the unregularized score achieves. This translates to accelerated convergence as shown in Figure~\ref{fig:Energy}, particularly in the Locked Room setting where covariate shift is greater.  
Notably, other baselines show negligible change with or without energy regularization.


\textbf{Generalization beyond exact state visits.}
This separability advantage translates directly into a generalization
capability that exact state visitation counting cannot provide. Rather than rejecting any unvisited state as OOD, the threshold $q$ admits states that structurally resemble the teacher's experience even if never visited exactly.
The energy regularizer of Section~\ref{loss} encourages separation between ID and OOD states by pushing the energy of ID states up and OOD states down during teacher training. 
Through the network's shared representation, this can also generalize beyond states visited during source-task training, such that states resembling the teacher's experience receive higher energies than more dissimilar states.
Figure~\ref{fig:grid_results}a illustrates this: in Locked Room, pre-key states receive lower scores than post-key states, despite both containing doors which were not present in the source task.
Exact visitation counting cannot make this distinction, since both pre- and post-key states are absent from the teacher's source-task visitation distribution.
With moderate $q \geq 0.2$, EBTL instead distinguishes between these states and provides useful guidance after key collection, when navigation resembles the teacher's training distribution.




\textbf{EBTL improves monotonically as teacher proficiency increases.}
We train teachers at different proficiency levels (Table~\ref{tab:suboptimal}) and observe positive transfer gains even with a suboptimal teacher, despite assuming a converged teacher. We conjecture that partially trained teachers retain meaningful energy-visitation structure; as training converges, energy scores on familiar states increase (Proposition~\ref{sec:energy-visitation}), with actions there becoming increasingly optimal, explaining the monotonic improvement in transfer gain. 
In contrast, indiscriminate imitation baselines often fail to benefit from stronger teachers and 
may even deteriorate due to source--target misalignment, as copying teacher actions in unseen 
states performs no better than random. 
Interestingly, KSRL performs better at low teacher 
optimality, likely because an underfit teacher retains general behaviors that transfer better 
under covariate shift, which the student captures directly through a cross-entropy objective.

\begin{figure}[t]
    \centering
    \includegraphics[width=0.9\textwidth]{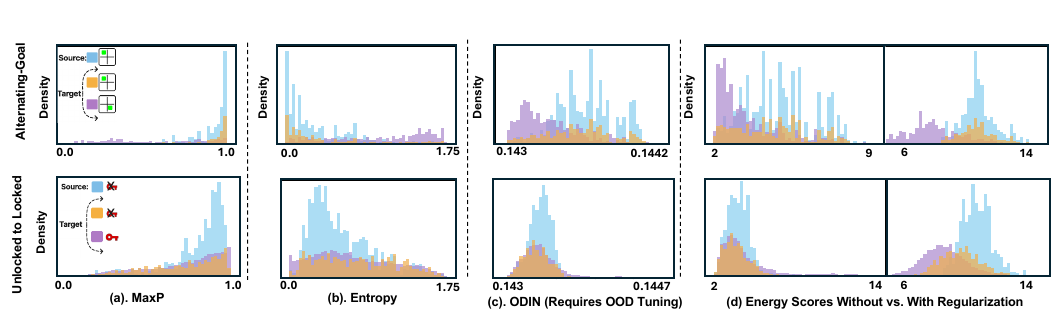}
   \caption{
   \textbf{OOD score distributions across detection 
methods in Minigrid.} 
Each panel shows the score distributions of the source task (blue) 
and target task (orange: shared sub-task, purple: non-shared sub-task) 
under the teacher policy. A good OOD detector should separate blue 
and purple while keeping blue and orange overlapping. (a) MaxP, 
(b) Entropy, and (c) ODIN, requiring OOD samples for hyperparameter 
calibration. 
(d) (Left) The energy score without regularization (no OOD samples) already 
matches ODIN, and (Right) with regularization (OOD samples used during teacher 
training) achieves substantially sharper separation.
}
    \label{fig:sep}
\end{figure}

\begin{table*}[t]
\centering
\small
\tighttable
\caption{
\textbf{Robustness to suboptimal teachers}. 
Relative transfer improvement (\%) over training from scratch (mean $\pm$ 95\% CI, 10 seeds) for Alternating Goal. 
Guidance ends at mid-training.
}
\label{tab:suboptimal}
\resizebox{0.8\textwidth}{!}{%
\begin{tabular}{@{}lccccccc@{}}
\toprule
\toprule
\textbf{Optimality} & \textbf{EBTL (0.3)} & \textbf{EBTL (0.5)} & \textbf{EBTL (0.7)} & \textbf{AA} & \textbf{JSRL} & \textbf{KSRL} & \textbf{Fine-tuning} \\
\midrule
96\% & $41.2\!\pm\!4.9$ & $\mathbf{46.1\!\pm\!4.5}$ & $40.3\!\pm\!4.7$ & $6.1\!\pm\!5.6$ & $-3.2\!\pm\!10.3$ & $17.5\!\pm\!6.7$ & $-32.2\!\pm\!12.9$ \\
70\% & $26.8\!\pm\!4.6$ & $\mathbf{31.2\!\pm\!5.7}$ & $27.0\!\pm\!5.4$ & $17.0\!\pm\!4.3$ & $24.8\!\pm\!5.4$ & $14.2\!\pm\!6.1$ & $-104.3\!\pm\!8.6$ \\
45\% & $16.7\!\pm\!4.6$ & $14.3\!\pm\!4.3$ & $18.9\!\pm\!9.2$ & $13.9\!\pm\!6.2$ & $15.9\!\pm\!3.7$ & $\mathbf{21.5\!\pm\!5.7}$ & $-80.6\!\pm\!3.0$ \\
15\% & $9.2\!\pm\!5.9$  & $12.8\!\pm\!4.7$ & $13.7\!\pm\!4.9$ & $12.5\!\pm\!7.1$ & $5.0\!\pm\!6.4$  & $\mathbf{22.9\!\pm\!4.3}$ & $-110.5\!\pm\!4.4$ \\
\bottomrule

\end{tabular}%
}
\end{table*}


\subsection{Multi-Task Setting: Overcooked}

We create a single-agent variant of the popular Overcooked \citep{carroll2019utility} environment designed to evaluate multi-task learning. 
An overview is shown in Figure~\ref{fig:env_overview}b. 
This environment is both \emph{long-horizon} and \emph{high-dimensional}.
\textbf{Long-horizon.} Each timestep has one active recipe (onion, tomato, or fish soup). Completing it requires fetching and placing three matching ingredients into a pot, waiting 20 steps, retrieving a dish, and delivering the soup to the serving station. A new recipe is sampled after each delivery.
\textbf{High-dimensional.} 
The state space is combinatorial, due to randomized placement of ingredient dispensers, pots, and serving stations; pot contents and cook-timer values; counter inventories; held objects; agent orientations; and so on.  
Together these factors yield over \(10^{12}\) states even in the simplest layout, making explicit visitation counting infeasible.
(Appendix~\ref{app:OvercookedSpace}).
We evaluate two rooms of increasing complexity, a simple and a ring-shaped room. 
Observations encode the active recipe, allowing the agent to distinguish tasks.
Covariate shift arises from recipe changes, which alter the required soup, and layout changes, reflected in differences in available ingredient dispensers between source and target.

For each room, we construct three transfer setups with progressively increasing shift between teacher and student environments, quantified by the KL divergence between their energy score distributions (Figure~\ref{fig:cook_results}). 
\textbf{(1) Recipe Shift (2 $\Rightarrow$ 3):} 
Both environments have all three ingredients. Source requires onion + tomato soup; target requires onion + tomato + fish soup. This is a pure recipe shift (one additional ingredient).
\textbf{(2) Recipe Shift (1 $\Rightarrow$ 3):} 
Both environments have all three ingredients. Source requires onion soup; target requires onion + tomato + fish soup. This is a larger recipe shift (two additional ingredients).
\textbf{(3) Recipe + Layout Shift (2 $\Rightarrow$ 2):} 
Source has onions and tomatoes and requires onion + tomato soup. Target has tomatoes and fish and requires tomato + fish soup. This combines both recipe shift (different recipe) and layout shift (different available ingredients).

\textbf{EBTL maintains positive transfer under increasing covariate shift, consistently outperforming all baselines.} 
As covariate shift between the source and target environments increases, transfer becomes more challenging.
This is evident in the slower convergence from Recipe (2 $\Rightarrow$ 3) to Recipe (1 $\Rightarrow$ 3) in the Ring Room (Figure~\ref{fig:cook_results}), where the teacher is only familiar with 1 rather than 2 recipes (out of 3 total).
Despite this, EBTL yields positive transfer performance by restricting guidance to states associated with recipes that the teacher has encountered during training.

\textbf{Shared layouts simplify OOD detection.}
When source and target share the same spatial layout, \textit{i.e.}, Recipe (2 $\Rightarrow$ 3) and Recipe (1 $\Rightarrow$ 3), covariate shift arises solely from recipe encoding, producing a clearly bimodal energy distribution that separates ID and OOD states (Figure~\ref{fig:cook_results}, rows 1 and 2). In contrast, under Recipe $+$ Layout (2 $\Rightarrow$ 2), the set of available ingredient dispensers differs between source and target, changing the objects present in the state. As a result, even states associated with familiar recipes may contain features rarely observed during teacher training, lowering ID energy scores and blurring the ID/OOD boundary. 

\begin{figure}[t]
    \centering
    \includegraphics[width=0.8\linewidth]{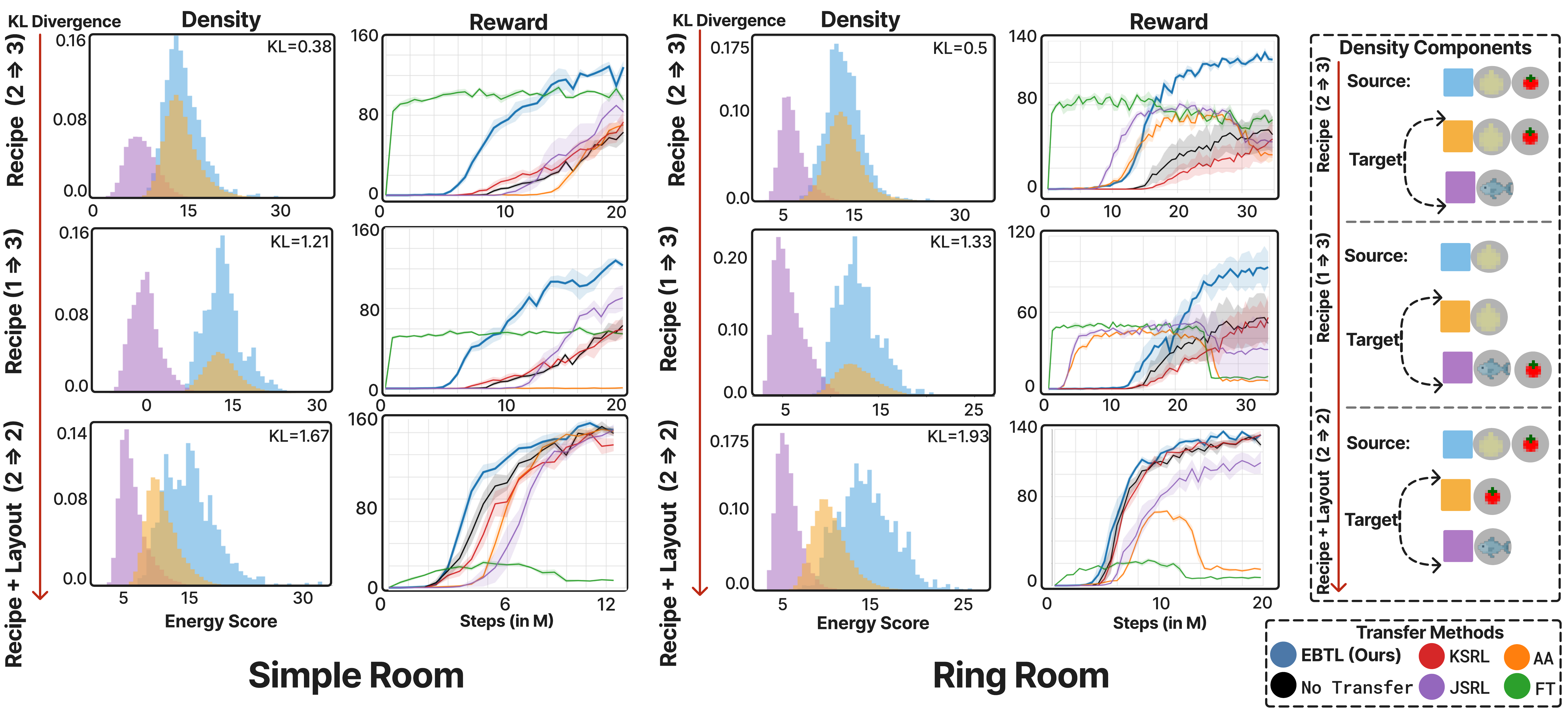}
    \caption{
    \textbf{Overcooked} (5 seeds). 
  \textbf{(Cols 1 \& 3)} Energy scores: source (blue) vs. target (orange + purple). Target is bimodal: one mode overlaps source (ID), one does not (OOD).
\textbf{(Cols 2 \& 4)} Learning curves: EBTL vs. baselines ($q=0.5$ for Recipe Shift, $q=0.2$ for Recipe + Layout Shift).
    }
\label{fig:cook_results}
\end{figure}

\textbf{Robustness to room complexity.}
EBTL achieves stable positive transfer across both Overcooked environments with high returns and low variance.
In contrast, baselines without selective guidance degrade as layout complexity increases. 
Action advising (AA) becomes unstable over training (Figure~\ref{fig:cook_results}), indicating that over-reliance on suboptimal advice hinders learning.
Fine-tuning often stalls in local minima under covariate shift: misaligned initialization biases early exploration toward low-reward regions, reinforcing bad value estimates and impeding recovery.

\textbf{Behavior under large covariate shifts.}
Under large covariate shifts, EBTL's gating offers less leverage and benefits from a more permissive threshold.
EBTL's gating is effective when two conditions hold: (i) the teacher assigns higher energy to transferable target states than to non-transferable ones, and (ii) the transferable mass lies inside the teacher's high-energy region (above threshold $\tau$). When both hold, $q$ can be set anywhere in a broad range and useful guidance is admitted. The $2\!\to\!2$ Overcooked transfer is the boundary case: condition (i) still holds visibly, but condition (ii) breaks because both the recipe and the kitchen layout change between source and target. The energy-score histograms in Figure~\ref{fig:cook_results} make the cause visible: the $2\!\to\!2$ target distribution (orange and purple) sits substantially to the left of the teacher's in-distribution mass (blue), so nearly every target state is classified as out-of-distribution. We therefore configure the $2\!\to\!2$ transfer with $q = 0.2$ rather than the $q = 0.5$ used in all other Overcooked settings: at $q = 0.5$ the threshold is calibrated against the in-distribution mass and almost no states clear it, leaving the teacher essentially silent; relaxing to $q = 0.2$ admits a useful fraction of guidance. This pattern generalizes: as the covariate shift between source and target becomes large, $q$ must be lowered to admit any guidance at all. EBTL's selective gating is therefore effective at handling covariate shift, but only up to a point. Once the target distribution barely overlaps with the source, the separation between transferable and non-transferable states collapses, and the selectivity that benefits EBTL in moderate-shift regimes provides less leverage. We view this as a known boundary of the method.

\section{Conclusion and Future Directions}
We introduced energy-based transfer learning (EBTL), which improves sample efficiency by selectively issuing teacher guidance based on an energy-based familiarity signal, without requiring additional network training or handcrafted OOD detectors. Empirically, EBTL consistently outperforms prior baselines across both single-task and multi-task transfer settings. More broadly, this work provides an initial step toward understanding how OOD detection can support transfer in reinforcement learning by linking energy to state-visitation patterns under on-policy training. A key limitation is that this connection may weaken in off-policy settings, where visitation no longer aligns with policy updates. Future work may explore more effective OOD detection methods for improved distributional separation, as well as extensions to multi-teacher settings where guidance is selected dynamically based on estimated familiarity.

\section*{Acknowledgements}
We thank the anonymous reviewers for their careful reading and constructive feedback, and Zhengyuan Li (Purdue University) for helpful discussions.

\newpage
\bibliography{collas2026_conference}
\bibliographystyle{collas2026_conference}

\newpage
\appendix
\section*{Appendix}
\input{appendix}

\end{document}

%% file: math_commands.tex
\usepackage{amsmath,amsfonts,bm}

\def\eqref#1{equation~\ref{#1}}

\def\1{\bm{1}}

\DeclareMathAlphabet{\mathsfit}{\encodingdefault}{\sfdefault}{m}{sl}
\SetMathAlphabet{\mathsfit}{bold}{\encodingdefault}{\sfdefault}{bx}{n}



%% file: appendix.tex
\renewcommand{\thesubsection}{\Alph{subsection}}

\subsection{Supporting Proof}
\subsubsection{Proof of Proposition~\ref{prop:onpolicy-id}}
\label{app:proof-onpolicy}
\begin{proof}
Let $\pi_\theta(a \mid s) = \mathrm{softmax}(f_\theta(s)/T)_a$ with 
logits $f_\theta(s) \in \mathbb{R}^K$ for $K$ actions and temperature 
$T > 0$, and write $\pi_j = \pi_\theta(j \mid s)$ for brevity. 
Let $A(s,a) = Q(s,a) - V(s)$ denote the advantage of action $a$ in 
state $s$, where $Q(s,a)$ and $V(s)$ are the action-value and 
state-value functions, respectively.

\paragraph{Logit update.}
A policy gradient step on $A(s,a)\log\pi_\theta(a \mid s)$ for a 
sampled action $a$ updates each logit by
\[
    \Delta f_j = \frac{\eta\,A(s,a)}{T}
    \bigl[\mathbf{1}_{j=a} - \pi_j\bigr],
\]
where $\eta > 0$ is the step size and $\mathbf{1}_{j=a}$ is the 
indicator function, increasing the logit of $a$ relative to all 
others when $A(s,a) > 0$.

\paragraph{Energy change.}
The free energy and score are defined as
\[
    E_\theta(s) = -T \log \sum_{j=1}^{K} e^{f_\theta(s)_j / T}, 
    \qquad 
    \phi_\theta(s) = -E_\theta(s).
\]
Since $\frac{\partial E_\theta(s)}{\partial f_j} = -\pi_j$, the 
first-order change in the score is
\begin{align*}
    \Delta \phi_\theta(s) 
    &= -\Delta E_\theta(s) \\
    &= -\sum_{j=1}^{K} \frac{\partial E_\theta(s)}{\partial f_j} 
       \cdot \Delta f_j \\
    &= \frac{\eta\,A(s,a)}{T} \sum_{j=1}^{K} 
       \pi_j \bigl[\mathbf{1}_{j=a} - \pi_j\bigr] \\
    &= \frac{\eta\,A(s,a)}{T} 
       \left[ \pi_a - \sum_{j=1}^{K} \pi_j^2 \right].
\end{align*}
The sign of $\Delta\phi_\theta(s)$ depends on $A(s,a)$ and the 
bracket term $\bigl[\pi_a - \sum_j\pi_j^2\bigr]$.

\paragraph{Non-negativity towards convergence.}
Let $a^* = \arg\max_a Q(s,a)$ denote the optimal action. 
Towards convergence, the policy concentrates mass on $a^*$, so we 
condition on the high-probability event $a = a^*$, under which 
$\pi_{a^*} = \max_j \pi_j$.

For the bracket term, $\pi_{a^*} \ge \pi_j$ for all $j$ implies
\[
    \sum_{j=1}^{K} \pi_j^2 
    \;\le\; \pi_{a^*} \sum_{j=1}^{K} \pi_j 
    \;=\; \pi_{a^*},
\]
where the last equality uses $\sum_j \pi_j = 1$, and hence 
$\bigl[\pi_{a^*} - \sum_j\pi_j^2\bigr] \ge 0$.

For the advantage term, since $a^* = \arg\max_a Q(s,a)$,
\[
    A(s,a^*) 
    \;=\; Q(s,a^*) - V(s) 
    \;=\; \sum_{a \neq a^*} \pi(a\mid s)
          \bigl[Q(s,a^*) - Q(s,a)\bigr] 
    \;\ge\; 0,
\]
as $Q(s,a^*) \ge Q(s,a)$ for all $a$ by definition of $a^*$.

Therefore, conditioning on $a = a^*$, both factors are 
non-negative and $\Delta\phi_\theta(s) \ge 0$. Since EBTL assumes 
the teacher has converged prior to transfer (Assumption~1), this 
condition is satisfied for the teacher at deployment time, and 
policy gradient updates reliably increase the energy score on 
visited states.
\end{proof}
\subsection{Training Details}
\subsubsection{GridWorld}
\paragraph{Reward Structure and Action Masking.}
In the MiniGrid experiments, agents are trained under a sparse reward setting: a reward of 1 is given only when the agent successfully reaches the goal location. No shaped or intermediate rewards are provided, making the task highly exploration-dependent. To mitigate the resulting challenge and accelerate learning, we apply action masking to dynamically restrict the agent’s action space based on its immediate environment. The action mask disables irrelevant or invalid actions at each timestep: (1) the \textit{forward} action is masked out if the agent is facing a wall, preventing redundant collisions; (2) the \textit{pickup} action is disabled unless the agent is directly facing a key; (3) the \textit{toggle} action is masked out unless the agent is facing a door; (4) the \textit{drop} action is always disabled, as object dropping is unnecessary in our tasks; and (5) the \textit{done} action is permanently disabled, since it is not used in our environments. This selective pruning of the action space reduces the likelihood of unproductive behavior and enables the agent to focus on learning goal-directed policies more effectively.

\paragraph{Teacher Training.}
In both experimental setups, we train two variants of the teacher policy using standard Proximal Policy Optimization (PPO) in the source environment: one with the energy-based loss and one without. 
For the teacher trained with energy loss, the \( m_{\text{in}} \) and \( m_{\text{out}} \) are set to 10 and 15 respectively. These values are chosen arbitrarily, as the separation between energy distributions is insensitive to the exact threshold choice (see Section~\ref{sec:energy-sensitivity}).
The training follows a consistent set of hyperparameters, as detailed in the next section. For the \textit{unlocked-to-locked} environment, 800K-step checkpoints are selected from both training variants. For the \textit{alternating-goal room} environment, 200K-step checkpoints are used. 
In the language-conditioned setting, 1.6M-step checkpoints are used for \textit{unlocked-to-locked}, and 2M-step checkpoints for \textit{alternating-goal room}.

\paragraph{Student Training.}
For each target task, we first train a student policy from scratch using standard PPO without any transfer to establish baseline performance. All experiments in the MiniGrid setups are conducted with 10 random seeds to ensure robustness. Within each domain, the student and teacher policies share the same model architecture.

\subsubsection{Overcooked-AI}
\paragraph{State Space Enumeration of Overcooked}
\label{app:OvercookedSpace}
Consider the Simple Layout illustrated in Figure~\ref{fig:env_overview}. 
The grid has dimensions $4 \times 5$. 
The interior region contains $6$ traversable tiles where the agent can move. 
The exterior non-corner tiles are reserved for environment objects (ingredient dispensers, pot, dish dispenser, serving counter), and their placements can be randomized. 

We explicitly count states under the \emph{single-agent, lossless} encoding generated by
\texttt{lossless\_state\_encoding\_single\_agent()} in \texttt{OvercookedGridWorld}:
The resulting state space is a combinator of the following factors. 
\begin{itemize}
    \item \textbf{Kitchen-layout permutations.} 
    Among the $14$ exterior tiles, $8$ are eligible for randomization. We must place $6$ \emph{distinct} stations
(\emph{onion}, \emph{tomato}, \emph{fish} dispensers; \emph{server}; \emph{pot}; \emph{dish} stack), leaving the remaining $2$ tiles as empty counters.
The number of distinct assignments is
\[
\binom{8}{6}\,6! \;=\; {}_{8}P_{6} \;=\; \frac{8!}{2!} \;=\; 20{,}160.
\]
\item \textbf{Agent position and orientation}: The agent may occupy any of the $6$ interior tiles and face one of $4$ directions (north, south, east, west), for a total of $6 \times 4 = 24$ possibilities.
\item \textbf{Urgency flag}: Binary indicator with $2$ possibilities, set to $1$ if the remaining horizon is less than $40$, and $0$ otherwise.
\item \textbf{Active recipe}: $3$ possibilities, indicating the current recipe type (onion, tomato, or fish).
\item \textbf{Pot state (mode + contents/recipe/timer)}:
Idle with $k\in\{0,1,2\}$ ingredients (order irrelevant): $\sum_{k=0}^{2}\binom{k+2}{2}=1+3+6=10$.
With $3$ ingredients, the pot is \emph{cooking}: there are $\binom{5}{3}=10$ recipe multisets, each with a remaining time in $\{1,\dots,20\}$, giving $10\times 20=200$ states.
When cooking finishes, it is \emph{done} with one of the same $10$ recipes.
Total $= 10 + 200 + 10 = 220$.
\item \textbf{Agent hand}: The agent may hold
(i) nothing, 
(ii) an empty dish, 
(iii) a finished soup (all $3$ slots filled with a combination of onion, tomato, and fish), 
or a single raw ingredient (onion, tomato, or fish).  
The number of distinct soup types is $\binom{3+3-1}{3}=\binom{5}{3}=10$.  
Hence the total possibilities are
\[
1 \;(\text{nothing}) \;+\; 1 \;(\text{empty dish}) \;+\; 10 \;(\text{soup types}) \;+\; 3 \;(\text{single ingredient})
= 15.
\]
\item \textbf{Counter items (2 exterior counters)}:
Each counter has $15$ options (empty; three ingredients; empty dish; ten soup types), giving $15^2$ overall.
\end{itemize} 
Multiplying the independent factors above gives:
\begin{align*}
|\mathcal{S}| \;=&\;
\underbrace{20{,}160}_{\text{layout}}
\times \underbrace{24}_{\text{agent pos/orient}}
\times \underbrace{2}_{\text{urgency}}
\times \underbrace{3}_{\text{active recipe}} 
\times \underbrace{220}_{\text{pot state}}
\times \underbrace{15}_{\text{agent hand}}
\times \underbrace{15^2}_{\text{two counters}} \\
=&\; 2{,}155{,}507{,}200{,}000
\;\approx\; 2.16 \times 10^{12}.
\end{align*}
This already conservative count highlights why explicit state-visitation is infeasible for the teacher model; more complex layouts such as the \emph{Ring} further enlarge the state space.

\paragraph{Reward Structure.}
In all Overcooked setups, no action masking is applied. Instead, shaped rewards are introduced to facilitate the training process. A shaped reward of 3 is given when the correct ingredient is added to a pot. An additional reward of 3 is awarded when a dish is picked up—provided there are no dishes already on the counter and the soup is either cooking or completed. A reward of 5 is granted when the soup is picked up. Furthermore, a shaped reward of 3 is given upon delivering the soup, regardless of whether it matches the currently active recipe. All shaped rewards follow a predefined linear decay schedule. In contrast, a sparse reward of 20 is awarded when the delivered soup matches the active recipe; this reward does not decay over time.

\paragraph{Teacher Training.}
In all Overcooked setups, teacher policies are trained in the source environment using standard Proximal Policy Optimization (PPO) with hyperparameters described in the following section. For each setup and source-target configuration, a specific checkpoint is selected to serve as the teacher for transfer. The table below lists the selected training step (in environment steps) corresponding to each teacher checkpoint.

\begin{table}[ht]
\centering
\caption{Selected teacher checkpoints (in environment steps) for each Overcooked setup and source-target configuration.}
\label{tab:teacher_checkpoints}
\begin{tabular}{lccc}
\toprule
\textbf{Setup} & \textbf{Recipe (2 $\rightarrow$ 3)} & \textbf{Recipe (1 $\rightarrow$ 3)} & \textbf{Recipe + Layout (2 $\rightarrow$ 2)} \\
\midrule
Simple Room & 19{,}008{,}000 & 9{,}004{,}800 & 12{,}000{,}000 \\
Ring Room   & 2{,}400{,}000  & 10{,}003{,}200 & 18{,}000{,}000 \\
\bottomrule
\end{tabular}
\end{table}

\paragraph{Student Training.}
In all Overcooked setups, student policies are trained in the target environment using PPO under a fixed transfer horizon. For the teacher trained with energy loss, the $m_{in}$ and $m_{out}$ are set to 12 and 14 respectively. The training is conducted using consistent hyperparameters, as detailed in the next section. All experiments are repeated with 5 random seeds to ensure stability and reproducibility. The transfer horizon varies depending on the setup and source-target configuration. The table below summarizes the number of environment steps used during student training for each case:
\begin{table}[ht]
\centering
\caption{Transfer horizons (in millions of environment steps) used for student training in each Overcooked setup and configuration. Each experiment is run with 5 random seeds.}
\label{tab:student_transfer_horizons}
\begin{tabular}{lccc}
\toprule
\textbf{Setup} & \textbf{Recipe (2 $\rightarrow$ 3)} & \textbf{Recipe (1 $\rightarrow$ 3)} & \textbf{Recipe + Layout (2 $\rightarrow$ 2)} \\
\midrule
Simple Room & 20M & 20M & 12M \\
Ring Room   & 35M & 35M & 20M \\
\bottomrule
\end{tabular}
\end{table}

\subsection{Hyperparameters} 
\label{app:hyperparameters}
All experiments use teacher and student policies trained with TorchRL's proximal policy optimization (PPO)~\citep{bou2023torchrl}.
\subsubsection{GridWorld}
All experimental setups in GridWorld are trained using a fixed set of PPO hyperparameters, summarized in Table \ref{tab:gridworld_hyperparams}. These settings remain consistent across all teacher and student training runs within the domain. For the language-conditioned variant, we use a learning rate of 0.0001, while keeping all other hyperparameters identical. 

\begin{table}[ht]
\centering
\caption{Hyperparameters used for all GridWorld experiments.}
\label{tab:gridworld_hyperparams}
\begin{tabular}{lc}
\toprule
\textbf{Hyperparameter} & \textbf{Value} \\
\midrule
Learning rate & 0.0005 \\
Discount factor (\(\gamma\)) & 0.9 \\
GAE lambda (\(\lambda\)) & 0.8 \\
Policy clip parameter & 0.2 \\
Value function clip parameter & 10.0 \\
Value loss coefficient & 0.5 \\
Entropy coefficient & 0.01 \\
Train batch size & 256 \\
SGD minibatch size & 128 \\
Number of SGD iterations & 4 \\
Number of parallel environments & 8 \\
Normalize advantage & False \\
\bottomrule
\end{tabular}
\end{table}

\subsubsection{Overcooked-AI}
All Overcooked experiments use a shared set of core PPO hyperparameters, listed in \cref{tab:overcooked_shared_hyperparams}. These settings are consistent across teacher and student training. However, the learning rate and reward shaping horizon vary depending on the layout and recipe configuration, summarized in \cref{tab:overcooked_specific_hyperparams}. We use the following notation: O = Onion, T = Tomato, F = Fish, OT = Onion + Tomato, TF = Tomato + Fish, OTF = Onion + Tomato + Fish.

\begin{table}[ht]
\centering
\begin{minipage}[t]{0.3\textwidth}
\centering
\caption{Shared PPO hyperparameters across all Overcooked experiments.}
\label{tab:overcooked_shared_hyperparams}
\begin{tabular}{lc}
\toprule
\textbf{Hyperparameter} & \textbf{Value} \\
\midrule
Discount factor (\(\gamma\)) & 0.99 \\
GAE lambda (\(\lambda\)) & 0.6 \\
KL coeff & 0.0 \\
Reward clipping & False \\
Clip parameter & 0.2 \\
VF clip parameter & 10.0 \\
VF loss coeff & 0.5 \\
Entropy coeff & 0.1 \\
Train batch size & 9600 \\
SGD minibatch size & 1600 \\
SGD iterations & 8 \\
Parallel envs & 24 \\
Normalize advantage & False \\
\bottomrule
\end{tabular}
\end{minipage}%
\hfill
\begin{minipage}[t]{0.65\textwidth}
\centering
\caption{Setup-specific learning rates and reward shaping horizons.}
\label{tab:overcooked_specific_hyperparams}
\begin{tabular}{llcc}
\toprule
\textbf{Layout} & \textbf{Config} & \textbf{LR} & \textbf{Horizon} \\
\midrule
\multirow{5}{*}{Simple} 
& Recipe (O)        & 0.001  & 8M \\
& Recipe (OT)       & 0.001  & 15M \\
& Recipe (OTF)      & 0.001  & 25M \\
& Recipe + Layout (OT)       & 0.001  & 10M \\
& Recipe + Layout(TF)       & 0.001  & 10M \\
\midrule
\multirow{5}{*}{Ring} 
& Recipe (O)        & 0.0006 & 10M \\
& Recipe (OT)       & 0.0006 & 20M \\
& Recipe (OTF)      & 0.0006 & 30M \\
& Recipe + Layout (OT)       & 0.0006 & 15M \\
& Recipe + Layout (TF)       & 0.0006 & 15M \\
\bottomrule
\end{tabular}
\end{minipage}
\end{table}

\subsection{Baselines}
\label{app:baselines}
\subsubsection{Baseline Implementation}
To ensure a fair comparison, we implement all baselines and our method within a unified training pipeline. All experiments use teacher and student policies trained with TorchRL’s proximal policy optimization (PPO)~\citep{bou2023torchrl}.
For methods where the teacher directly overwrites the student’s action (e.g., AA and JSRL~\citep{uchendu2023jump}), the resulting actions are off-policy with respect to the student. To ensure a valid comparison, we apply importance sampling correction to account for this mismatch. This allows us to isolate the effect of the teacher’s guidance quality, rather than confounding factors introduced by off-policy data.
For KSRL~\citep{schmitt2018kickstarting}, which introduces an additional imitation objective, we incorporate the auxiliary loss on top of the standard PPO loss, following the formulation in the original work.

\subsubsection{Hyperparameter Tuning for Baselines}
\begin{figure}[t]
    \centering
    \includegraphics[width=\linewidth]{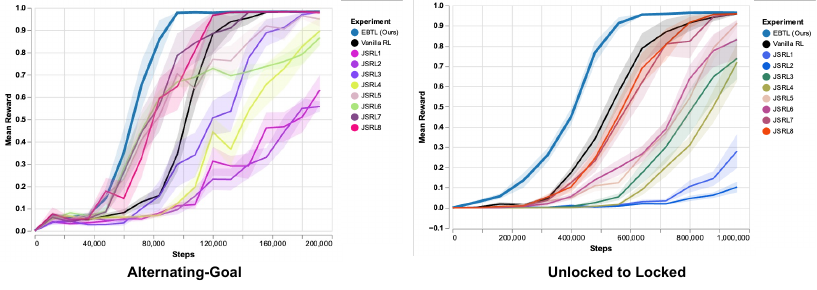}
    \caption{\textbf{10 seeds. JSRL hyperparameter tuning in MiniGrid.}
    }
    \label{fig:JSRL}
\end{figure}

\begin{figure}[t]
    \centering
    \includegraphics[width=\linewidth]{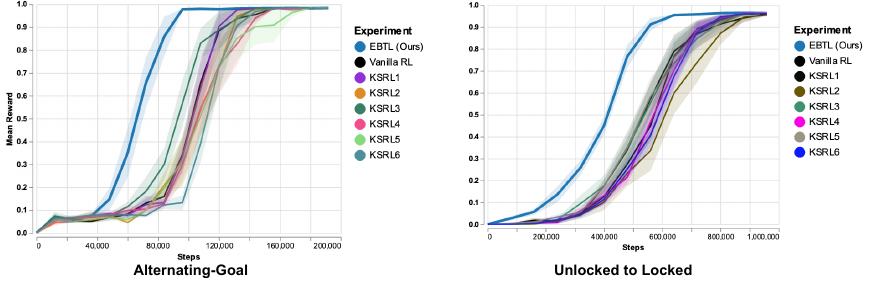}
    \caption{\textbf{10 seeds. KSRL hyperparameter tuning in MiniGrid.}
    }
    \label{fig:KSRL}
\end{figure}

\paragraph{JSRL Experiments Details.}
JumpStart RL (JSRL)~\citep{uchendu2023jump} is a curriculum-based transfer method with three key hyperparameters: the number of stages, the tolerance for stage transition, and the initial number of advising steps. The curriculum progresses based on the student’s evaluation performance rather than training steps alone: the policy is periodically evaluated, and a stage transition is triggered when the moving average of recent evaluations is within a specified tolerance of the previous best performance. 

JSRL does not introduce additional trainable networks, but it requires repeated policy evaluation throughout training, increasing overall wall-clock time. In our experiments, we use 50 evaluation timepoints over the full training horizon and evaluate JSRL across a grid of its three key hyperparameters: the number of curriculum stages, the tolerance for stage transition, and the initial number of advising steps. Specifically, we sweep stages in $\{10, 30\}$, tolerance in $\{0, 0.05\}$, and initial advising steps in $\{50, 200\}$, resulting in eight configurations (JSRL1--JSRL8), as shown in Figure~\ref{fig:JSRL}. 
We note that JSRL exhibits large variance across hyperparameter settings. We therefore report JSRL3 in the main paper as it represents the average performance case across configurations.
\begin{itemize}
    \item \textbf{JSRL1}: stages = 30, initial advising steps = 200, tolerance = 0.05
    \item \textbf{JSRL2}: stages = 30, initial advising steps = 200, tolerance = 0.0
    \item \textbf{JSRL3}: stages = 10, initial advising steps = 200, tolerance = 0.05
    \item \textbf{JSRL4}: stages = 10, initial advising steps = 200, tolerance = 0.0
    \item \textbf{JSRL5}: stages = 30, initial advising steps = 50, tolerance = 0.05
    \item \textbf{JSRL6}: stages = 30, initial advising steps = 50, tolerance = 0.0
    \item \textbf{JSRL7}: stages = 10, initial advising steps = 50, tolerance = 0.05
    \item \textbf{JSRL8}: stages = 10, initial advising steps = 50, tolerance = 0.0
\end{itemize}

\paragraph{JSRL Results.} Across all settings, JSRL performs poorly.
This behavior is consistent with its design: since JSRL schedules guidance purely based on training progress and does not consider whether the current state lies within the teacher’s training distribution, it may continue to provide guidance in mismatched states, leading to degraded transfer performance.

\paragraph{KSRL Hyperparameter Tuning.}
Kickstarting RL (KSRL)~\citep{schmitt2018kickstarting} has two key hyperparameters: the imitation weight $\lambda$, which controls the strength of the auxiliary distillation loss, and its decay schedule over training. We follow the original paper and evaluate $\lambda \in \{1, 2\}$.

\paragraph{KSRL Hyperparameter Configurations.}
Kickstarting RL (KSRL)~\citep{schmitt2018kickstarting} has two key hyperparameters: the imitation weight $\lambda$ and its decay schedule. The original work considers both linear decay and Population-Based Training (PBT). As reported in Table 2 of \citep{schmitt2018kickstarting}, PBT achieves performance comparable to linear decay, while certain linear schedules that terminate imitation early can even outperform it. To ensure consistency across methods, we adopt linear decay in all experiments.
For AlternatingGoal (200K training horizon), the imitation loss decays to zero at 30\%, 50\%, and 70\% of training. For Unlocked-to-Locked (1M horizon), the decay ends at 25\%, 50\%, and 75\% of training. These schedules match those used in our method’s hyperparameter tuning (Figure~\ref{fig:grid_results}). We evaluate $\lambda \in \{1, 2\}$, resulting in six configurations (KS1--KS6). 
As all configurations exhibit similar performance (Figure~\ref{fig:KSRL}), we randomly choose the setting with $\lambda = 1.0$ in the main paper as the representative KSRL configuration.
\begin{itemize}
    \item \textbf{KS1}: $\lambda = 1.0$, decay ends at 30\% (AlternatingGoal) / 25\% (Unlocked-to-Locked)
    \item \textbf{KS2}: $\lambda = 1.0$, decay ends at 50\% (both settings)
    \item \textbf{KS3}: $\lambda = 1.0$, decay ends at 70\% (AlternatingGoal) / 75\% (Unlocked-to-Locked)
    \item \textbf{KS4}: $\lambda = 2.0$, decay ends at 30\% (AlternatingGoal) / 25\% (Unlocked-to-Locked)
    \item \textbf{KS5}: $\lambda = 2.0$, decay ends at 50\% (both settings)
    \item \textbf{KS6}: $\lambda = 2.0$, decay ends at 70\% (AlternatingGoal) / 75\% (Unlocked-to-Locked)
\end{itemize}

\paragraph{KSRL Results.}
As shown in Figure~\ref{fig:KSRL}, KSRL consistently underperforms EBTL across both transfer settings. On average, KSRL achieves performance comparable to training from scratch (Vanilla RL) and fails to provide meaningful transfer gains. Across different hyperparameter configurations (KS1--KS6), KSRL exhibits limited sensitivity to the choice of $\lambda$ and decay schedule, and none of the settings reach the performance level of EBTL. This suggests that uniformly applying imitation, even with a decaying schedule, is insufficient under distribution shift, as it does not account for whether the teacher’s behavior is reliable in the current state.

\subsection{Continuous Action Space}
\label{app:cont-energy}

EBTL also applies to continuous control: the action space can be discretized into bins (as in OpenVLA \citep{kim2024openvla}), or an energy score can be defined directly in the continuous domain. Appendix~\ref{app:CEBTL} provides the full continuous formulation (C-EBTL).
In this case, the teacher policy outputs a mean and diagonal covariance for each state $s$, and the energy score is derived from the log-determinant of the covariance: $\phi(s) = -E(s) = \tfrac12 \log|\Sigma(s)| = \sum_{i=1}^D \log \sigma_i(s)$, where $\sigma_i(s)$ is the standard deviation of the $i$-th action dimension.
We evaluate on Meta-World\citep{yu2020meta}, a suite of continuous-control manipulation tasks for a simulated Sawyer arm. Both state and action spaces are continuous; the action space is four-dimensional, consisting of three Cartesian end-effector displacements plus a gripper command. We consider two settings: window manipulation (Window Open, Window Close) and button manipulation (Button Press, Button Press Down). In the window setting the teacher is trained on Window Open; in the button setting the teacher is trained on Button Press. In each setting, the \emph{student must learn both tasks in that pair}. 
\begin{figure}[t]
    \centering
    \includegraphics[width=\linewidth]{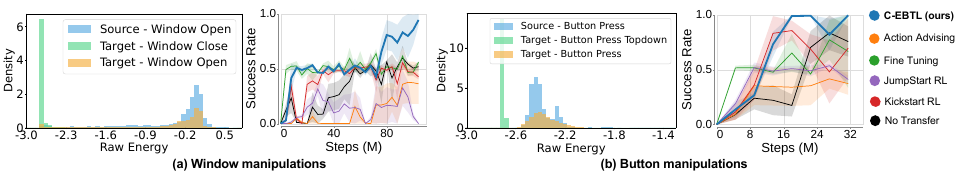}
\caption{Results (3 seeds). In each environment, the left column shows the empirical density of the teacher’s state energy $\phi$. Blue traces the source rollouts used to train the teacher. Orange and green are target rollouts during transfer and typically form a bimodal pattern: one component overlaps the source shared task (in distribution), while the other occupies a separate region (non-shared task, out of distribution).
$q$ is set to 0.3 in both settings.
}
\label{fig:Results_cont}
\end{figure}

\paragraph{Energy score separates ID from OOD.}
As shown in Figure~\ref{fig:Results_cont}, in both environments the distributions are clearly bimodal: states from the unseen target subtask cluster at a lower mode, while states from the teacher’s training subtask occupy a higher one. Across both settings, the OOD mode concentrates near \(-2.8\). This follows from the floor we impose on the policy standard deviations for numerical stability in the network, \(\sigma_i \ge 0.5\), which implies \(\phi(s)=\sum_{i=1}^4 \log \sigma_i(s) \approx -2.8\). In contrast, the ID energies differ between settings, consistent with Proposition~\ref{prop:onpolicy-cont}: when a state admits many compatible actions, on-policy updates raise \(\phi_T(s)\). Window manipulation allows more compatible actions (approach the handle, then slide to the target), so its ID energies are higher. Button pressing requires precise alignment and actuation, so fewer actions are compatible and the ID energies remain lower than in window tasks, but still above the OOD mode. 

\subsection{Energy formulation on Continuous Energy Score}
\label{app:CEBTL}
\subsubsection{Continuous Control}
In robotic control, the action space \(\mathcal{A}\subseteq\mathbb{R}^n\) is continuous and \(n\)-dimensional, with each component \(a_i\) corresponding to a separate control command.  A common parameterization is a diagonal Gaussian policy, wherein the actor network outputs two vectors \(\mu(s) = [\mu_1(s), \dots, \mu_n(s)]^\top\) and \(\sigma(s) = [\sigma_1(s), \dots, \sigma_n(s)]^\top\), and the policy density is given by
\begin{equation}\label{eq:Action}
p(a\mid s)
= \frac{\exp\!\bigl(-\tfrac12\,(a - \mu(s))^\top\Sigma(s)^{-1}(a - \mu(s))\bigr)}
       {(2\pi)^{n/2}\,\lvert\Sigma(s)\rvert^{1/2}}
\end{equation}
where \(\Sigma(s)=\mathrm{diag}(\sigma_1^2(s),\dots,\sigma_n^2(s))\).  Equivalently,
\(\pi_\theta(a\mid s)=\mathcal{N}\bigl(a;\,\mu(s),\,\Sigma(s)\bigr)\),
so each action dimension is sampled as \(a_i\sim\mathcal{N}(\mu_i(s),\sigma_i^2(s))\).

An energy‐based model defines a joint energy function
$E(s,y)\colon \mathcal{S}\times\mathcal{Y}\to\mathbb{R}$,
which assigns a scalar energy to each state–output pair \((s,y)\)\citep{lecun2006tutorial}.  From this one obtains the conditional Gibbs distribution
\begin{equation}\label{eq:EBM-full}
p(y\mid s)
= \frac{\exp\bigl(-E(s,y)\bigr)}{Z(s)}, \quad
Z(s) = \int_{\mathcal{Y}}\exp\bigl(-E(s,y')\bigr)\,\mathrm{d}y'.
\end{equation}
where \(Z(s)\) is the state‐conditional partition function.  
 Here, we derive a continuous-action energy formulation for diagonal Gaussian policies from Equation~\ref{eq:Action} and Equation~\ref{eq:EBM-full}.
We define the joint energy as
\begin{equation}\label{eq:ESA}
E(s,a)
=\tfrac12\,(a-\mu(s))^\top\,\Sigma(s)^{-1}\,(a-\mu(s)).
\end{equation}
The corresponding state‐conditional partition function is
\begin{equation}\label{eq:Zs}
Z(s)
=\int_{\mathbb{R}^D}\exp\bigl(-E(s,a')\bigr)\,\mathrm{d}a'
=(2\pi)^{D/2}\,\lvert\Sigma(s)\rvert^{1/2}.
\end{equation} 

Connecting Equation~\ref{eq:ESA}, Equation~\ref{eq:Zs}, and Equation~\ref{eq:EBM-full} and marginalizing out $a$ yields the Helmholtz free energy 
$
E(s)=-\log Z(s)
=-\frac{D}{2}\log(2\pi)\;-\;\tfrac12\log\lvert\Sigma(s)\rvert$.  
Since the additive constant \(-\tfrac{D}{2}\log(2\pi)\) does not affect ranking, we drop it and define the simplified energy score as the negative of free energy:
\begin{equation}\label{eq:stateEnergySimple}
\phi(s) = -E(s)
= \tfrac12\log\lvert\Sigma(s)\rvert
=\sum_{i=1}^D\log\sigma_i(s).
\end{equation}

\subsubsection{Energy Score as OOD detector}
\paragraph{Energy score in continuous action spaces.}
For a diagonal Gaussian policy 
$\pi_\theta(a \mid s)=\mathcal{N}(a;\,\mu_\theta(s),\Sigma_\theta(s))$ with 
$\Sigma_\theta(s)=\mathrm{diag}(\sigma_1^2(s),\dots,\sigma_D^2(s)) \succ 0$, 
we define the state energy score as
\begin{equation}
\phi_\theta(s)
= \tfrac{1}{2}\log |\Sigma_\theta(s)|
= \sum_{i=1}^D \log \sigma_i(s).
\end{equation}

\begin{proposition}[Monotonicity under advantage-weighted updates]
\label{prop:onpolicy-cont}
Fix a state $s$ and a diagonal Gaussian policy $\pi_\theta(a \mid s)$. 
Consider a fixed batch $\mathcal{D}_s$ collected under the current policy and the surrogate objective
\begin{equation}
\mathcal{L}_s(\theta)
=
\mathbb{E}_{a \sim \mathcal{D}_s}\big[\, w(s,a)\,\log \pi_\theta(a \mid s) \,\big],
\end{equation}
where $w(s,a) \ge 0$ and $W = \mathbb{E}[w] > 0$.

Define the advantage-weighted mean and variance along coordinate $i$:
\begin{equation}
\bar{a}_{w,i} = \frac{\mathbb{E}[w a_i]}{W}, 
\qquad
v_{w,i} = \frac{\mathbb{E}[w (a_i - \bar{a}_{w,i})^2]}{W}.
\end{equation}

Let $\theta^+$ denote the parameters after one gradient ascent step on $\mathcal{L}_s$ with step size $\eta > 0$. 
If
\begin{equation}
v_{w,i} \ge \sigma_i^2(s) \quad \forall i,
\qquad
v_{w,j} > \sigma_j^2(s) \text{ for some } j,
\end{equation}
then
\begin{equation}
\phi_{\theta^+}(s) > \phi_\theta(s).
\end{equation}
\end{proposition}

\begin{proof}
Parameterize $\theta_i = \log \sigma_i^2$. The gradients of $\mathcal{L}_s$ are
\begin{equation}
\frac{\partial \mathcal{L}_s}{\partial \mu_i}
=
\frac{W}{\sigma_i^2}(\bar{a}_{w,i} - \mu_i),
\qquad
\frac{\partial \mathcal{L}_s}{\partial \theta_i}
=
\frac{W}{2}
\left(
\frac{v_{w,i} + (\bar{a}_{w,i} - \mu_i)^2}{\sigma_i^2}
- 1
\right).
\end{equation}
Let
\begin{equation}
\Delta_i
=
\frac{v_{w,i} + (\bar{a}_{w,i} - \mu_i)^2}{\sigma_i^2}
- 1.
\end{equation}

After one gradient ascent step,
\begin{equation}
\theta_i^+ = \theta_i + \eta \frac{W}{2} \Delta_i
\quad \Longrightarrow \quad
\sigma_i^{2+} = \sigma_i^2 \exp\!\left(\eta \frac{W}{2} \Delta_i\right).
\end{equation}

Since $(\bar{a}_{w,i} - \mu_i)^2 \ge 0$, the condition
$v_{w,i} \ge \sigma_i^2$ for all $i$ and strict inequality for some $j$
implies
\begin{equation}
\Delta_i \ge 0 \ \forall i,
\qquad
\Delta_j > 0 \text{ for some } j.
\end{equation}
Thus $\sigma_i^{2+} \ge \sigma_i^2$ for all $i$, with strict increase for at least one coordinate. Therefore,
\begin{equation}
|\Sigma_{\theta^+}(s)| > |\Sigma_\theta(s)|
\quad \Longrightarrow \quad
\phi_{\theta^+}(s)
= \tfrac{1}{2} \log |\Sigma_{\theta^+}(s)|
>
\tfrac{1}{2} \log |\Sigma_\theta(s)|
= \phi_\theta(s).
\end{equation}
\end{proof}

\paragraph{Frequently visited states receive higher scores.}
We initialize the policy with log-standard-deviation parameters set to $0$ (i.e., $\sigma_i = 1$ for all coordinates), a standard practice in continuous-control  to stabilize early exploration~\citep{yu2020meta}.
At the beginning of training, the predicted mean $\mu(s)$ is typically inaccurate, leading to large residuals $(a - \mu(s))^2$ when actions are sampled. 
Since the variance update is driven by the squared deviation between sampled actions and the current mean, each update at state $s$ tends to increase the estimated variance when the observed spread of compatible actions is large. 
Under on-policy training, states that are visited more frequently are updated more often, and thus tend to accumulate larger variance estimates over time. 
As the energy score $\phi(s) = \tfrac{1}{2}\log|\Sigma(s)|$ increases monotonically with the variance, this leads to higher energy scores for frequently visited states. 
Furthermore, in continuous action spaces many tasks exhibit aleatoric uncertainty, where multiple actions can yield similar outcomes, resulting in a non-zero variance even after convergence. 
The final energy score therefore reflects the spread of compatible actions at state $s$, providing a practical proxy for state familiarity.


\subsection{Model Architecture}
\begin{figure}[ht]
    \centering
    \includegraphics[width=0.9\linewidth]{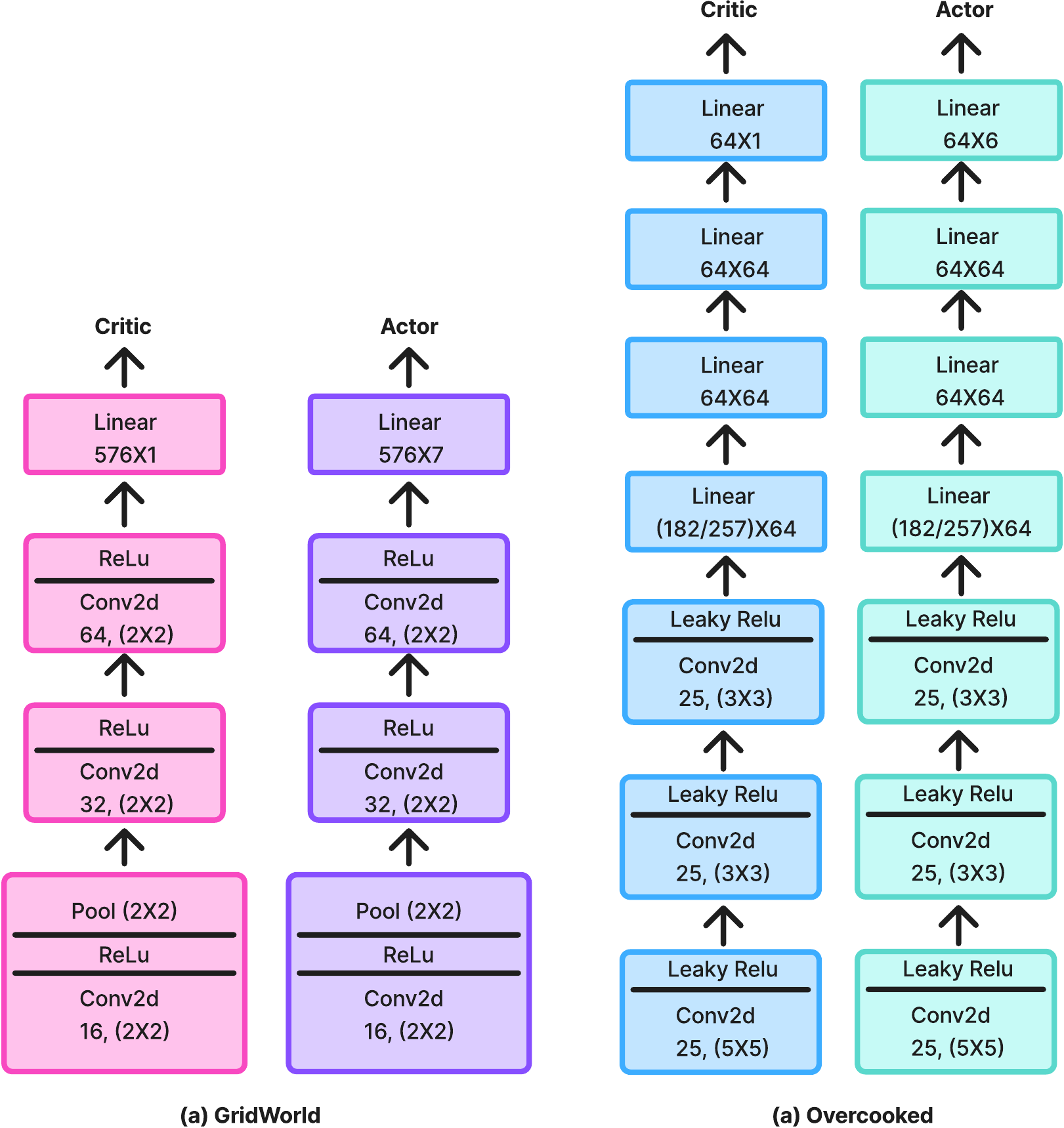}
    \caption{Actor-Critic architectures used in our experiments. (a) MiniGrid. (b) Overcooked.}
    \label{fig:model_architectures}
\end{figure}

\noindent All MiniGrid experiments share the same model architecture shown in Figure~\ref{fig:model_architectures}a. Similarly, all Overcooked experiments use the architecture in Figure~\ref{fig:model_architectures}b. Due to layout size differences in Overcooked, the dense layer input size is set to 182 for \textit{Simple} layouts and 257 for \textit{Ring} layouts.

For the language-conditioned MiniGrid experiments (Figure~7(a)), the visual backbone is slightly modified to incorporate task embeddings. In the \textit{unlocked-to-locked} setting, we use three convolutional layers with 32–64–96 channels and ReLU activations, producing an 864-dimensional flattened feature, followed by actor and critic heads of sizes 864→128→7 and 864→128→1, respectively. In the \textit{alternating-goal room} setting, a lighter backbone with 16–32–64 channels produces a 576-dimensional feature, followed by single-layer heads 576→7 (actor) and 576→1 (critic). Language conditioning is implemented using fixed 128-dimensional sentence embeddings extracted from \texttt{google/embeddinggemma-300m}, which are fused with the visual features via FiLM modulation prior to the final linear layer.

\subsection{Energy-Based Loss}
\label{app:energy-loss}
\subsubsection{Effect of Margin Hyperparameters on Separation}
\label{app:margin}

We evaluate whether varying the energy thresholds \( m_{\text{in}} \) and \( m_{\text{out}} \) affects the teacher’s ability to distinguish between false and true out-of-distribution (OOD) states. The energy loss used during training is defined over the energy score \( \phi(s) = -E(s) \) as:
\begin{align*}
\mathcal{L}_{\text{energy}} = & \; \mathbb{E}_{\mathbf{s}_{\text{in}} \sim \mathcal{D}_{\text{in}}^{\text{train}}} \left[ \left( \max\left(0, m_{\text{in}} - \phi(\mathbf{s}_{\text{in}}) \right) \right)^2 \right] 
 + \mathbb{E}_{\mathbf{s}_{\text{out}} \sim \mathcal{D}_{\text{out}}^{\text{train}}} \left[ \left( \max\left(0, \phi(\mathbf{s}_{\text{out}}) - m_{\text{out}} \right) \right)^2 \right].
\end{align*}

\paragraph{Experimental Setup}
Experiments are conducted in the \textit{GridWorld (unlocked-to-locked)} environment. During training, the in-distribution (ID) set consists of the most recent 3{,}000 frames collected from the agent’s own trajectory. The out-of-distribution (OOD) set is fixed and sampled from 100 episodes of a random policy in the target environment, where the agent is randomly initialized in any room at the start of each episode to ensure unbiased state coverage (rather than being constrained to the upper room).
We evaluate six combinations of \( (m_{\text{in}}, m_{\text{out}}) \) used in the energy regularization loss (defined over energy scores \( \phi(s) = -E(s) \)):
(10, 15), (5, 10), (15, 20), (10, 10), (15, 15), (12, 14).
Each configuration is trained with 5 random seeds using a shared PPO setup and evaluated at the 800{,}000-step checkpoint.

\paragraph{Sensitivity Evaluation Protocol.}
\label{sec:energy-sensitivity}
We assess whether the teacher consistently distinguishes between \textit{false OOD} states -- those similar to ID states and where guidance should be issued -- and \textit{true OOD} states -- those clearly out-of-distribution and where guidance should be withheld. Both sets are drawn from a fixed OOD dataset collected via a random policy in the target environment. For each \( (m_{\text{in}}, m_{\text{out}}) \) configuration, we compute the divergence between the energy score distributions of false and true OOD states across three training seeds using Jensen-Shannon divergence, total variation distance, Hellinger distance, and Kullback-Leibler (KL) divergence. To evaluate sensitivity, we apply one-way ANOVA and Kruskal-Wallis tests to determine whether this separation remains consistent across different regularization settings. A high p-value indicates that the teacher’s ability to determine when to issue guidance is robust to the choice of \( (m_{\text{in}}, m_{\text{out}}) \).

\begin{table}[ht]
\centering
\begin{tabular}{lcc}
\toprule
\textbf{Metric} & \textbf{ANOVA p-value} & \textbf{Kruskal--Wallis p-value} \\
\midrule
Jensen--Shannon     & 0.1138 & 0.1592 \\
Kullback--Leibler   & 0.2457 & 0.1799 \\
Total Variation     & 0.1728 & 0.2322 \\
Hellinger Distance  & 0.1247 & 0.1592 \\
\bottomrule
\end{tabular}
\caption{Statistical test results (p-values) for divergence between False OOD and True OOD energy distributions across different \( (m_{\text{in}}, m_{\text{out}}) \) settings.}
\label{tab:false-vs-true-pvals}
\end{table}

\paragraph{Results.}  
As shown in Table~\ref{tab:false-vs-true-pvals}, we observe no statistically significant variation in the separation between false and true OOD states across different \( (m_{\text{in}}, m_{\text{out}}) \) configurations. The ANOVA and Kruskal-Wallis tests yield p-values above 0.1 for all four divergence metrics, indicating that the teacher’s ability to distinguish between states where guidance should or should not be issued is stable across regularization settings. 

\subsubsection{Choice of ID States}
\label{app:choice-of-id}

When available, we set the threshold quantile \(q\) from the empirical distribution of \emph{in-train} states \(\mathcal{D}_{\text{in}}^{\text{train}}\) collected during teacher on-policy learning (with exploration). When the teacher’s training distribution is unavailable, we approximate this \emph{post-train} by rolling out the converged teacher (no exploration) and computing \(q\) from those states. Because exploration noise is negligible at convergence, these rollouts serve as a reliable proxy for the high-density regions of the teacher’s visitation distribution. We validated both choices by running 100 on-policy evaluation episodes and setting \(q\) from the resulting state samples across 10 seeds in two GridWorld settings. Tables~\ref{tab:choice-id:altgoal}--\ref{tab:choice-id:unlock} report student transfer performance (relative to training from scratch) across quantiles.

\begin{table*}[t]
  \centering
  \tighttable
  \scriptsize                     
  \caption{Choice of ID states for setting the quantile \(q\). “Post-train” derives \(q\) from evaluation rollouts of the converged teacher (no exploration); “In-train” derives \(q\) from exploration-time states during teacher training. Values are student transfer performance relative to training from scratch (mean \(\pm\) 95\% CI)}.
  \label{tab:choice-id}
  
  \begin{subtable}{\textwidth}
    \centering
    \caption{Alternating-Goal Environment}
    \resizebox{\textwidth}{!}{%
    \begin{tabular}{lcccccccccc}
      \toprule
      \textbf{Threshold} & 0 & 0.1 & 0.2 & 0.3 & 0.4 & 0.5 & 0.6 & 0.7 & 0.8 & 0.9 \\
      \midrule
      \textbf{Post-train} & $-9.8\!\pm\!6.0$ & $21.2\!\pm\!8.8$ & $22.5\!\pm\!9.0$ & $28.9\!\pm\!10.6$ & $33.1\!\pm\!8.0$ & $37.8\!\pm\!6.5$ & $46.9\!\pm\!4.2$ & $37.8\!\pm\!8.7$ & $25.9\!\pm\!12.0$ & $34.7\!\pm\!7.1$ \\
      \textbf{In-train}   & $1.9\!\pm\!8.0$  & $36.4\!\pm\!6.9$ & $41.6\!\pm\!4.3$ & $41.2\!\pm\!4.9$  & $46.1\!\pm\!8.2$ & $46.1\!\pm\!4.5$ & $37.3\!\pm\!4.2$ & $40.3\!\pm\!4.7$ & $25.8\!\pm\!7.4$  & $29.2\!\pm\!5.9$ \\
      \bottomrule
    \end{tabular}}
    \label{tab:choice-id:altgoal}
  \end{subtable}


  \begin{subtable}{\textwidth}
    \centering
    \caption{UnlockedToLocked Environment}
    \resizebox{\textwidth}{!}{%
    \begin{tabular}{lcccccccccc}
      \toprule
      \textbf{Threshold} & 0 & 0.1 & 0.2 & 0.3 & 0.4 & 0.5 & 0.6 & 0.7 & 0.8 & 0.9 \\
      \midrule
      \textbf{Post-train} & $11.2\!\pm\!4.6$ & $21.5\!\pm\!4.5$ & $31.1\!\pm\!1.6$ & $29.2\!\pm\!4.8$ & $31.7\!\pm\!5.2$ & $36.8\!\pm\!4.6$ & $40.0\!\pm\!2.8$ & $38.0\!\pm\!3.3$ & $35.8\!\pm\!5.3$ & $24.3\!\pm\!3.9$ \\
      \textbf{In-train}   & $16.8\!\pm\!4.4$ & $27.1\!\pm\!3.8$ & $29.3\!\pm\!3.7$ & $31.2\!\pm\!4.8$ & $28.8\!\pm\!3.1$ & $33.6\!\pm\!3.4$ & $35.3\!\pm\!2.5$ & $35.1\!\pm\!2.7$ & $30.7\!\pm\!7.6$ & $32.9\!\pm\!5.8$ \\
      \bottomrule
    \end{tabular}}
    \label{tab:choice-id:unlock}
  \end{subtable}
\end{table*}

\subsubsection{Choice of OOD Batches}
\label{app:ood-batches}

The role of $L_{\text{energy}}$ is to separate states that lie within the teacher’s training distribution from those that do not. In the main paper, OOD states are taken from random rollouts of the student environment. This choice was made because they provide a concrete and intuitive illustration of this boundary. 
To address the concern about future leakage, we conduct a new experiment in the \emph{Unlocked-to-Locked} setting where OOD states are sampled uniformly from the full MiniGrid observation space. This sampling does not assume anything about the target task; it draws from all valid observations that could occur in the domain, and we match the sample size to the original setup for a fair comparison. This removes any possibility of information leakage.
Figure~\ref{fig:random} reports the results. Performance remains consistent with the main paper, and we continue to observe clear transfer improvements. 

\begin{figure}[ht]
    \centering
    \includegraphics[width=0.5\linewidth]{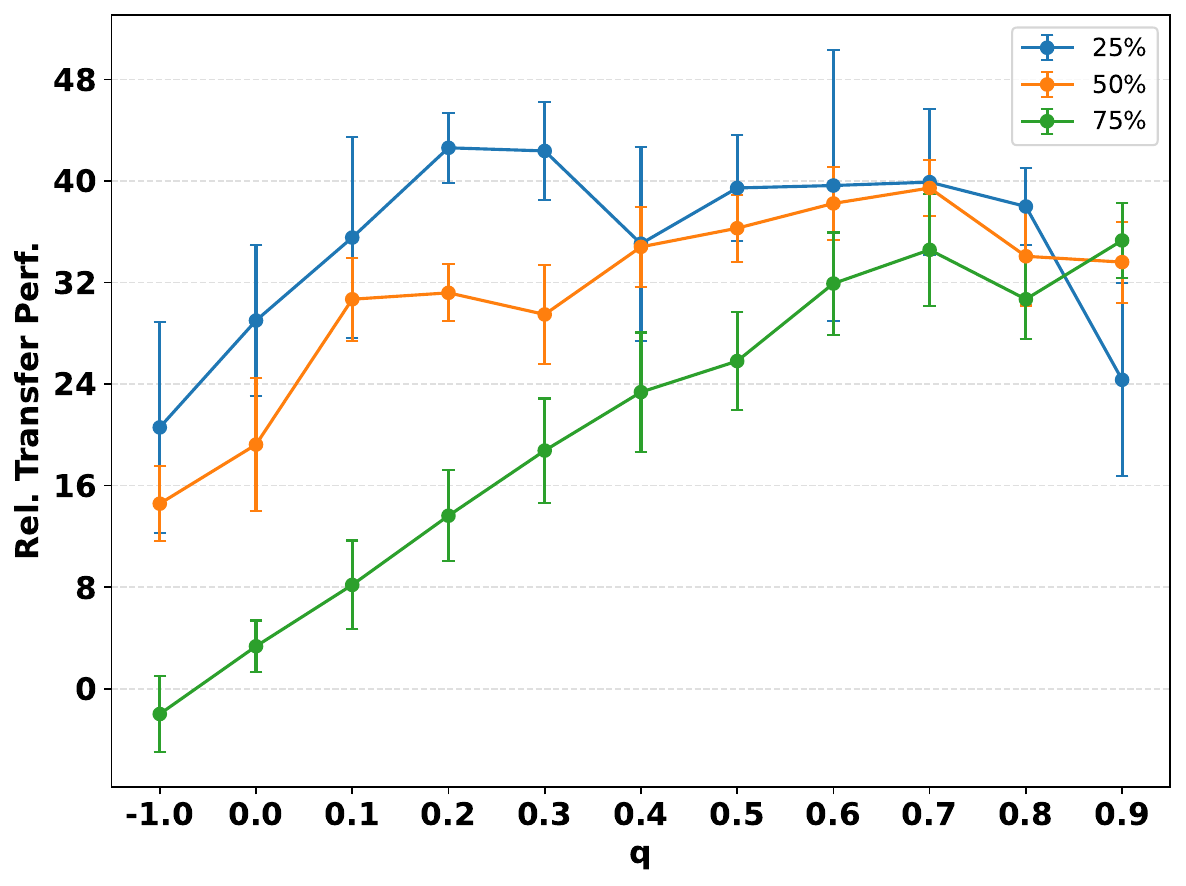}
    \caption{10 seeds. Effect of replacing rollout-based OOD batches with randomly sampled OOD batches. $q=-1$ advise in all states (AA Baseline).}
    \label{fig:random}
\end{figure}

\subsubsection{Effect of Energy Regularization on Teacher Convergence}
\label{app:convergence}
We assess whether adding the energy regularizer to the \emph{teacher} objective harms final performance or slows learning. 
Across two GridWorld source tasks and 10 seeds, final returns are unchanged, while convergence is faster with the energy term.

\begin{table}[ht]
\centering
\small
\tighttable
\caption{Training steps to convergence (mean over 10 seeds; lower is better).}
\begin{tabular}{lrr}
\toprule
\textbf{Teacher} & \textbf{Without energy loss} & \textbf{With energy loss} \\
\midrule
Alternating Room & 100{,}000 & 60{,}000 \\
Unlocked-to-Locked & 260{,}000 & 220{,}000 \\
\bottomrule
\end{tabular}
\end{table}

We conjecture that the acceleration arises because the energy term adds an inductive bias that highlights which states are in-distribution (high score) versus out-of-distribution (low score), guiding updates toward familiar regions of the state space more efficiently.

\subsection{Decay Schedules}
\label{app:decay}
This section provides the full comparison between linear decay and several budget–based decay mechanisms. Although both approaches aim to reduce teacher influence over training, we found that linear decay is considerably easier to control and more stable across tasks.

All experiments use the \emph{Unlocked-to-Locked} setting (1M steps). We evaluate five variants:

\begin{enumerate}
    \item \textbf{No Decay}: the teacher issues advice whenever the state is classified as in-distribution.
    \item \textbf{Single Budget (No Reset)}: a fixed budget equal to 10\% of the total steps; once depleted, no further advice is allowed. This follows the \textit{single budget} setting in \citet{torrey2013teaching}, where the teacher is allocated a finite number of advice calls that are consumed over time. As the budget is exhausted, the frequency of advice implicitly decreases, effectively inducing a decay in teacher intervention without an explicit time-based schedule.
    \item \textbf{Interval Budget (With Reset)}: every 10{,}240 steps, the teacher is allowed to spend 10\% of that interval as advice.
    \item \textbf{Interval Budget + Linear Decay}: variant (3) with a linear decay factor applied within each interval.
    \item \textbf{Linear Decay (Ours)}: a single linear decay schedule over training, ending at the midpoint.
\end{enumerate}

Figure~\ref{fig:decay} reports the performance of all five schedules. 
The results show that budget-based schemes, with or without reset, often lead to negative transfer when $q < 0.7$. The issue is not budget exhaustion but that a fixed budget allows too much advice early in training. This limits the student’s opportunity to act on its own during the exploration phase, which is important for learning the target task. Because the number of early interventions depends on episode structure and exploration behavior, budget schedules are also difficult to tune.

Linear decay avoids this problem by reducing advice gradually and predictably over time. It prevents excessive early intervention while giving the student increasing autonomy as training progresses, without requiring a fixed allowance of advice.

\begin{figure}[ht]
    \centering
    \includegraphics[width=0.5\linewidth]{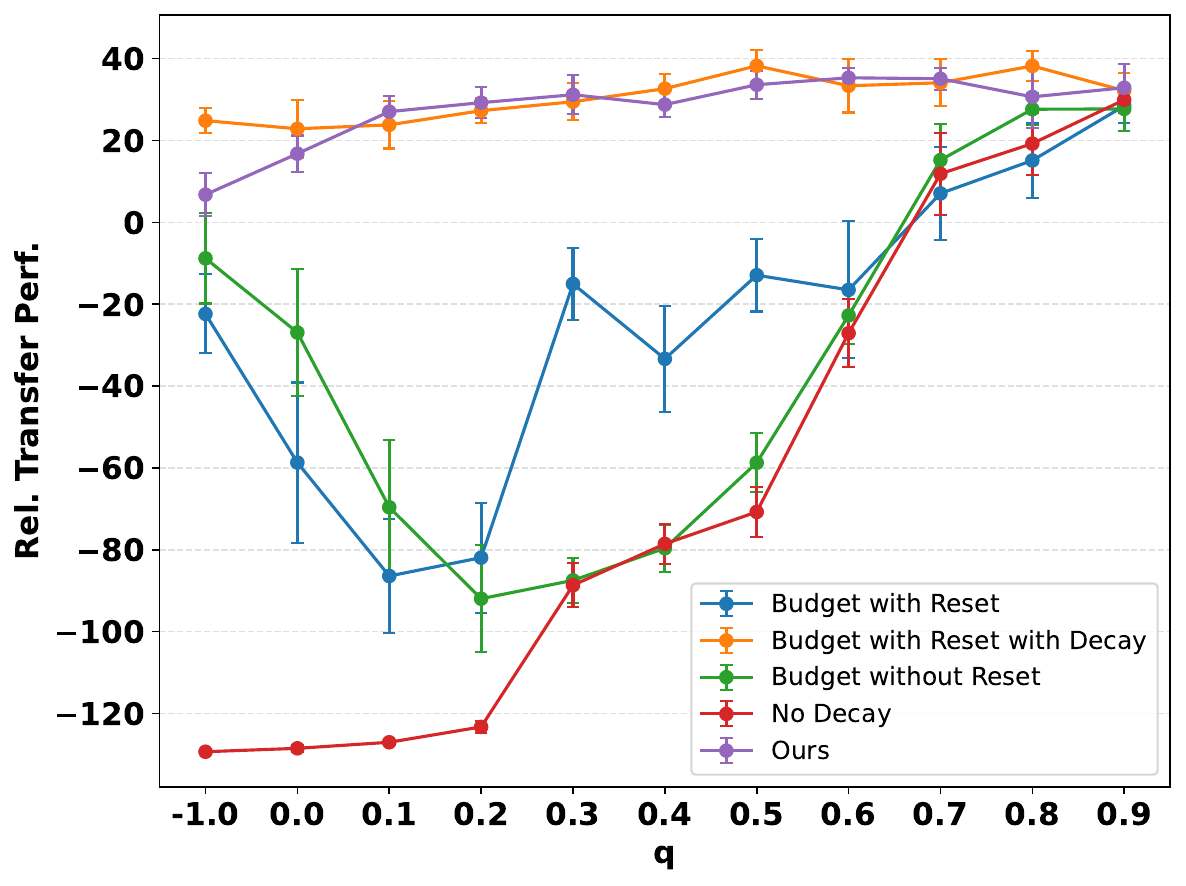}
    \caption{10 seeds. Comparison of five decay schedules. $q=-1$ advise in all states (AA Baseline).}
    \label{fig:decay}
\end{figure}

\subsection{Heatmaps of Average Energy Quantiles}
\label{app:map}

We visualize the heatmaps of average energy quantiles under the teacher for Alternating Goal in Figure~\ref{fig:Map}. The teacher assigns higher energy to states seen in training (goal in Room 1, upper-left) than to unseen states (goal in Room 3, lower-right).

\begin{figure}[htbp]
        \centering
        \includegraphics[width=0.8\linewidth]{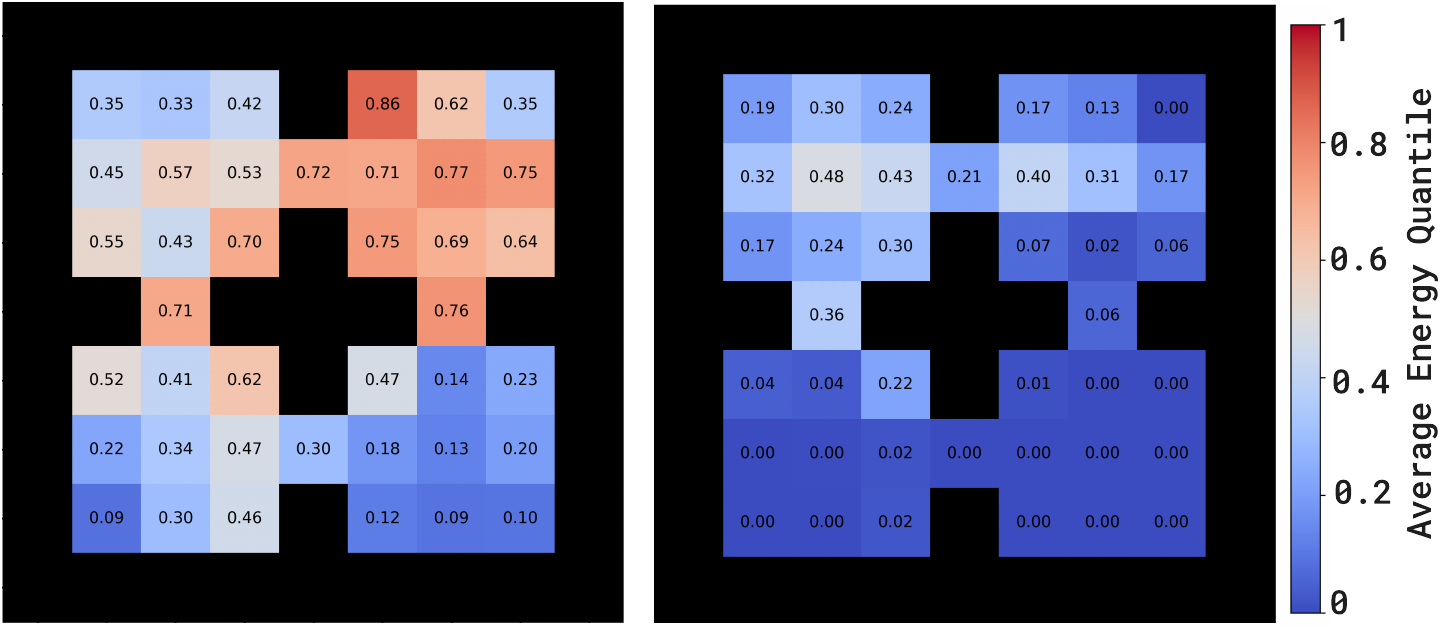}
        \caption{
Heatmaps of average energy quantiles under the teacher for Alternating Goal. Left: source (goal in Room 1). Right: target (goal in Room 3). Higher quantiles indicate greater teacher familiarity.}
        \label{fig:Map}
\end{figure}

\subsection{Advice Rates During Training}
\label{app:advice-rates}
This section reports the rate at which the teacher issues advice and the rate at which the student follows that advice (after applying the decay schedule) in the \emph{Unlocked-to-Locked} environment. We evaluate several values of $q$ under multiple decay schedules (25\%, 50\%, 75\% of the training horizon). The result is shown in Figure~\ref{fig:advice}.

As training progresses, the issue-advice rate tends to increase. Once the student begins to solve more of the task on its own, it encounters states that fall within the teacher’s training distribution more frequently, and the teacher identifies a larger fraction of states as familiar. The take-advice rate is the product of the issue-advice rate and the decay factor, so it decreases over time even as the teacher becomes more confident.
These measurements describe how the influence of the teacher evolves throughout training and show that EBTL creates a controlled shift from teacher-guided actions to fully student-driven behavior.

\begin{figure}[htbp]
    \centering

    \begin{subfigure}{\textwidth}
        \centering
        \includegraphics[width=0.8\textwidth]{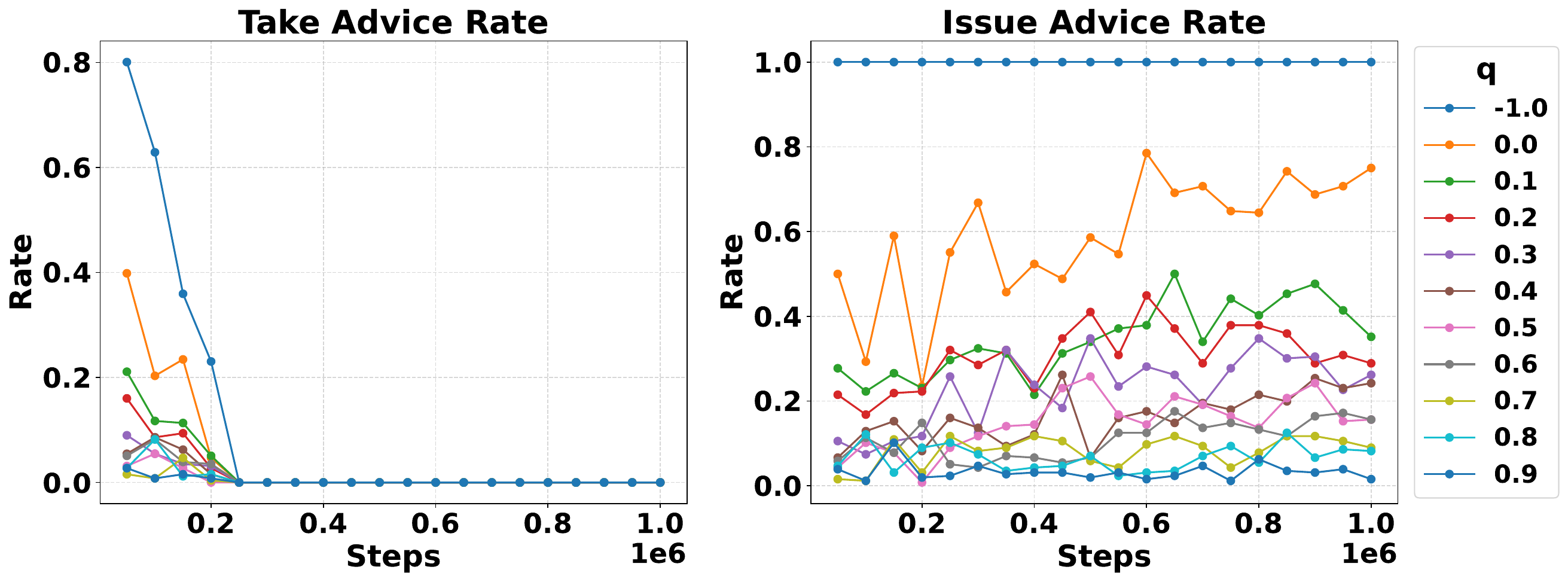}
        \label{fig:a}
    \end{subfigure}

    \begin{subfigure}{\textwidth}
        \centering
        \includegraphics[width=0.8\textwidth]{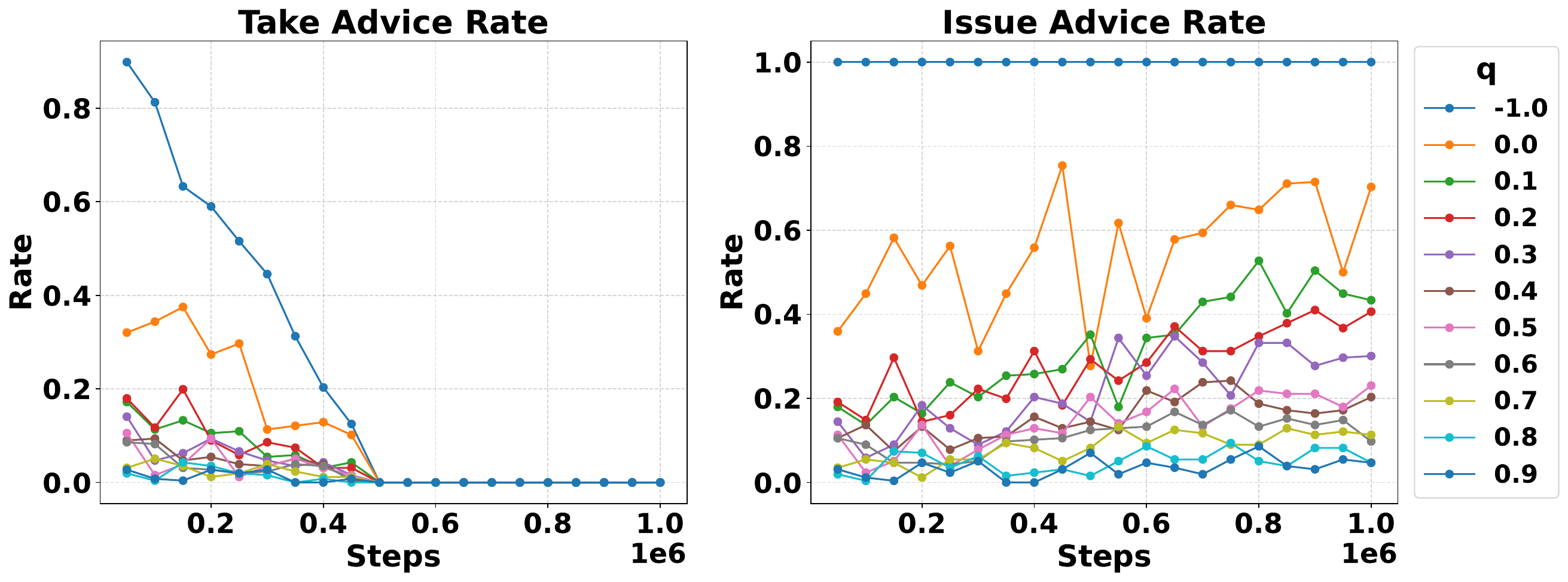}
        \label{fig:b}
    \end{subfigure}

    \begin{subfigure}{\textwidth}
        \centering
        \includegraphics[width=0.8\textwidth]{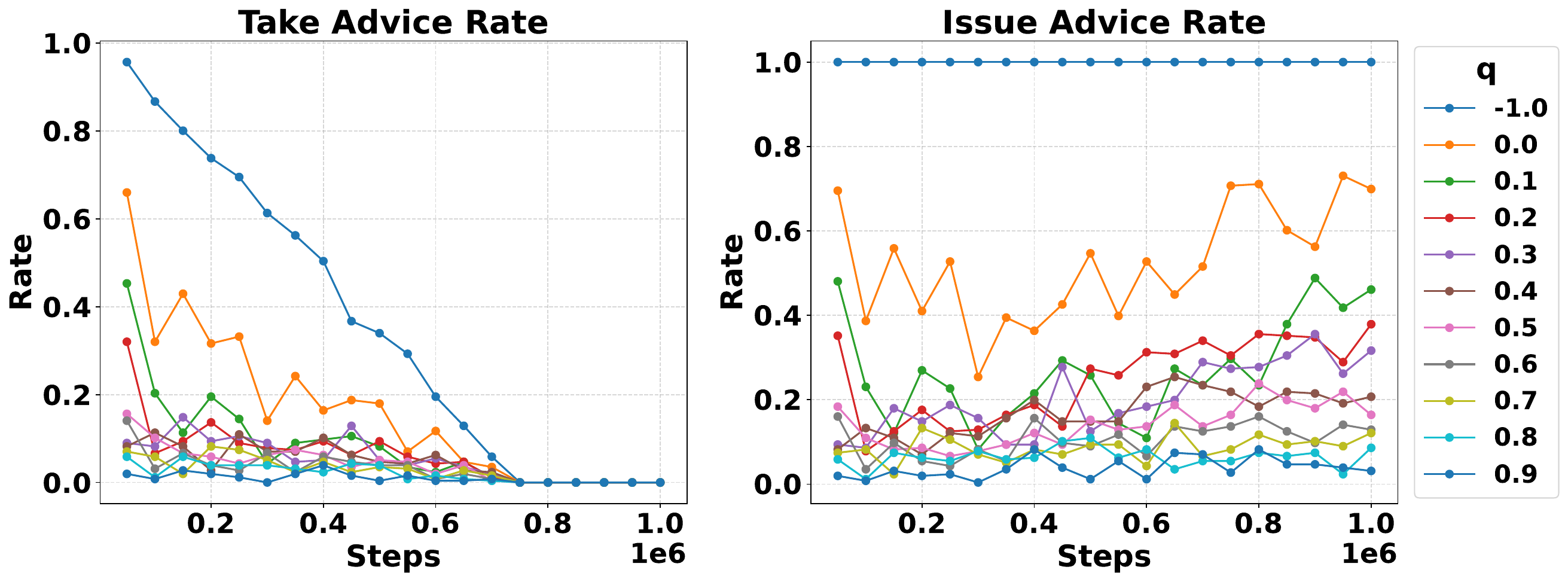}
        \label{fig:c}
    \end{subfigure}

    \caption{10 seeds. Advice schedules where advice ends at (a) 25\%, (b) 50\%, and (c) 75\% of training in the Unlocked-to-Locked environment. $q=-1$ advise in all states (AA Baseline).}
    \label{fig:advice}
\end{figure}

\subsection{Additional Learning Curves Across $q$}
\label{sec:more_results}

Figure~\ref{fig:grid_results}(c) reports relative transfer performance across varying $q$ against the baselines, revealing a clear balance. When $q$ is too small, the threshold is permissive and admits too much harmful advice in out-of-distribution states; when $q$ is too large, it is overly strict and the teacher provides too little guidance to be useful. The curve in panel~(c) makes this visible: it climbs as $q$ increases past the harmful-advice regime, flattens into a broad plateau across $q \in [0.2, 0.7]$, and declines again as guidance becomes too sparse. Beyond this plateau, relative transfer performance is limited. In the main paper, panel~(d) plots learning curves only for the representative value $q = 0.5$, since showing all $q$ values alongside the baselines would clutter the panel. To make explicit how the aggregate values in panel~(c) map to the underlying training dynamics, Figure~\ref{fig:varying_q} provides the full set of curves for every $q \in \{0.2, 0.3, 0.4, 0.5, 0.6, 0.7\}$, all of which yield near-optimal sample efficiency and final return, remaining largely overlapping and well-separated from the baselines.

\begin{figure}[ht]
    \centering
    \includegraphics[width=0.9\linewidth]{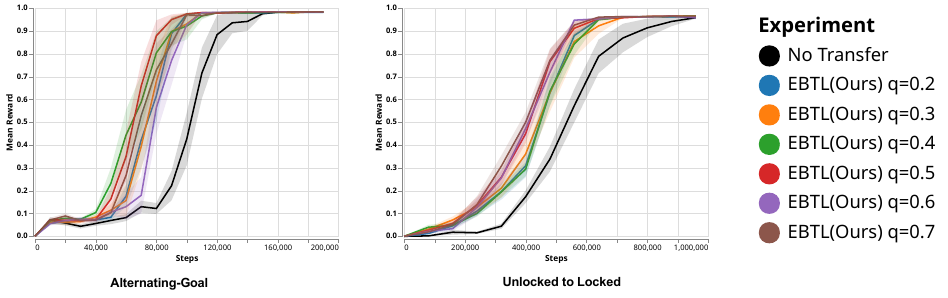}
    \caption{
    \textbf{Learning curves across $q$} (10 seeds). EBTL for
    $q \in \{0.2, \dots, 0.7\}$ vs.\ baselines; near-identical across the range.
    Guidance ends at mid-training.
    }
    \label{fig:varying_q}
\end{figure}

\subsection{Comparison to Critic-Based Introspection}
\label{sec:iaa}

Introspective Action Advising (IAA)~\citep{campbell2023introspective} is a
closely related introspective action-advising method. Unlike EBTL, which derives
its intervention signal from the teacher's \emph{policy}, IAA operates on the
\emph{critic} side: it fine-tunes a copy of the teacher's value function on the
student's data during transfer, and the quality of its advice depends heavily on
this online estimation. Because this places IAA in a methodologically distinct
family that hinges on a separately learned value estimate, we examine it here in
the appendix rather than alongside the policy-side baselines in the main paper.

The mechanism proceeds in two phases. Let $V^{\text{Src}}_{\pi_T}$ denote
the teacher's value function learned and frozen on the source task, and
$V^{\text{new}}_{\pi_T}$ a copy of it that is fine-tuned during transfer. During
an initial burn-in of $\delta$ steps, the teacher issues no advice; instead,
$V^{\text{new}}_{\pi_T}$ is updated on the returns observed from the student's
rollouts in the target task, adapting the source value estimate toward the
target. After burn-in, for each state $s$ visited by the student the teacher
compares the two estimates and issues advice only when the value gap is below a
threshold $\epsilon$:

\[
  \bigl| V^{\text{new}}_{\pi_T}(s) - V^{\text{Src}}_{\pi_T}(s) \bigr| \le \epsilon .
\]

A small gap is taken to mean the teacher's behavior at $s$ transfers to the
target task, while a large gap signals task mismatch and advice is withheld.

Because the criterion compares a frozen source estimate against a critic
that is still being fine-tuned on limited, non-stationary student data, the
reliability of the resulting signal hinges on how the teacher's critic evolves
over the course of the student's learning: early in transfer the value gap may
reflect estimation error rather than genuine task mismatch. The evolution of
this critic is therefore central to whether IAA can identify transferable
states. 
To evaluate the baseline as favorably as possible, we examine the signal
directly and measure how well the value-gap criterion separates transferable
from non-transferable states throughout training, in both GridWorld settings
(Alternating-Goal and Unlocked-to-Locked). For an equally fair comparison, when
fine-tuning the teacher's critic we test two variants: one that freezes the
convolutional backbone and adapts only the value head, and one that fine-tunes
the full network. The results are shown in
Figure~\ref{fig:critic_evolution_alternating} for Alternating-Goal and in
Figure~\ref{fig:critic_evolution_locked} for Unlocked-to-Locked.

As both figures show, the value-gap criterion only begins to separate
transferable from non-transferable states once the student is more than halfway
through training; before that point, the ID and Transferable-OOD curves are not
reliably above the Non-transferable-OOD curve, so early advice is largely
uninformative.
In these experiments, this creates a timing mismatch for transfer: as
Figure~\ref{fig:grid_results}(d) shows, EBTL has already converged by this stage,
meaning IAA's signal becomes trustworthy only after the window in which guidance
is most valuable has passed. We further observe no clear difference between the
full fine-tuning and frozen-backbone variants, indicating that the delay stems
from the critic's adaptation to non-stationary student data rather than from the
choice of which parameters to fine-tune.

To understand the original IAA configuration as thoroughly as possible, we
sweep the decay schedule, testing variants that reach zero advising probability
at 50\% and at 75\% of training. For an apples to apples comparison, we run
each method under both the exponential decay used in the original IAA paper
and the linear decay used by EBTL, and report the two schedules separately in
Figure~\ref{fig:original_iaa}. We also include the no-transfer baseline and
standard action advising (AA) without any filtering, so that the contribution
of IAA's introspection signal can be isolated from the contribution of the
schedule itself.

\textbf{We reproduce IAA, but its benefit is largely confined to the
exponential decay schedule.} Under the original IAA configuration of
exponential decay with $\epsilon = 0.9$, we recover the result reported by
\citet{campbell2023introspective}, with IAA improving substantially over the
no-transfer baseline. However, under this schedule even standard action
advising without any filtering also exceeds the baseline, suggesting that part
of IAA's reported gain stems from the aggressive early decay rather than from
its introspection signal. The linear decay panel tells a different story:
among the advising methods, only EBTL remains clearly above the baseline.
Across both schedules, EBTL is the best performing method. We attribute this
robustness to the fact that EBTL filters transferable from non-transferable
states at the level of the policy itself and so does not need an aggressive
schedule to mask poor advice. The exponential schedule benefits the other
methods because it annuls advice quickly during the early exploration phase,
limiting the influence of harmful guidance that they cannot otherwise filter.

\begin{figure}[htbp]
    \centering
    \includegraphics[width=0.85\textwidth]{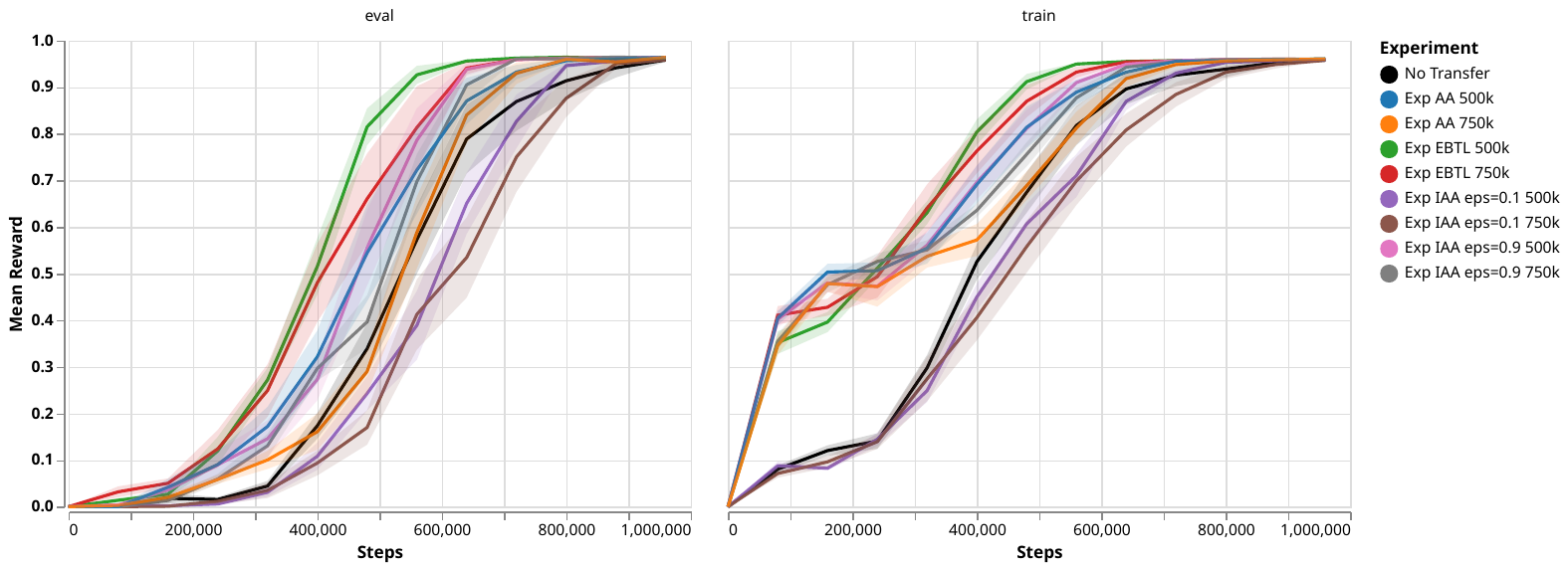}\\[2pt]
    \includegraphics[width=0.85\textwidth]{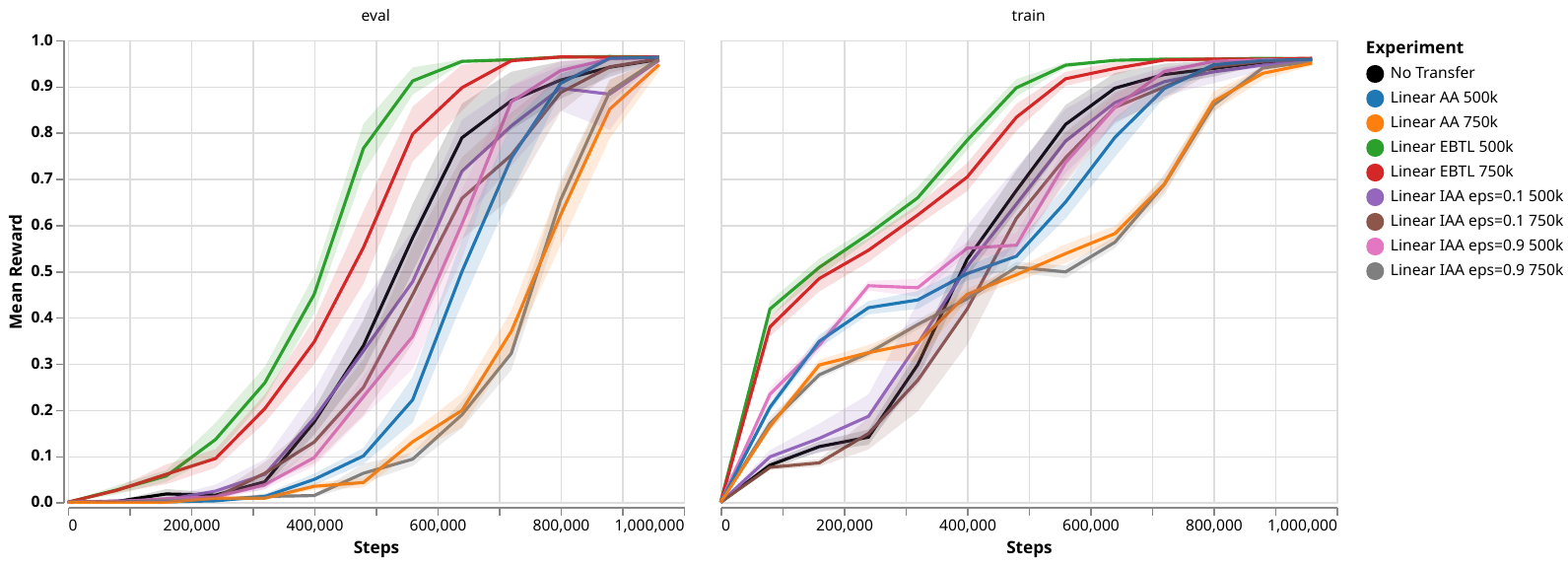}
    \caption{\textbf{IAA reproduction and decay-schedule comparison.} 10 seeds.
    Top: exponential decay (as in the original IAA paper). Bottom: linear decay
    (the EBTL default). Under exponential decay we reproduce the IAA result,
    but standard action advising without filtering also exceeds the baseline,
    suggesting the schedule itself contributes substantially to the gain.
    Under linear decay, only EBTL remains clearly above the baseline. EBTL is
    the best performing method under both schedules.}
    \label{fig:original_iaa}
\end{figure}

Figure~\ref{fig:critic_evolution_locked} indicates that the teacher's
critic is not informative at the start of student training and may require
considerable adaptation before its signal becomes reliable. For a more complete
evaluation, we therefore test IAA with substantially longer burn-in periods, so
that the teacher's critic has more time to evolve on student data before any
advice is issued. We sweep two burn-in windows, one beginning at 25\% of
training and one beginning at 50\% of training, paired with linear decay
schedules that reach zero advising probability at the 50\%, 75\%, or 100\% mark.
We use $\epsilon = 0.1$ and $\epsilon = 0.15$, the two thresholds that
Figure~\ref{fig:critic_evolution_locked} shows yield the cleanest separation
between transferable and non-transferable states. This sweep differs in
intent from the previous one: rather than matching the configuration reported
by \citet{campbell2023introspective}, the goal is to wait until the critic has
plausibly developed the ability to discriminate transferable from
non-transferable advice before deploying it. Results are shown in
Figure~\ref{fig:iaa_full}.

EBTL still outperforms IAA across the sweep, with roughly half of the
IAA configurations clearing the no-transfer baseline. The best IAA configuration
uses $\epsilon = 0.1$ with the burn-in ending at 250K steps and linear advice
decay reaching zero at 750K. This is consistent with
Figure~\ref{fig:critic_evolution_locked}: at $\epsilon = 0.1$ the teacher's
critic only acquires meaningful separation between transferable and
non-transferable states after roughly 40\% of the horizon, so configurations
that begin advising before then are largely working with an uninformative
signal. The pattern reinforces our central claim: when the introspection signal
genuinely separates transferable from non-transferable states, transfer
improves. 
One challenge with IAA is that the timing of this separation depends on
how the teacher's critic evolves under fine-tuning on student data, and that
evolution is run-specific: it is not fully apparent until student training has
progressed. Selecting an effective burn-in window in advance would in principle
require knowing at what fraction of training the critic's value gap becomes
informative, which is the trajectory captured in
Figure~\ref{fig:critic_evolution_locked}, and that trajectory typically only
emerges once student training is well underway. This makes it difficult to
identify a strong IAA configuration ahead of time.

\begin{figure}[htbp]
    \centering
    \includegraphics[width=0.85\textwidth]{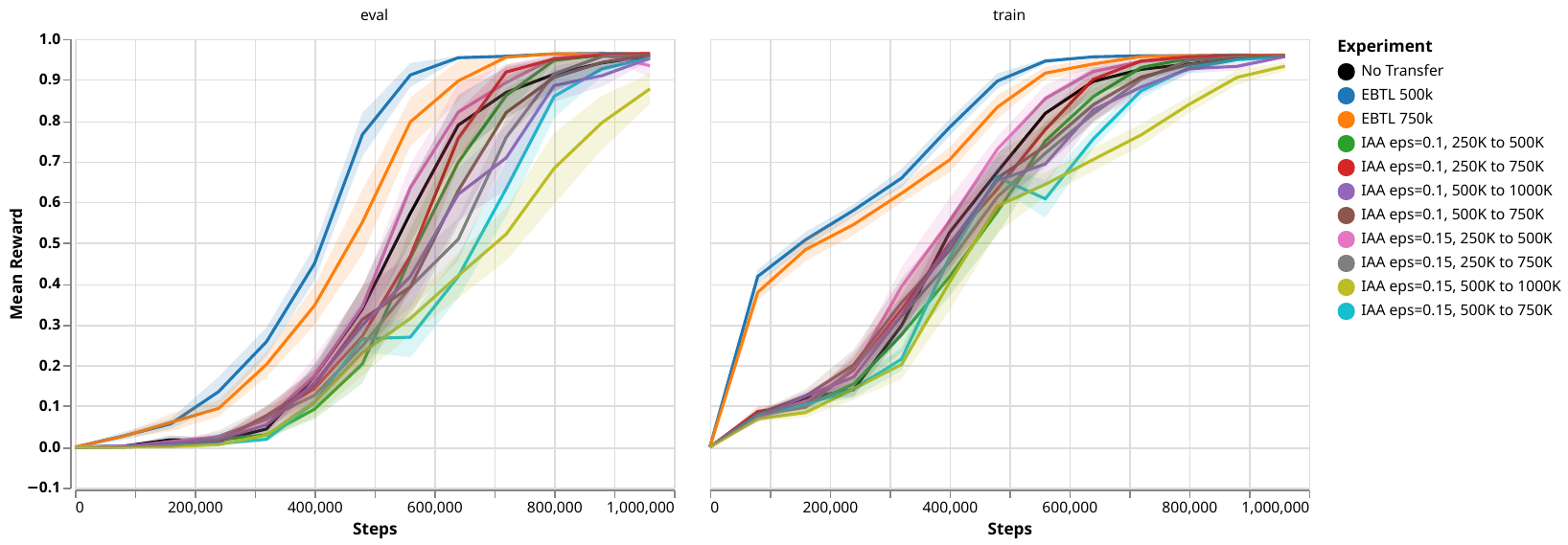}
    \caption{\textbf{IAA with extended burn-in periods.} 10 seeds. Linear Decay. Learning curves
    for IAA across burn-in start points (25\% and 50\% of training), advice
    decay endpoints (50\%, 75\%, and 100\% of training), and thresholds
    ($\epsilon \in \{0.1, 0.15\}$).}
    \label{fig:iaa_full}
\end{figure}

\begin{figure}[htbp]
    \centering
    \includegraphics[width=0.34\textwidth]{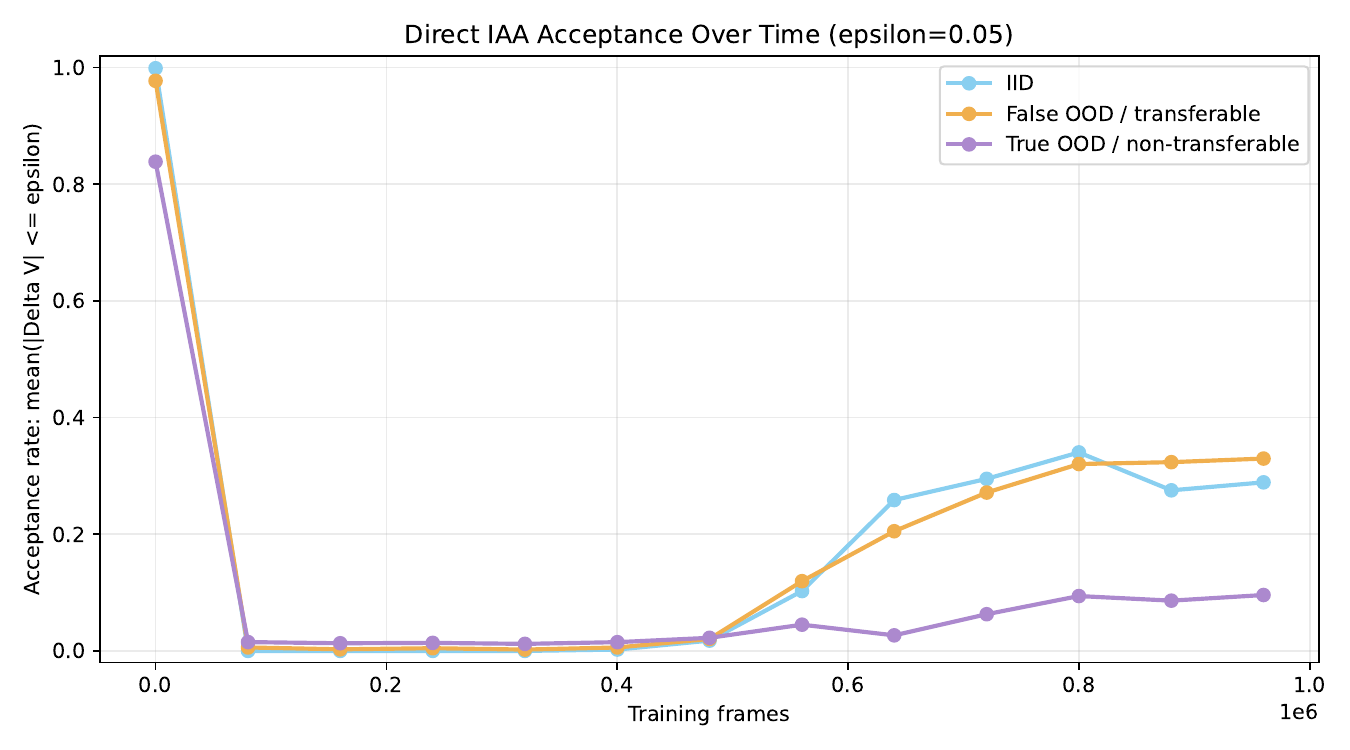}\hspace{2pt}%
    \includegraphics[width=0.34\textwidth]{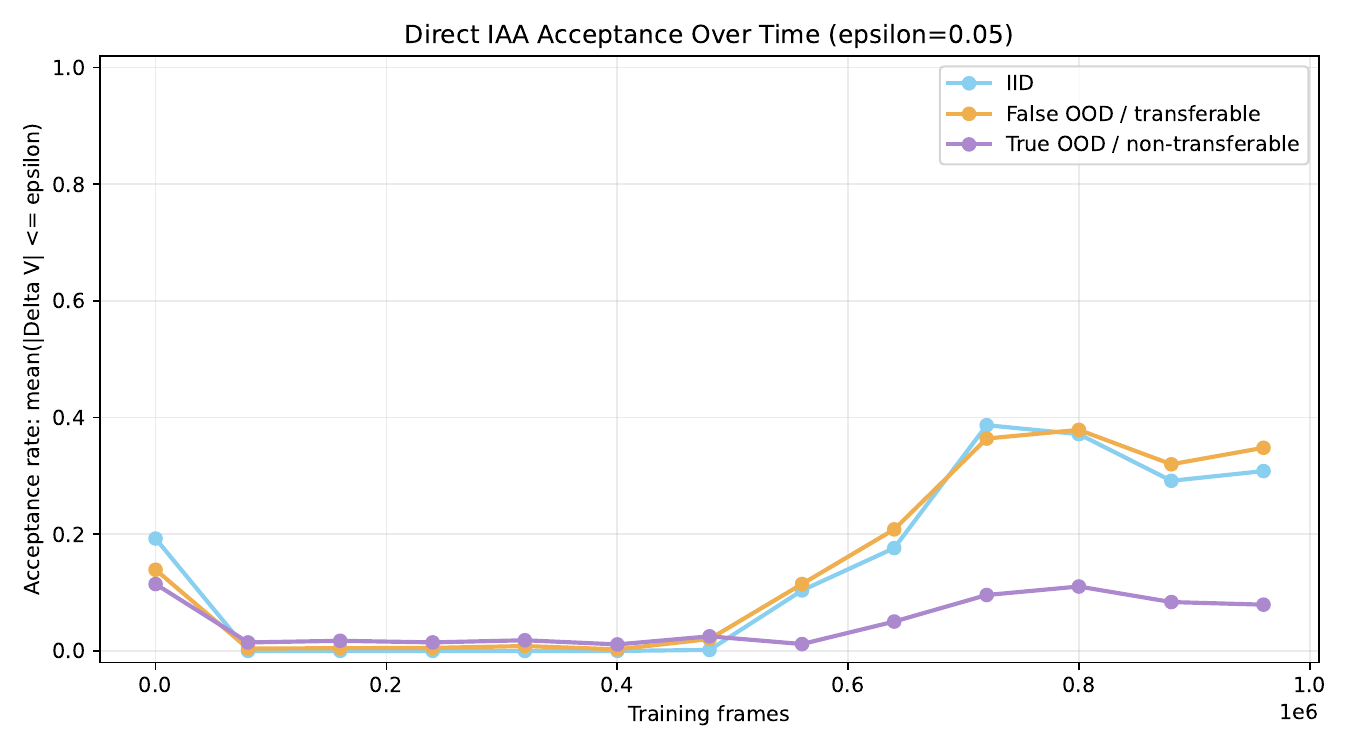}\\[-1pt]
    \includegraphics[width=0.34\textwidth]{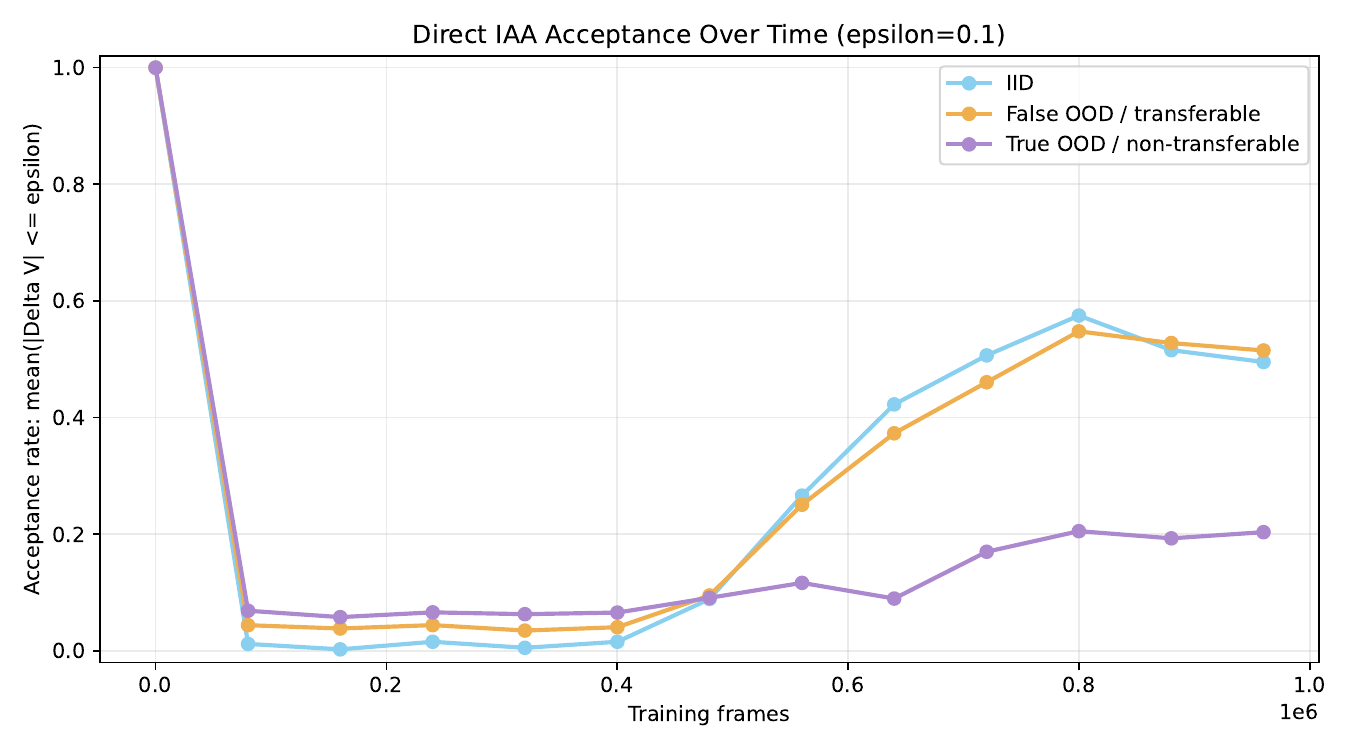}\hspace{2pt}%
    \includegraphics[width=0.34\textwidth]{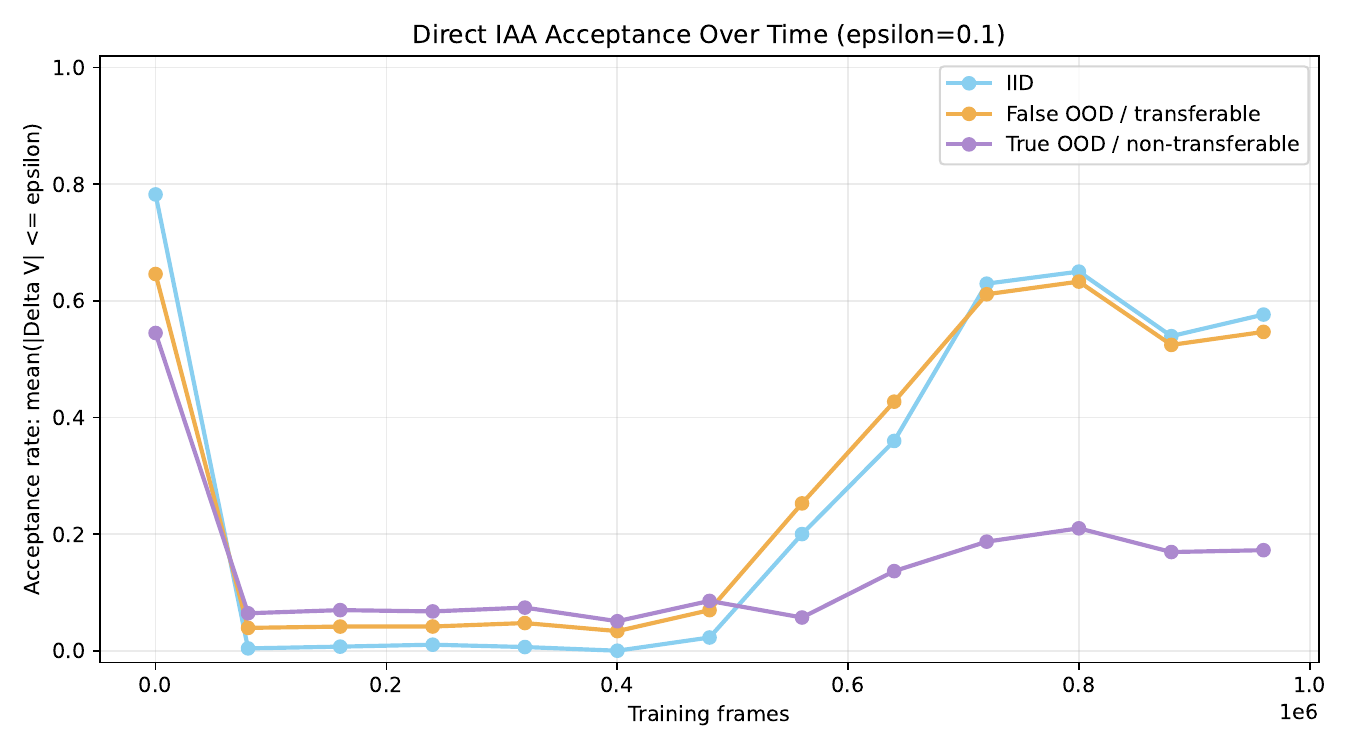}\\[-1pt]
    \includegraphics[width=0.34\textwidth]{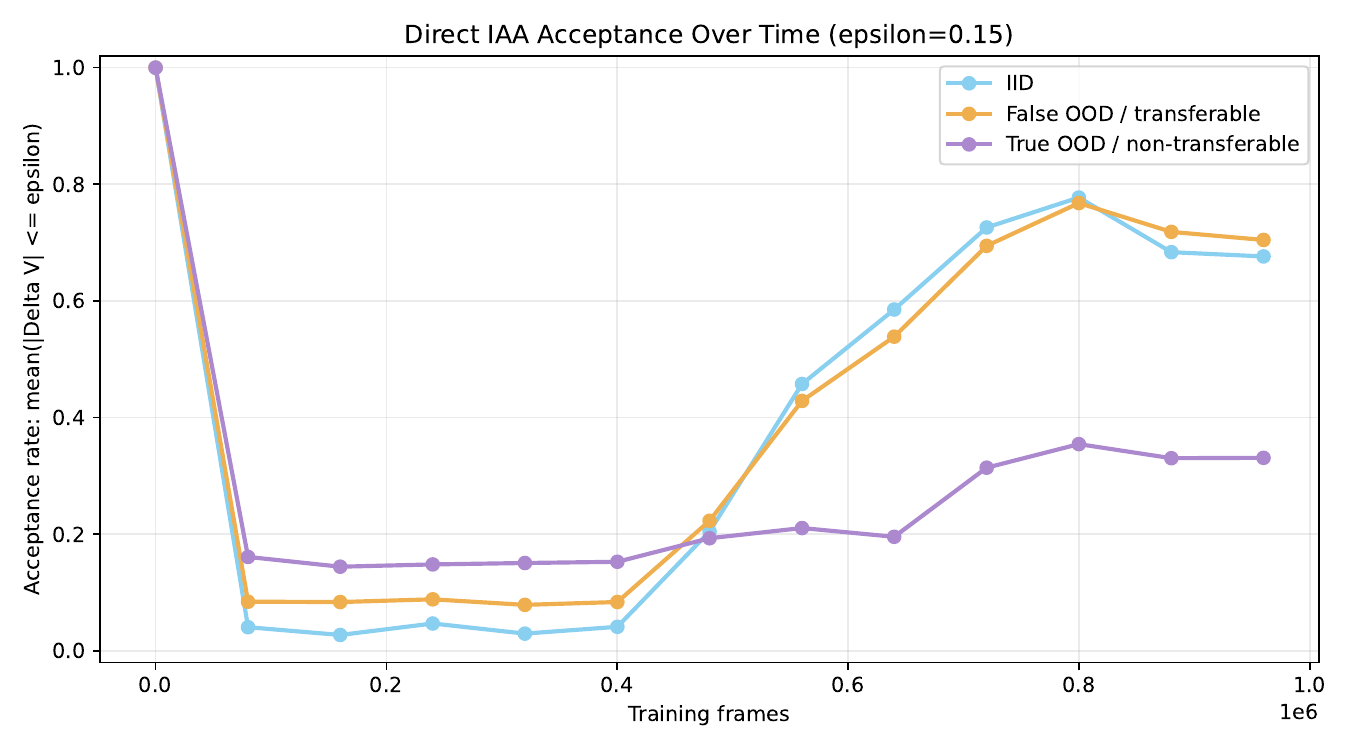}\hspace{2pt}%
    \includegraphics[width=0.34\textwidth]{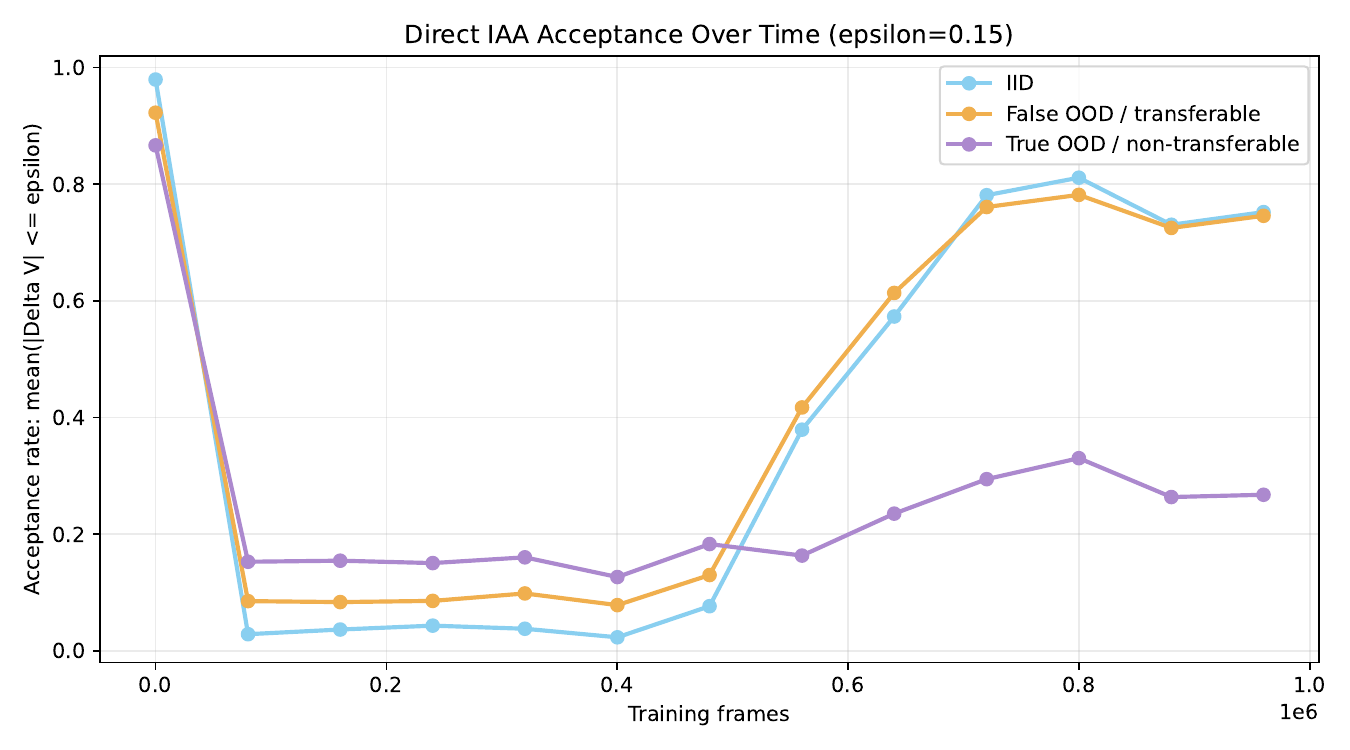}\\[-1pt]
    \includegraphics[width=0.34\textwidth]{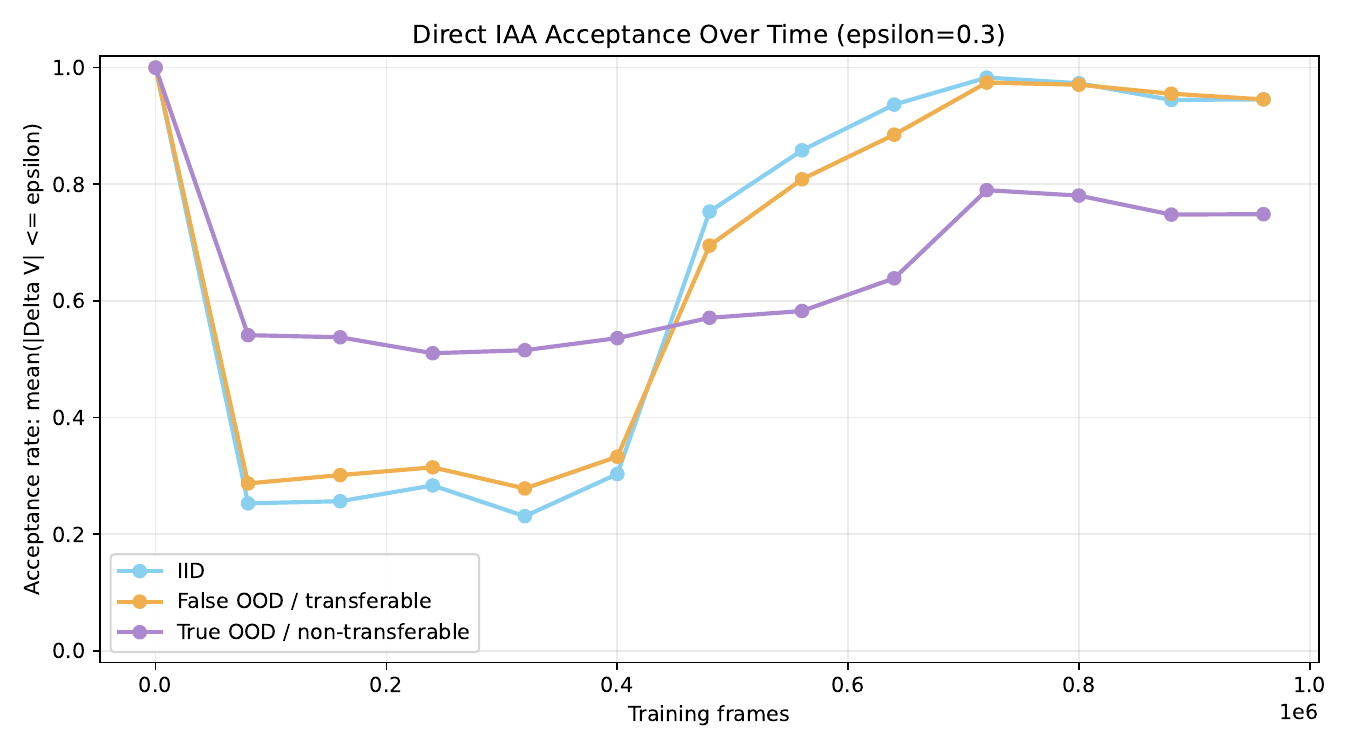}\hspace{2pt}%
    \includegraphics[width=0.34\textwidth]{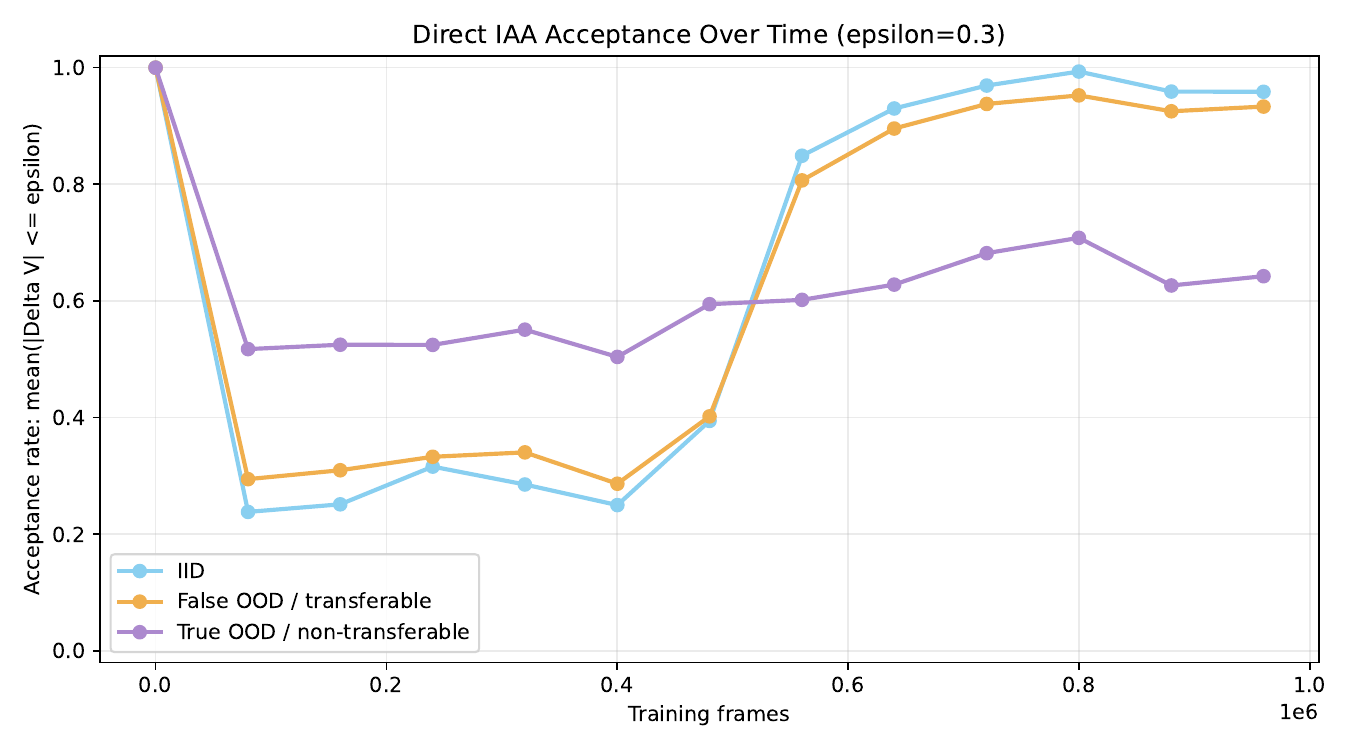}\\[-1pt]
    \includegraphics[width=0.34\textwidth]{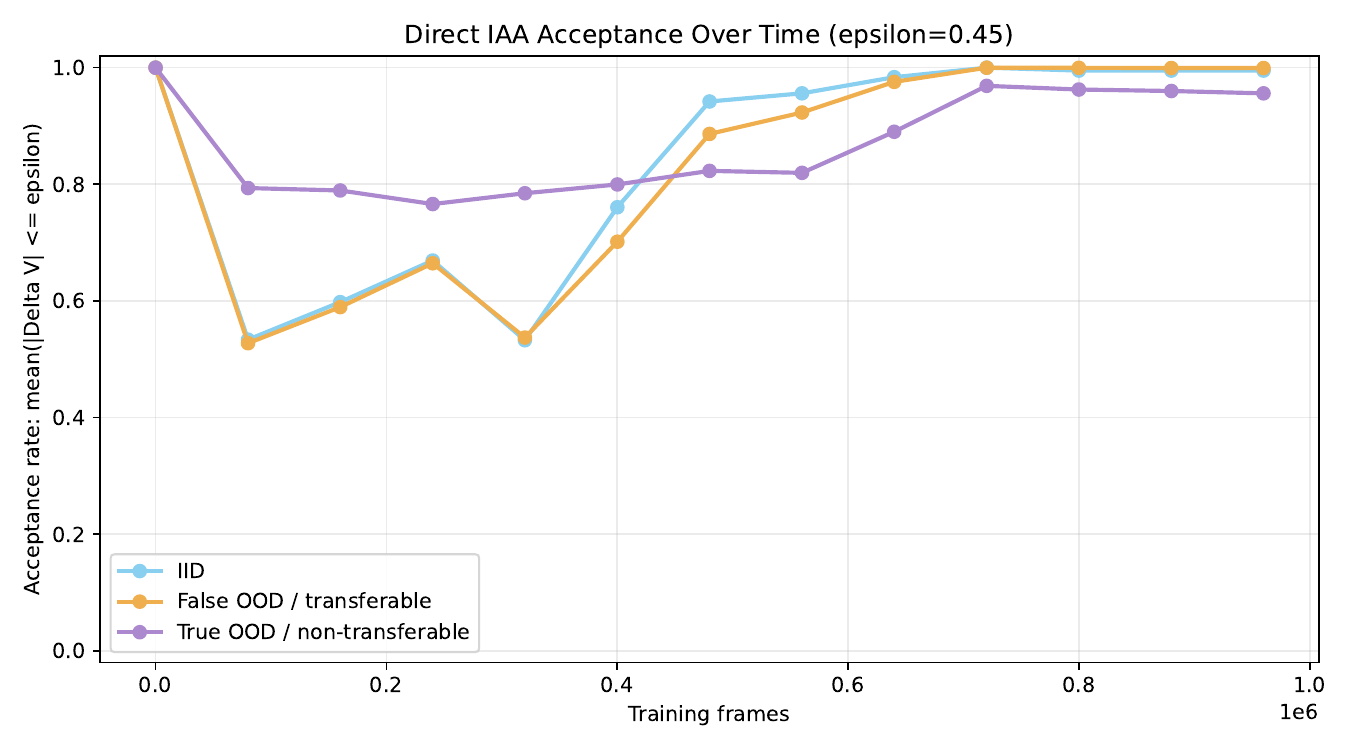}\hspace{2pt}%
    \includegraphics[width=0.34\textwidth]{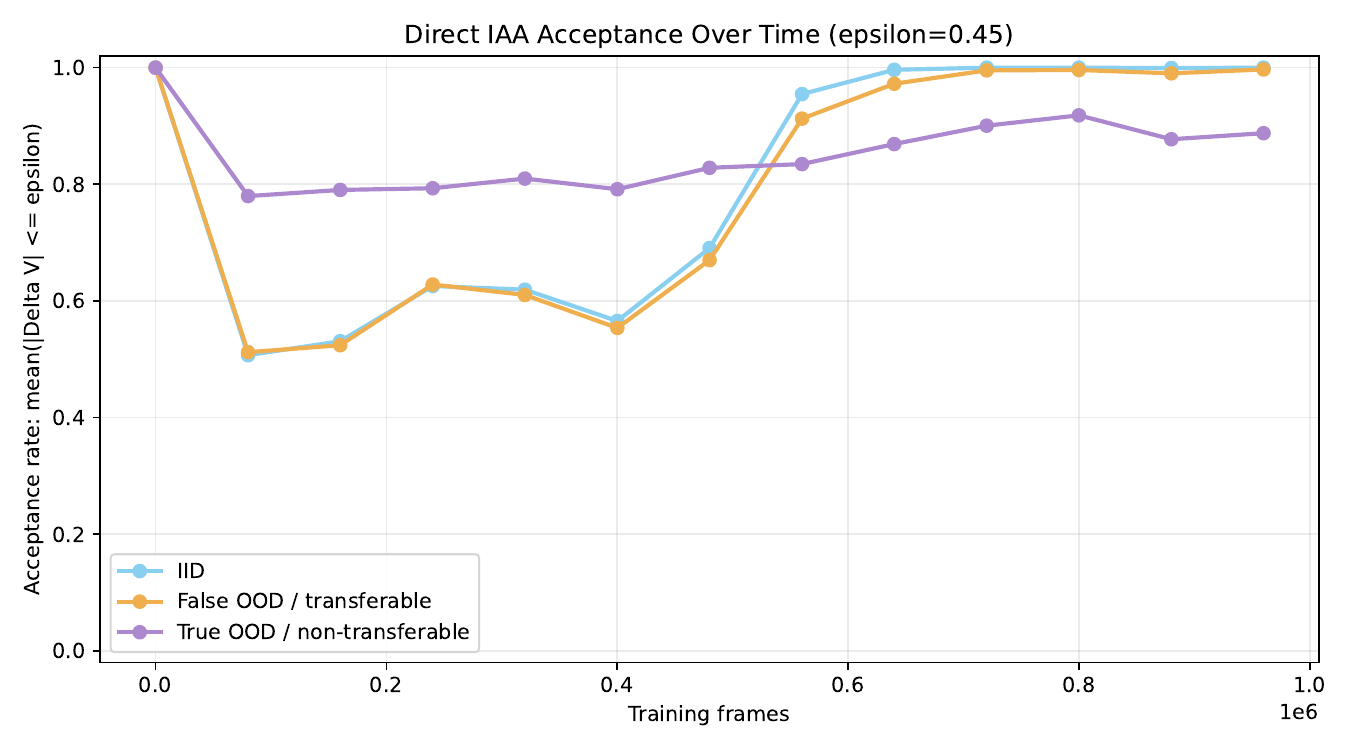}\\[-1pt]
    \includegraphics[width=0.34\textwidth]{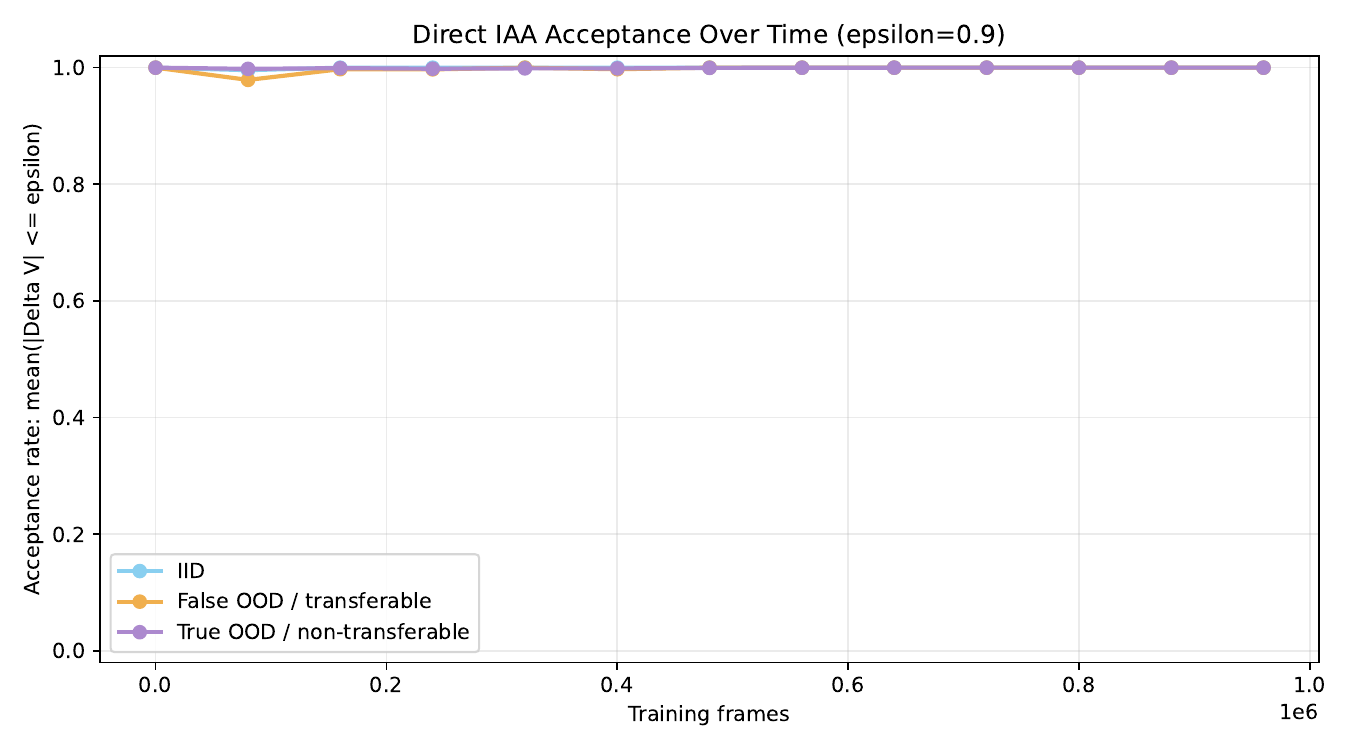}\hspace{2pt}%
    \includegraphics[width=0.34\textwidth]{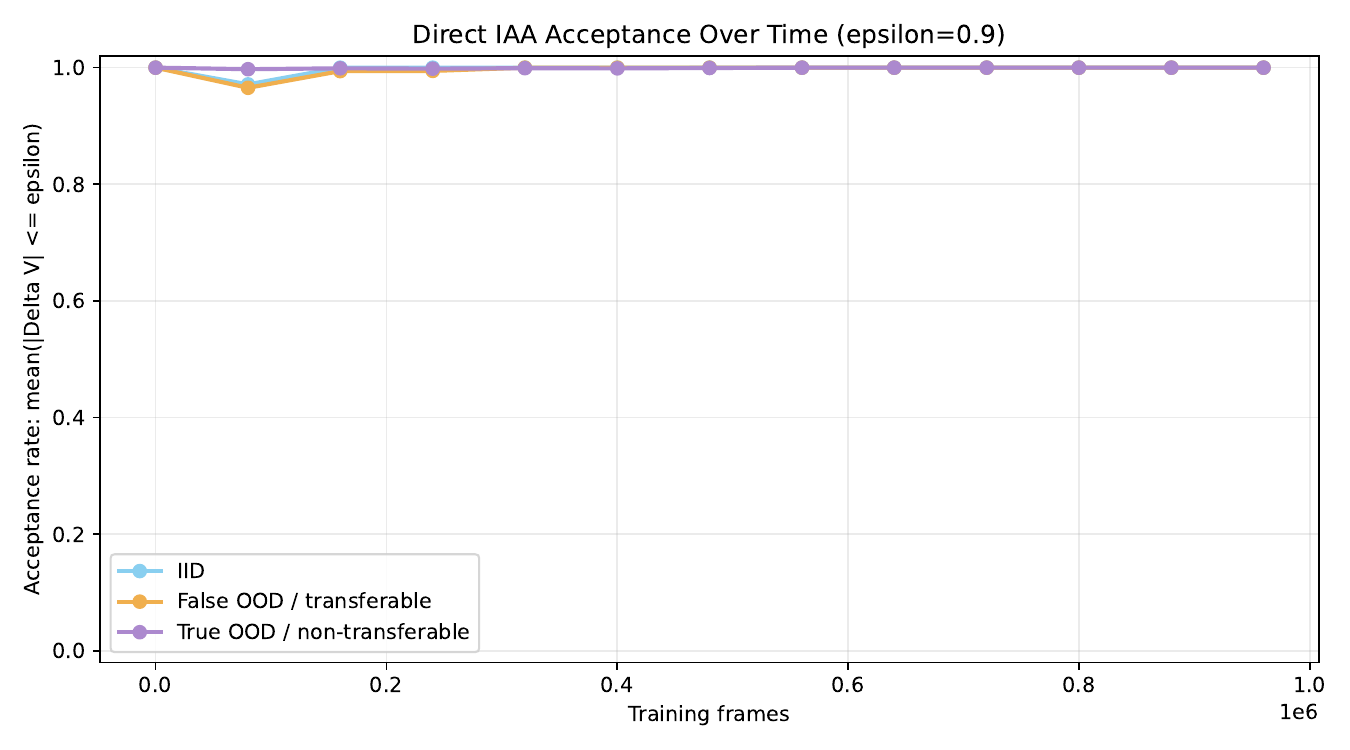}
    \caption{\textbf{IAA value-gap separation in the Unlocked-to-Locked
    setting}, shown across $\epsilon \in \{0.05, 0.1, 0.15, 0.3, 0.45, 0.9\}$
    (one per row); left column freezes the convolutional backbone, right column
    fine-tunes the full network. Each panel tracks three groups of states over
    the course of the student's learning: \emph{ID} (in-distribution, collected
    during teacher training), \emph{Transferable OOD} (states encountered during
    the student's learning where the teacher \emph{should} issue advice), and
    \emph{Non-transferable OOD} (states where the teacher \emph{should not}). A
    reliable signal should place the ID and Transferable-OOD curves above the
    Non-transferable-OOD curve, indicating that good advice is issued more often
    than bad advice.}
    \label{fig:critic_evolution_locked}
\end{figure}

\begin{figure}[htbp]
    \centering
    \includegraphics[width=0.34\textwidth]{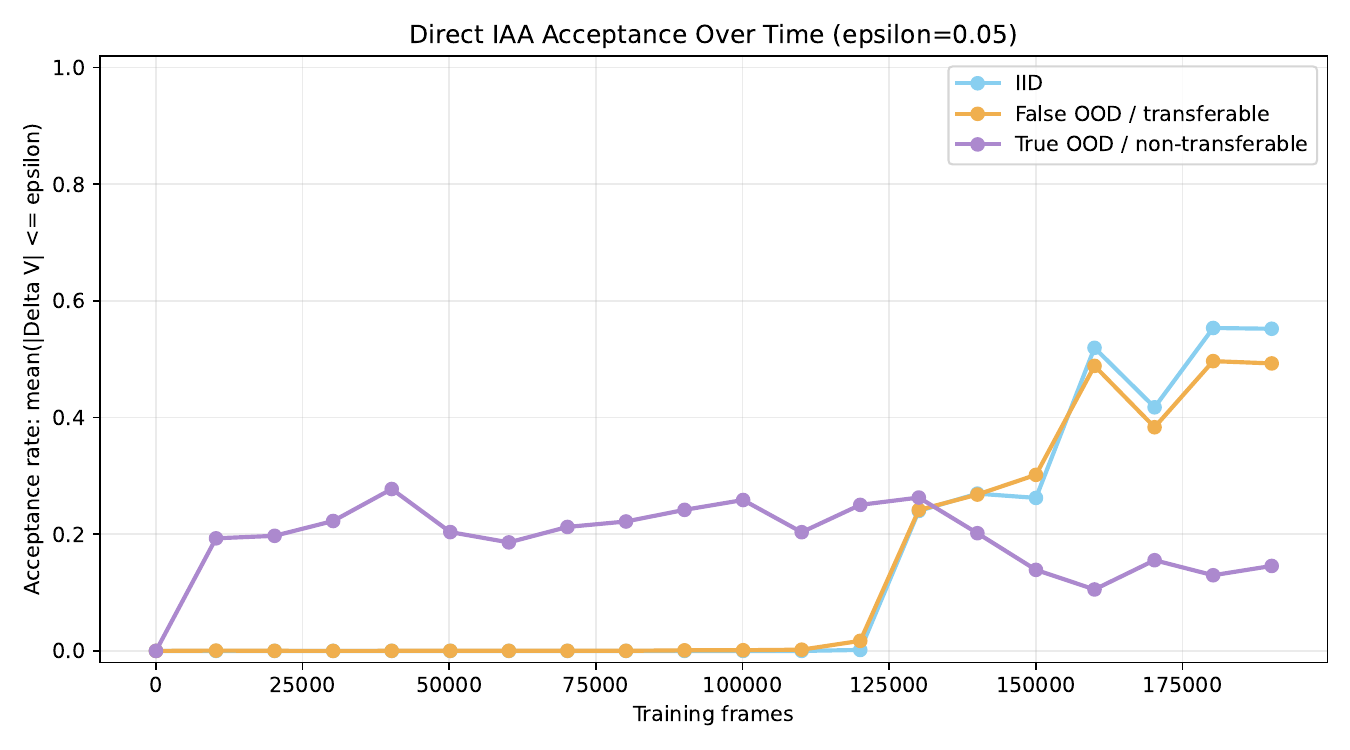}\hspace{2pt}%
    \includegraphics[width=0.34\textwidth]{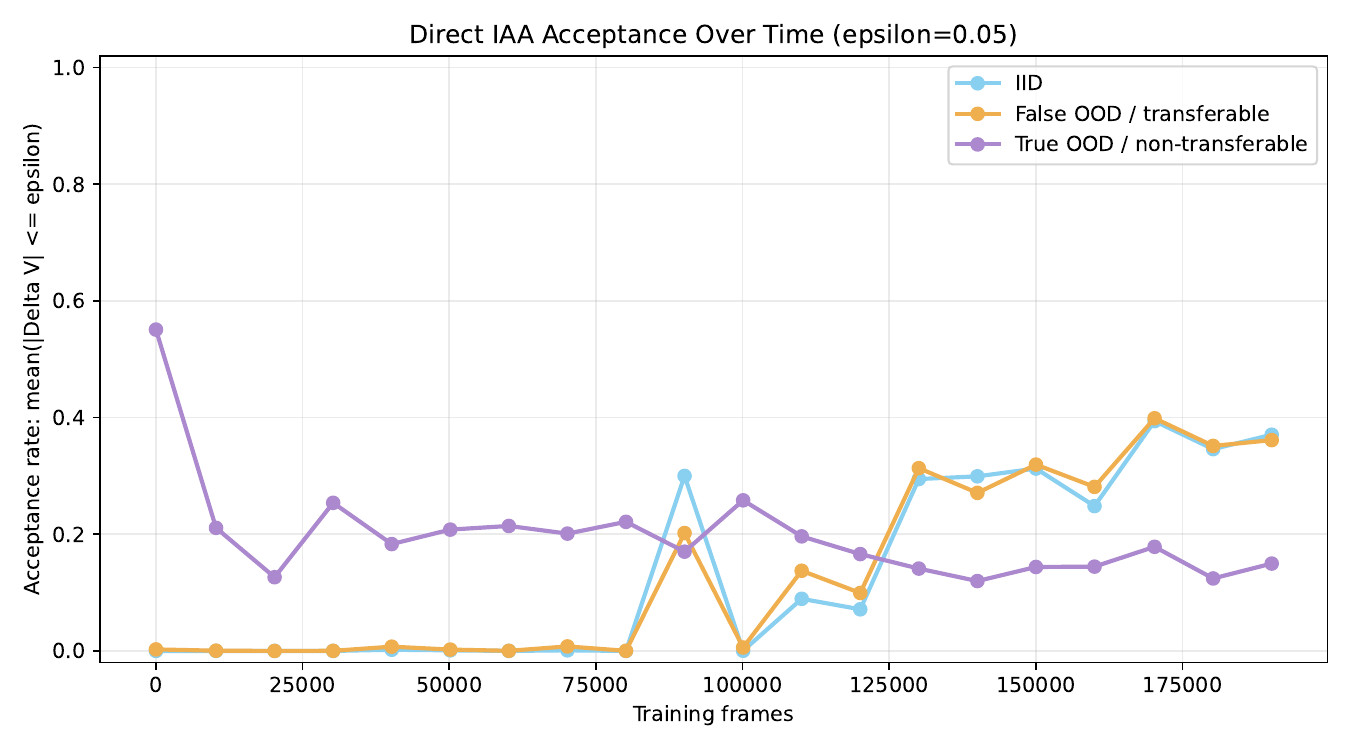}\\[-1pt]
    \includegraphics[width=0.34\textwidth]{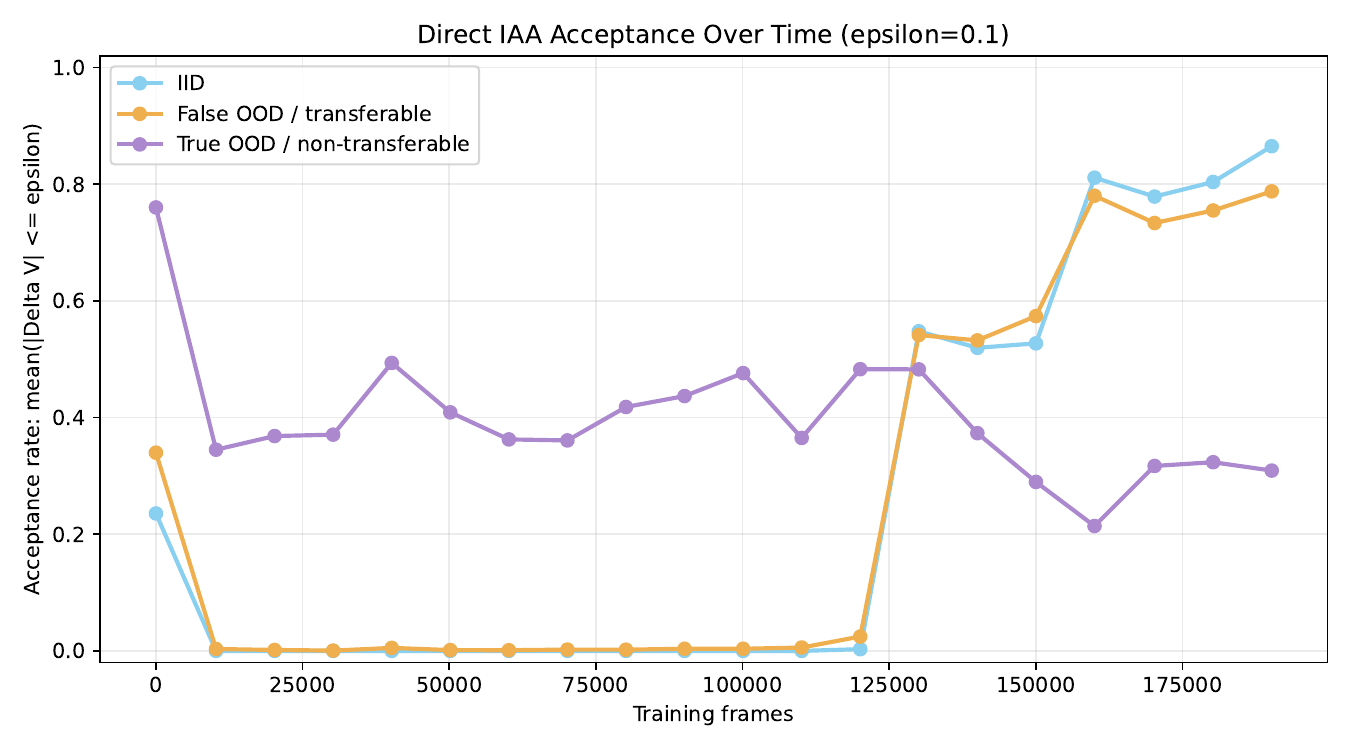}\hspace{2pt}%
    \includegraphics[width=0.34\textwidth]{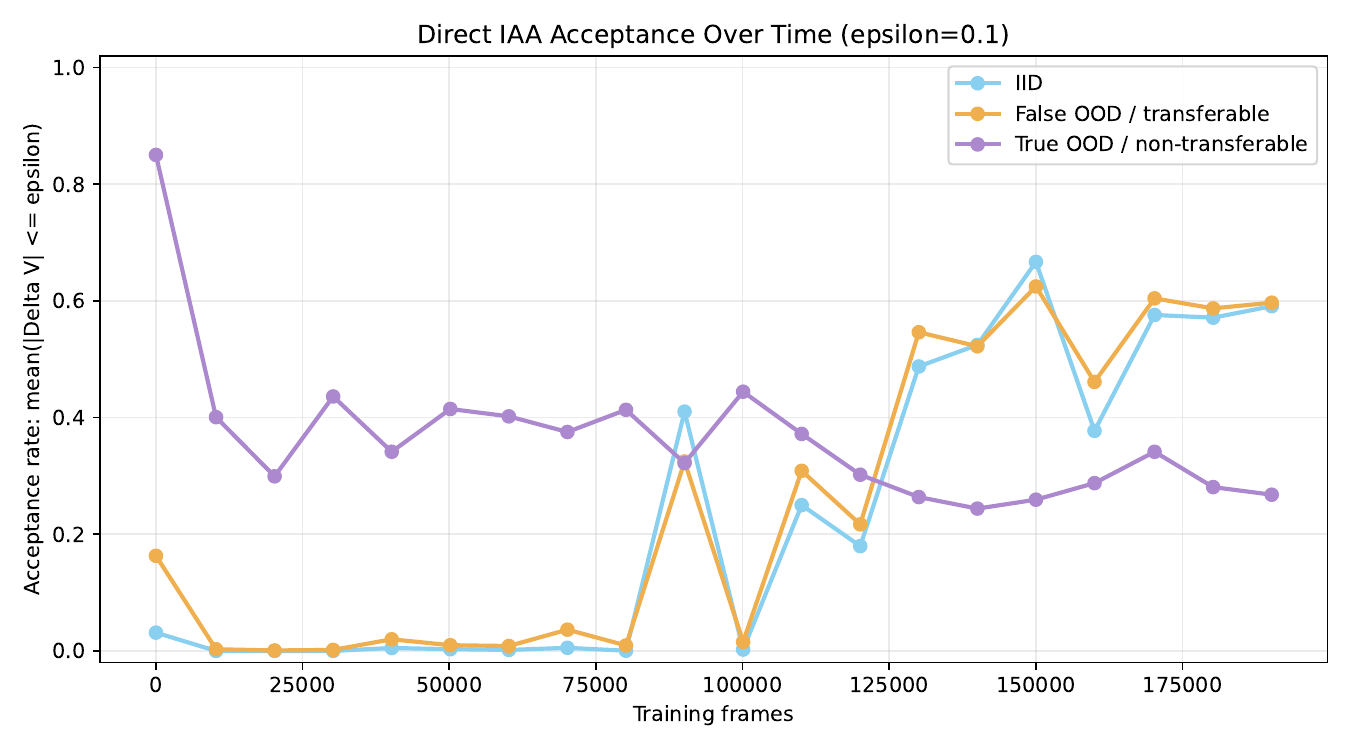}\\[-1pt]
    \includegraphics[width=0.34\textwidth]{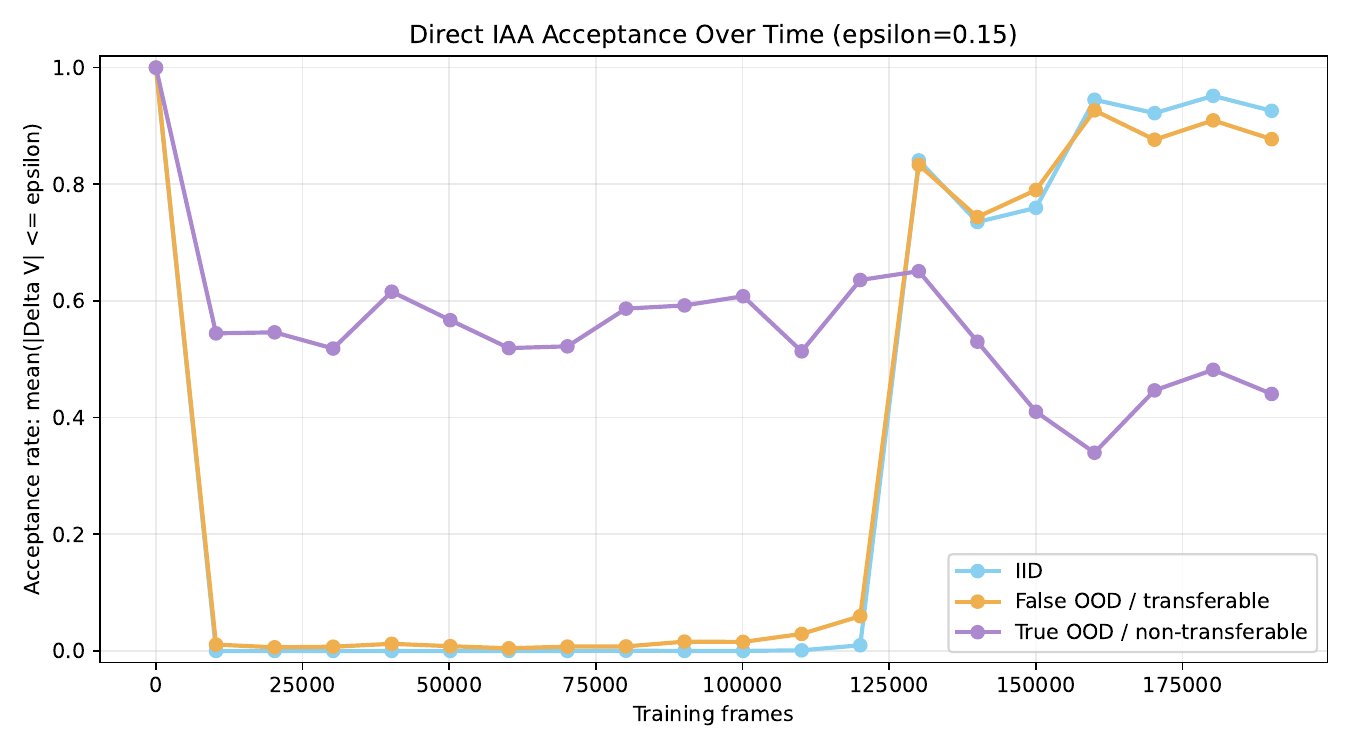}\hspace{2pt}%
    \includegraphics[width=0.34\textwidth]{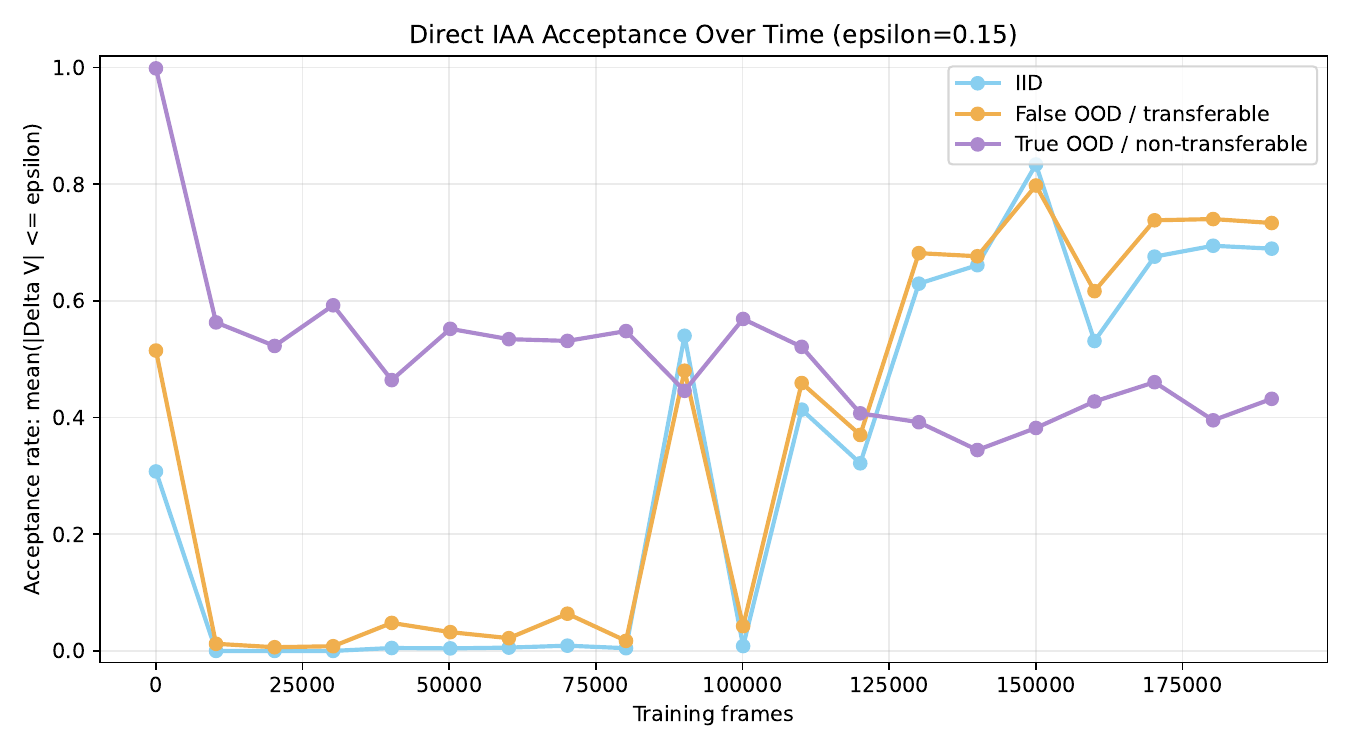}\\[-1pt]
    \includegraphics[width=0.34\textwidth]{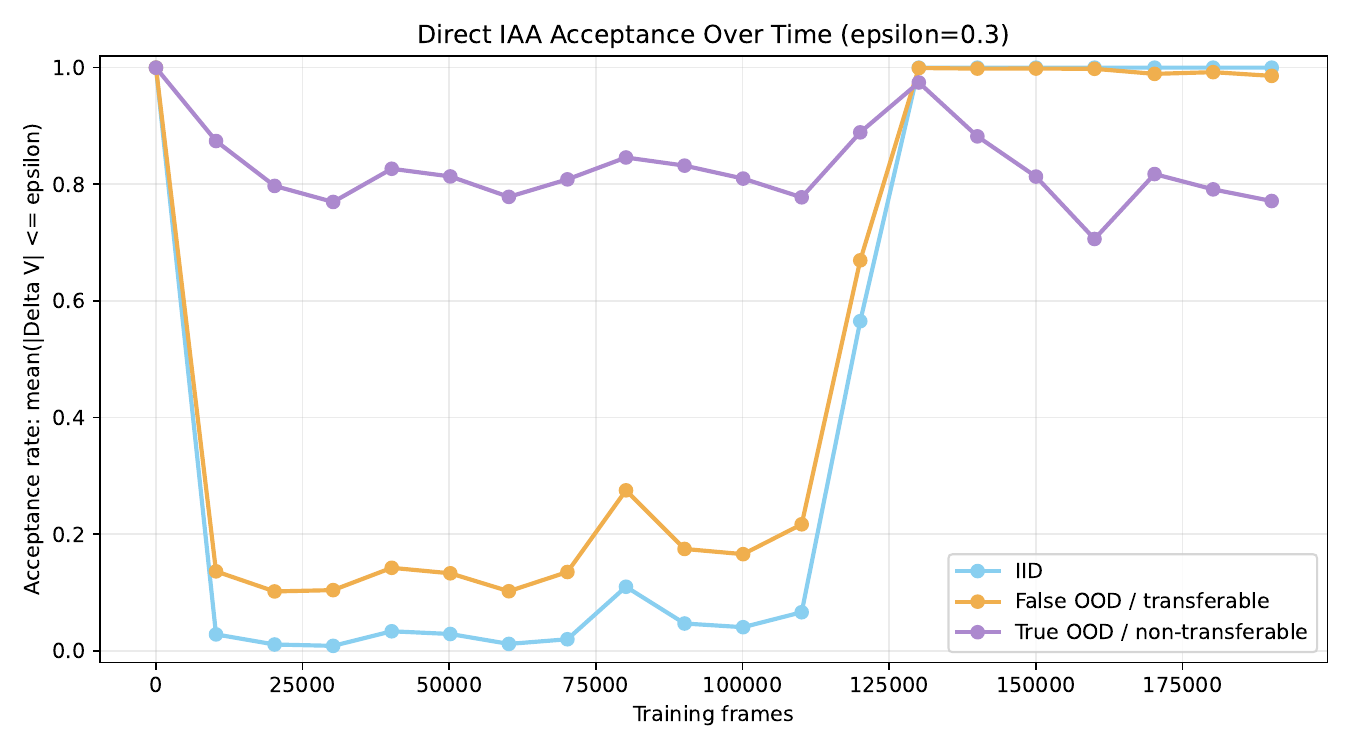}\hspace{2pt}%
    \includegraphics[width=0.34\textwidth]{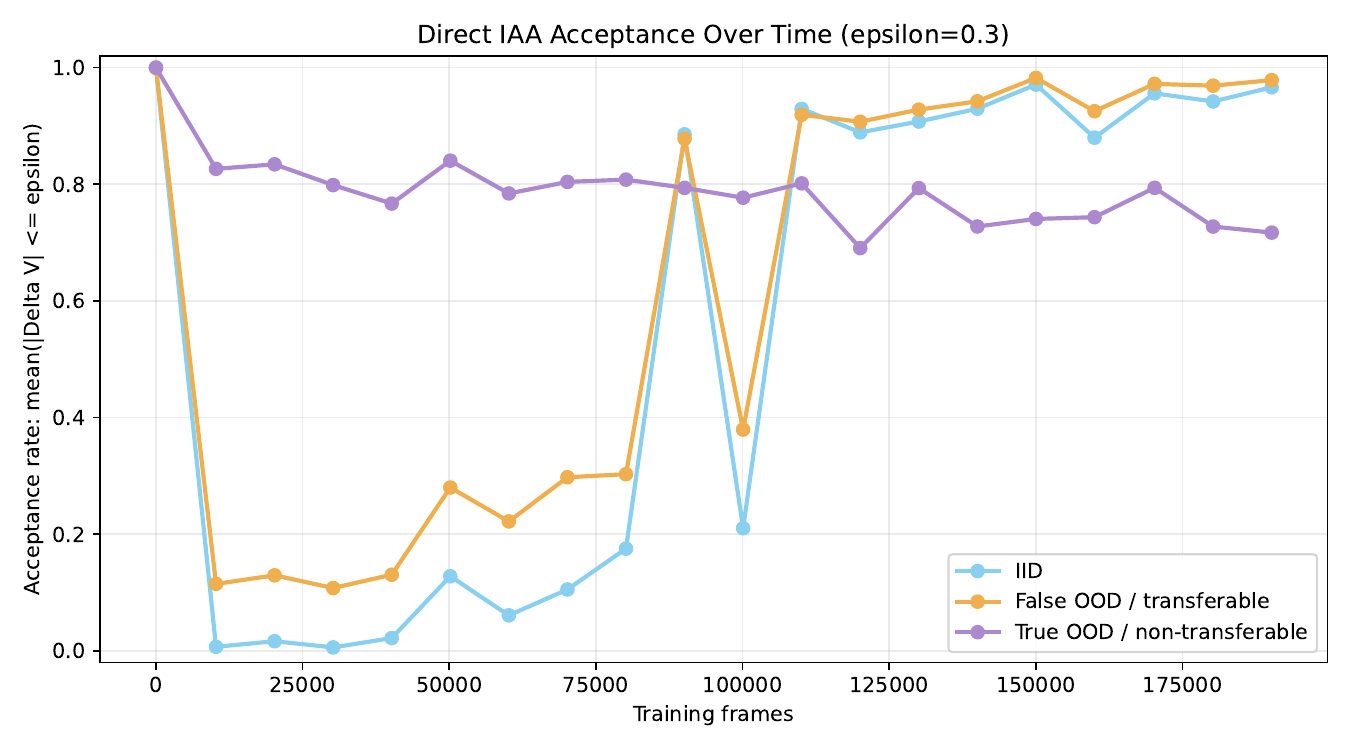}\\[-1pt]
    \includegraphics[width=0.34\textwidth]{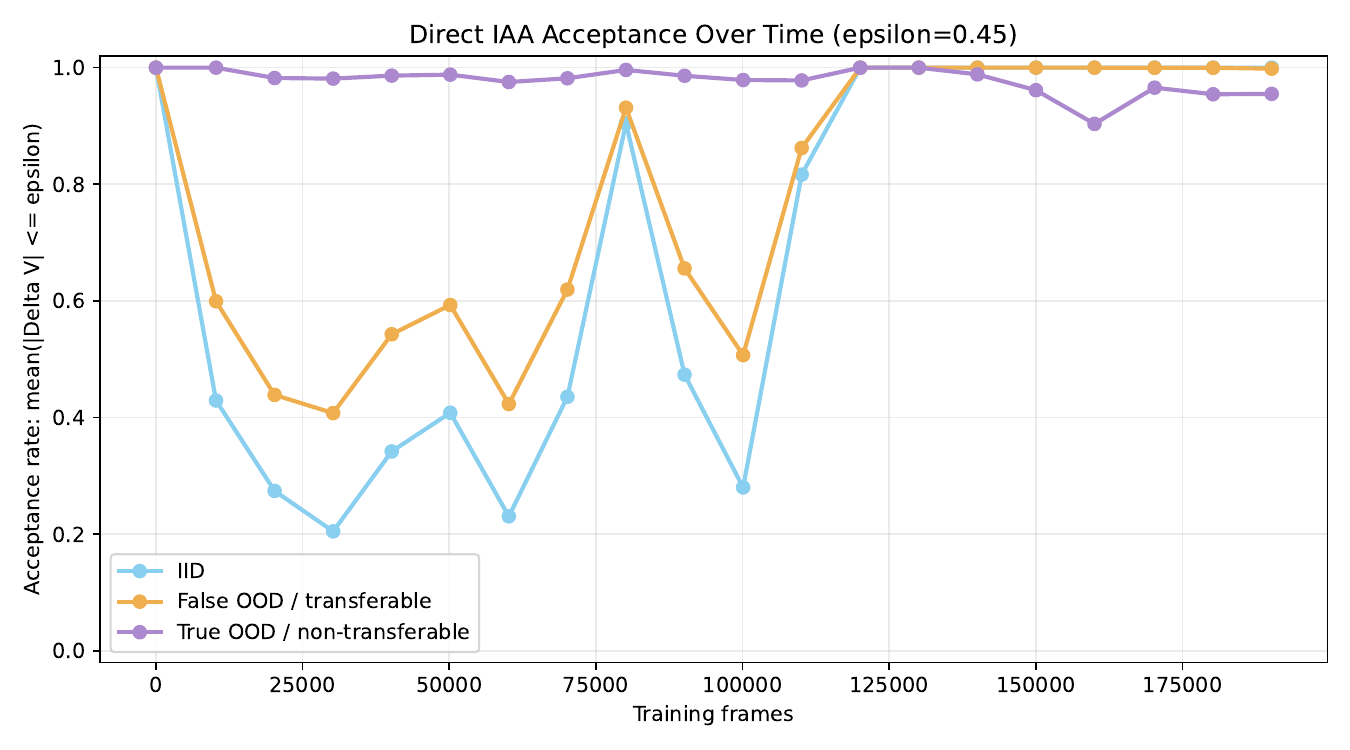}\hspace{2pt}%
    \includegraphics[width=0.34\textwidth]{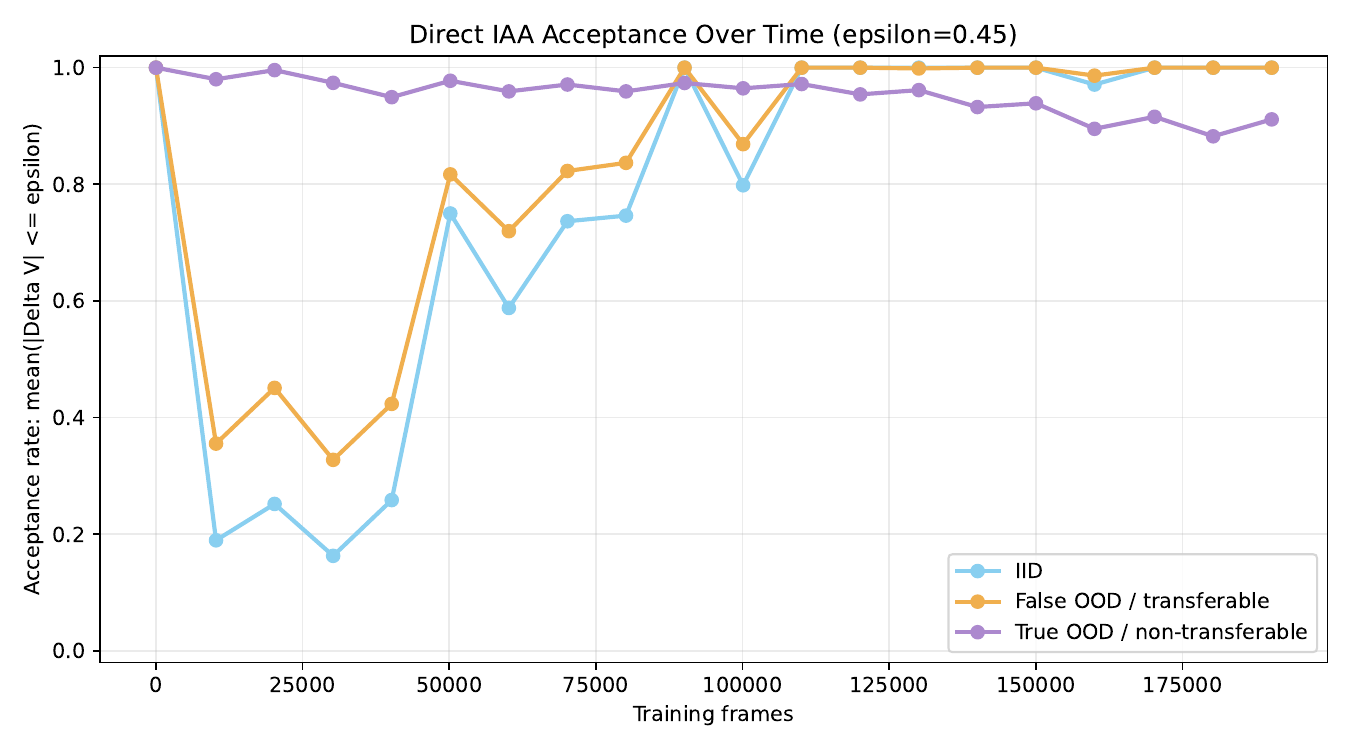}\\[-1pt]
    \includegraphics[width=0.34\textwidth]{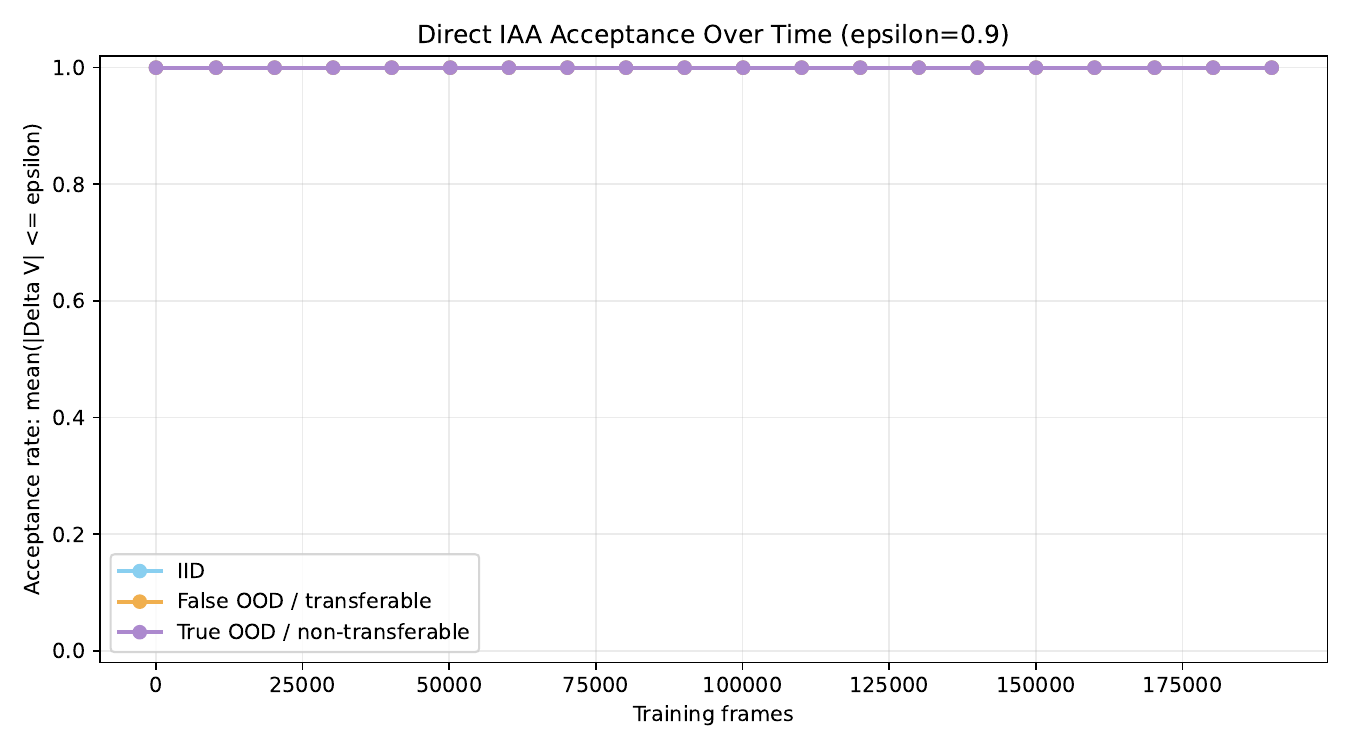}\hspace{2pt}%
    \includegraphics[width=0.34\textwidth]{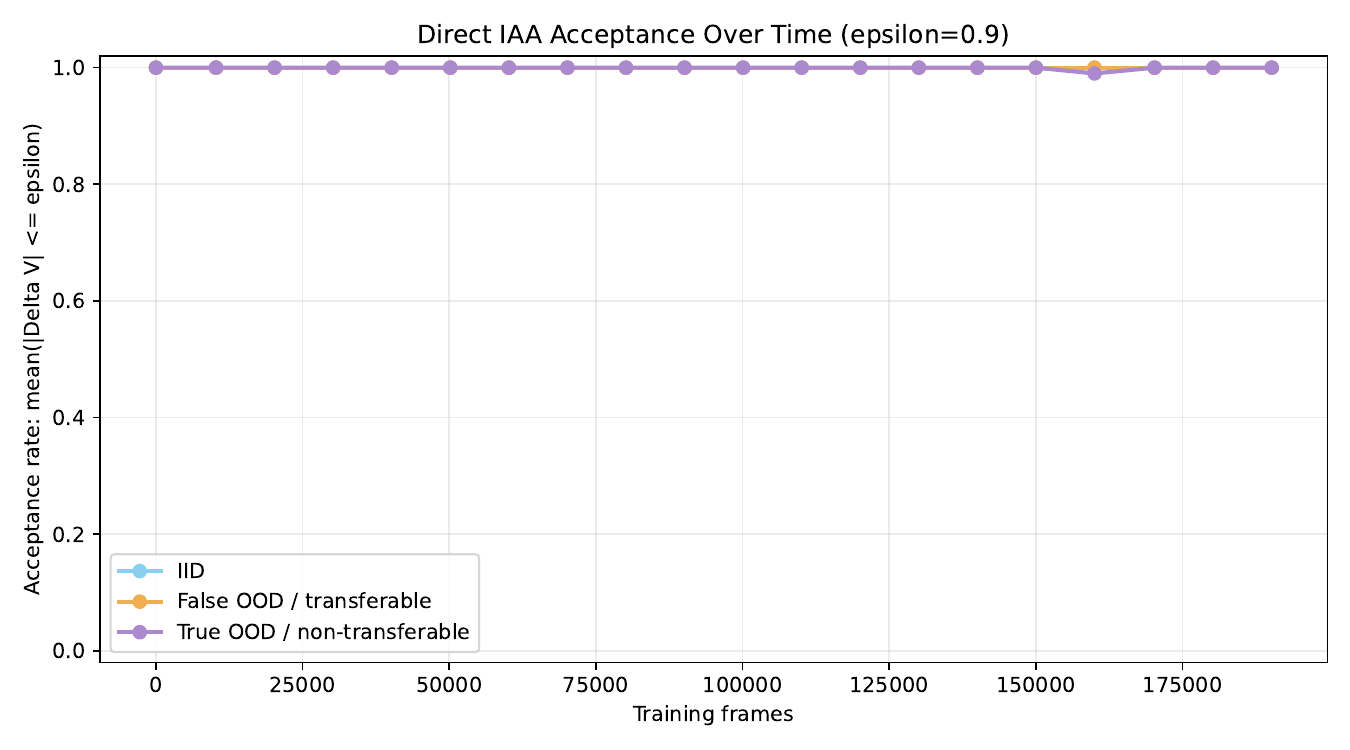}
    \caption{\textbf{IAA value-gap separation in the Alternating-Goal
    setting}, across $\epsilon \in \{0.05, 0.1, 0.15, 0.3, 0.45, 0.9\}$ (one per
    row); left column freezes the convolutional backbone, right column
    fine-tunes the full network. Curves and the expected ordering are as in
    Figure~\ref{fig:critic_evolution_locked}: a reliable signal keeps the ID and
    Transferable-OOD curves above the Non-transferable-OOD curve.}
    \label{fig:critic_evolution_alternating}
\end{figure}